\documentclass{article}

\usepackage{microtype}
\usepackage{graphicx}
\usepackage{subfigure}
\usepackage{booktabs} % for professional tables

\usepackage{hyperref}

\usepackage[accepted]{icml2023}

\usepackage{amsmath}
\usepackage{amssymb}
\usepackage{mathtools}
\usepackage{amsthm}

\usepackage[capitalize]{cleveref}

\theoremstyle{plain}
\newtheorem{theorem}{Theorem}[section]
\newtheorem{proposition}[theorem]{Proposition}
\newtheorem{lemma}[theorem]{Lemma}
\newtheorem{corollary}[theorem]{Corollary}
\theoremstyle{definition}

\newtheorem{assumption}[theorem]{Assumption}
\newtheorem{remark}[theorem]{Remark}

\usepackage[textsize=tiny]{todonotes}

\icmltitlerunning{Stochastic Gradient Succeeds for Bandits}

\usepackage{amsmath,amsfonts,bm}
\usepackage{amssymb,amsthm}
\usepackage{mathtools}
\usepackage{algorithm,algorithmic}
\usepackage{url}
\usepackage{multirow}
\usepackage{array}
\usepackage{cancel}
\usepackage{etoc}
\usepackage{pifont}
\usepackage{wrapfig}

\renewcommand{\algorithmicrequire}{\textbf{Input:}}
\renewcommand{\algorithmicensure}{\textbf{Output:}}

\crefname{proposition}{Proposition}{Propositions}
\crefname{theorem}{Theorem}{Theorems}
\crefname{lemma}{Lemma}{Lemmas}
\crefname{update_rule}{Update}{Updates}
\crefname{algorithm}{Algorithm}{Algorithms}
\crefname{figure}{Figure}{Figures}
\crefname{claim}{Claim}{Claims}
\crefname{table}{Table}{Tables}

\def\eqref#1{equation~\ref{#1}}
\def\1{\bm{1}}

\DeclareMathAlphabet{\mathsfit}{\encodingdefault}{\sfdefault}{m}{sl}
\SetMathAlphabet{\mathsfit}{bold}{\encodingdefault}{\sfdefault}{bx}{n}

\def\gA{{\mathcal{A}}}

\def\gD{{\mathcal{D}}}
\def\gE{{\mathcal{E}}}
\def\gF{{\mathcal{F}}}

\def\gN{{\mathcal{N}}}

\def\gP{{\mathcal{P}}}

\def\sI{{\mathbb{I}}}

\def\sR{{\mathbb{R}}}

\newcommand{\R}{\mathbb{R}}

\newcommand{\softmax}{\mathrm{softmax}}

\DeclareMathOperator*{\argmax}{arg\,max}

\theoremstyle{definition}

\newtheorem{update_rule}{Update}
\newtheorem{claim}{Claim}

\DeclareMathOperator*{\expectation}{\mathbb{E}}

\def\rvzero{{\mathbf{0}}}

\def\diagonalmatrix{\text{diag}}

\DeclareMathOperator*{\probability}{Pr}

\newcommand{\cA}{\mathcal{A}}
\newcommand{\cF}{\mathcal{F}}
\newcommand{\chE}{\mathbb{E}}
\newcommand{\EE}[1]{\mathbb{E}[#1]}
\newcommand{\PP}{\mathbb{P}}
\newcommand{\EEt}[1]{\mathbb{E}_t[#1]}

\newlength\tocrulewidth
\newcommand{\calA}{\mathcal{A}}
\newcommand{\calE}{\mathcal{E}}
\newcommand{\bbE}{\mathbb{E}}
\newcommand{\bbP}{\mathbb{P}}
\makeatletter
\newcommand{\thickhline}{%
    \noalign {\ifnum 0=`}\fi \hrule height 0.8pt
    \futurelet \reserved@a \@xhline
}
\newcolumntype{"}{@{\hskip\tabcolsep\vrule width 1.5pt\hskip\tabcolsep}}
\makeatother

\usepackage{enumitem}

\begin{document}

\twocolumn[
\icmltitle{Stochastic Gradient Succeeds for Bandits}



\icmlsetsymbol{equal}{*}

\begin{icmlauthorlist}
\icmlauthor{Jincheng Mei}{equal,brain}
\icmlauthor{Zixin Zhong}{equal,ua}
\icmlauthor{Bo Dai}{brain,gt}
\icmlauthor{Alekh Agarwal}{gr}
\icmlauthor{Csaba Szepesv{\' a}ri}{ua}
\icmlauthor{Dale Schuurmans}{brain,ua}

\end{icmlauthorlist}

\icmlaffiliation{brain}{Google DeepMind}
\icmlaffiliation{ua}{University of Alberta}
\icmlaffiliation{gt}{Georgia Tech}
\icmlaffiliation{gr}{Google Research}

\icmlcorrespondingauthor{Jincheng Mei}{jcmei@google.com}
\icmlcorrespondingauthor{Zixin Zhong}{zzhong10@ualberta.ca}
\icmlcorrespondingauthor{Csaba Szepesv{\' a}ri}{szepi@google.com}

\icmlkeywords{Machine Learning, ICML}

\vskip 0.3in
]



\printAffiliationsAndNotice{$^*$Equal contribution. This version corrects a mistake in the proofs for \cref{thm:asymp_global_converg_gradient_bandit_sampled_reward} by adding a martingale concentration inequality of \cref{thm:conc_new}. The authors highly appreciate the help from Sharan Vaswani and Michael Lu at Simon Fraser University, and Anant Raj at SIERRA Project Team (Inria), Coordinated Science Laboratory (CSL), UIUC, for spotting the mistake in a previous version.} % otherwise use the standard text.

\begin{abstract}
We show that the \emph{stochastic gradient} bandit algorithm converges to a \emph{globally optimal} policy at an $O(1/t)$ rate,
even with a \emph{constant} step size.
Remarkably, global convergence of the stochastic gradient bandit algorithm has not been previously established,
even though %it is one of the oldest algorithms known 
it is an old algorithm known to be applicable to bandits. %\alekh{This is highly questionable. Thompson's paper is from 1933, and the Lai and Robbins KL-UCB paper is contemporary with the very beginning of stochastic optimization.}
The new result is achieved by establishing two novel technical findings: 
%\textbf{(i)}
first, the noise of the stochastic updates in the gradient bandit algorithm satisfies a strong ``growth condition'' %``self-bounding'' 
property, where the variance diminishes whenever progress becomes small, 
implying that additional noise control via diminishing step sizes is unnecessary;
%and \textbf{(ii)} 
second, a form of ``weak exploration'' is automatically achieved through the stochastic gradient updates, since they prevent the action probabilities from decaying faster than $O(1/t)$, thus ensuring that every action is sampled infinitely often with probability $1$. 
% indicates there is no need of any extra noise control techniques, such as decaying learning rates, regularization, or variance reduction with softmax parametrizatin. 
These two findings 
% leads the analysis of \emph{global convergence} of stochastic gradient bandit, indicating that 
can be used to show that the stochastic gradient update is already ``sufficient'' for bandits in the sense that exploration versus exploitation is automatically balanced in a manner that ensures almost sure convergence to a global optimum.
% Our analysis reveals a deeper reason for the global convergence that the softmax parameterization automatically reduces the aggressiveness of the update, not only cancels but also annihilates the noise in stochastic updates. 
%Experimental results further verify the theoretical findings.
These novel theoretical findings are further verified by experimental results.
%\alekh{Removed the remark about regret from abstract as it will likely confuse people.} 
%\alekh{It seems strange to call the property of noise a growth condition, as growth conditions usually refer to something growing fast enough. Isn't it a self-bounding property of the noise, really, since we are saying something like variance is bounded by square of expectation?}
\end{abstract}

\section{Introduction}
\label{sec:introduction}

Algorithms for multi-armed bandits (MABs) need to balance exploration and exploitation to achieve desirable performance properties \citep{lattimore2020bandit}. Well known bandit algorithms generally introduce auxiliary mechanisms to control the exploration-exploitation trade-off. For example, the upper confidence bound strategies (UCB, \citet{lai1985asymptotically, auer2002finite}), manage exploration explicitly by designing auxiliary bonuses that induce optimism under uncertainty; Thompson sampling \citep{thompson1933likelihood,agrawal2012analysis} maintains a posterior over rewards that has to be updated and sampled from tractably. 
% which can make this approach challenging to apply in complex scenarios where function approximation is needed to model rewards over large action spaces.
%are challenging to scale to complex setups with large action spaces, where parametric reward models with function approximation are desirable, and the posterior is challenging to tractably represent and compute. 
The theoretical analysis of such algorithms typically focuses on bounding their regret, i.e., showing that the average reward obtained by the algorithm approaches that of the optimal action in hindsight, at a rate that is statistically optimal. However, managing exploration bonuses via uncertainty quantification is difficult in all but simple environments~\citep{gawlikowski2021survey}, while posterior updating and sampling can also be computationally difficult in practice, which make the UCB and Thompson sampling challenging to apply in real world scenarios. 
% In this paper, we instead study a simpler yet widely-used alternative, the gradient bandit algorithm~\citep{sutton2018reinforcement}, that uses stochastic gradient-based techniques without any further \emph{explicit exploration} to solve multi-armed bandit problems. 

Meanwhile, stochastic gradient-based techniques have witnessed widespread use across the breadth of machine learning. For bandit and reinforcement learning problems, stochastic gradient yields a particularly simple algorithm---the stochastic gradient bandit algorithm~\citep[Section 2.8]{sutton2018reinforcement}---that omits any \emph{explicit} control mechanism over \emph{ exploration}. This algorithm is naturally compatible with deep neural networks, in stark contrast to UCB and Thompson sampling, and has seen significant practical success~\citep{schulman2015trust,schulman2017proximal,ouyang2022training}. 
% Given the widespread use of gradient-based approaches and the empirical success they enjoy, 
Surprisingly, the theoretical understanding of this algorithm remains under-developed: the global convergence and regret properties of the stochastic gradient bandit algorithm is still open, which naturally raises the question:
% It is natural to ask whether there are underlying theoretical properties that explain their effectiveness. 
\begin{center}
    \emph{Is the stochastic gradient bandit algorithm able to balance exploration vs. exploitation to identify an optimal action?}
    % in terms of expected reward, and if so, what is the trade-off achieved in balancing exploration and exploitation?
\end{center}
% The question we consider is whether such algorithms eventually identify an optimal action in terms of expected reward, and if so, what is the trade-off achieved in balancing exploration and exploitation. 
Such an understanding is paramount to further improving the underlying approach. In this paper, we take a step in this direction by studying the convergence properties of the canonical stochastic gradient bandit algorithm in the simplest setting of multi-armed bandits, and answer the question affirmatively: the distribution over the arms maintained by this algorithm almost surely concentrates  asymptotically on a globally optimal action. %as the algorithm processes. % over time.

\begin{table*}[t]%[ht]
\begin{center}
\begin{small}
\renewcommand{\arraystretch}{1.3}
\begin{tabular}{  l | c | c | c  }
\thickhline
Reference  & Convergence rate & Learning rate $\eta$ & Remarks \\
\hline
\citet{zhang2020sample} &  $\vphantom{\Big(}\tilde{O}(1/\sqrt{t})$  & $\tilde{O}(1/\sqrt{t})$ & Log-barrier regularization \\ 
\citet{ding2021beyond} &  $O(1/\sqrt{t})$ & $O(1/t)$ & Entropy regularization \\
\citet{zhang2021convergence} & $O(1/\sqrt{t})$ & $O(1/\log{t})$ & Gradient truncation + variance reduction\\
\citet{yuan2022general} & $\tilde{O}(1/t^{1/3})$ & $O(1/t)$ &   ABC assumptions      \\
\citet{mei2022role} & $O(C_{\Delta}/t)$ &  $O(1/t)$ & Oracle baseline, $C_{\Delta}$ is initialization and problem dependent \\
\textbf{This paper} & $\bm{O(C_{\Delta}/t)}$ & $\bm{O(1)}$ & \textbf{$C_{\Delta}$ is initialization and problem dependent} \\
\thickhline
\end{tabular}
\renewcommand{\arraystretch}{1}
\end{small}
\end{center}
%\vskip -0.1in
\caption{Global convergence results for gradient based methods. To simplify comparison, regret or sample complexity bounds have been converted to convergence rates through the usual translations.  (This comparison does not capture differences in other metrics like regret.) We also compare the learning rates used in the underlying algorithm, since the ability to use a constant learning rate is a key insight of the analysis presented %in this paper 
and is generally preferred in practice. Note that \citet{zhang2021convergence} claimed a constant learning rate; however, their learning rate must follow a decaying ``truncation threshold'', which means the actual learning rate in \citep{zhang2021convergence} is decaying.}
\label{table:stochastic_policy_gradient_results}
\end{table*}

Of course, the broader literature on the use of stochastic gradient techniques in reinforcement learning has a long tradition, dating back to stochastic approximation 
\citep{robbins1951stochastic} with likelihood ratios (``log trick'') \citep{glynn1990likelihood} and REINFORCE policy gradient estimation~\citep{williams1992simple}. 
% Unlike UCB or Thompson sampling, the canonical gradient bandit algorithm does not have a regret analysis, nor has its global convergence properties been previously characterized.  %and its global convergence still remains an open question.
It is well known that REINFORCE \citep{williams1992simple} with on-line Monte Carlo sampling provides an unbiased gradient estimator of bounded variance, which is sufficient to guarantee convergence to a stationary point if the learning rates are decayed appropriately \citep{robbins1951stochastic,zhang2020global}. However, convergence to a stationary point is an extremely weak guarantee for bandits, since any deterministic policy, whether optimal or sub-optimal, has a zero gradient and is therefore a stationary point.  Convergence to a stationary point, on its own, is insufficient to ensure convergence to a globally optimal policy or even to establish regret bounds. % is not enough to guarantee global convergence or regret results.

Recently, it has been shown that if true gradients are used, softmax policy gradient (PG) methods converge to a globally optimal policy asymptotically \citep{agarwal2021theory}, with an $O(1/t)$ rate \citep{mei2020global}, albeit with initialization and problem dependent constants \citep{mei2020escaping,li2021softmax}. However, these results do not apply to the gradient bandit algorithm as it uses \emph{stochastic} gradients, 
and a key theoretical challenge is to account for the effects of stochasticity from on-policy sampling and reward noise.

More recent results on the global convergence of PG methods with \emph{stochastic} updates have been established, as summarized in
%Combining results of convergence to stationary points  with PG global convergence results, more recent results on PG global convergence with stochastic updates have been established, as summarized in 
\cref{table:stochastic_policy_gradient_results}. 
In particular, \citet{zhang2020global} showed that REINFORCE \citep{williams1992simple} with $O(1/\sqrt{t})$ learning rates and log-barrier regularization has $\tilde{O}(1/\sqrt{t})$ average regret. 
\citet{ding2021beyond} proved that softmax PG with $O(1/t)$ learning rates and entropy regularization has $\tilde{O}(\epsilon^{-2})$ sample complexity. 
\citet{zhang2021convergence} showed that with gradient truncation softmax PG gives $O(\epsilon^{-2})$ sample complexity.
Under extra assumptions, 
%\bo{if we do not want to explain what is ABC assumption, just use ``extra assumptions''}, 
\citet{yuan2022general} obtained $\tilde{O}(\epsilon^{-3})$ sample complexity with $O(1/t)$ learning rates. 
\citet{mei2022role} analyzed on-policy natural PG with value baselines and $O(1/t)$ learning rates and proved $O(1/t)$ convergence rate.
The results in these works are expressed in different metrics,
%The original results in those works are stated in different metrics, 
such as average regret, sample complexity, or convergence rate, which can sometimes make comparisons difficult. %makes them not comparable to each other in general. 
However, for the bandit case, where
%bandit algorithms we consider here, i.e., if 
only one example is used in each iteration, these metrics become comparable, such that $O(\epsilon^{-n})$ sample complexity is equal to $O(t^{-1/n})$ convergence rate, which is stronger than $O(t^{-1/n})$ average regret, but not necessarily vice versa.%
\footnote{
It is worth noting that these works also contain results for general Markov decision processes (MDPs).  
We express their results for bandits here by treating this case as a single state MDP. 
}

%Two of the common 
The two key
shortcomings in these existing results are,
%This existing work exhibits two common shortcomings:
%Given the literature context, we note that most existing results have two common features: 
%\textbf{(i)},
\textbf{first},
they introduce
decaying learning rates (or regularization and/or variance reduction) to explicitly control noise, % and to further achieve convergence,
and \textbf{second},
such auxiliary techniques generally incur additional computation and 
%\textbf{(ii)} 
decelerate convergence to an $O(1/\sqrt{t})$ or slower rate. 
The only exception to %\textbf{(ii)}
the latter is \citet{mei2022role}, which considers an aggressive $O(1/t)$ learning rate decay and still establishes $O(1/t)$ convergence, but leverages an unrealistic baseline to achieve this.

In this paper, we provide a new global convergence analysis of stochastic gradient bandit algorithms \emph{with constant learning rates} by establishing novel properties and techniques. There are two main benefits to the results presented.
\begin{itemize}
    \item %Our results are obtained using
    By considering only constant learning rates, we show that auxiliary forms of noise control such as learning rate decay, regularization and variance reduction are \emph{unnecessary} to achieve global convergence, which justifies the use of far simpler algorithms in practice.
        \item Unlike previous work, we show that a \emph{practical} and general algorithm can achieve an optimal $O(1/t)$ convergence rate and $O(\log{T})$ regret asymptotically.
\end{itemize}
The remaining paper is organized as follows. 
\cref{sec:gradient_bandit_algorithms,sec:standard_stochastic_gradient_analysis} introduce the gradient bandit algorithms and standard stochastic gradient analysis respectively. 
\cref{sec:vanishing_noise} presents our novel technical findings, characterizing the automatic noise cancellation effect and global landscape properties,
% refined stochastic analysis of using landscape properties and softmax Jacobian behavior to automatically reduce stochastic noise, 
which is then leveraged in \cref{sec:global_convergence_analysis} to establish novel global convergence results. 
\cref{sec:effect_of_baselines} discusses the effect of using baselines. \cref{sec:simulation_results} presents a simulation study to verify the theoretical findings, and \cref{sec:further_discussions}  provides further discussions. \cref{sec:conclusions} briefly concludes this work.

\section{Gradient Bandit Algorithms}
\label{sec:gradient_bandit_algorithms}

A multi-armed bandit (MAB) problem is specified by an action set $[K] \coloneqq \{ 1, 2, \cdots, K \}$ and random rewards with mean vector $r \in \sR^K$. For each action $a \in [K]$, the mean reward $r(a)$ is the expectation of a bounded reward distribution, %i.e.,
\begin{align}
\label{eq:true_mean_reward_expectation_bounded_sampled_reward}
    r(a) = \int_{-R_{\max}}^{R_{\max}}{ x \cdot P_a(x) \mu(d x)},
\end{align}
where $\mu$ is a finite measure over $[-R_{\max}, R_{\max}]$, $P_a(x) \ge 0$ is a probability density function with respect to $\mu$, and $R_{\max} > 0$ is the reward range. Since the sampled reward is bounded, we also have $r \in [-R_{\max}, R_{\max}]^K$. We make the following assumption on $r$.
\begin{assumption}[True mean reward has no ties]
\label{assp:reward_no_ties}
For all $i, j \in [K]$, if $i \not= j$, then $r(i) \not= r(j)$.
\end{assumption}

\begin{remark}
\label{rmk:reward_no_ties_1}
\cref{assp:reward_no_ties} is used in the proofs for \cref{thm:asymp_global_converg_gradient_bandit_sampled_reward}. In particular, ``convergence toward strict one-hot policies'' above \cref{thm:asymp_global_converg_gradient_bandit_sampled_reward} is needed as a result of \cref{assp:reward_no_ties}. We discuss intuition later for establishing the same result without \cref{assp:reward_no_ties}.
\end{remark}

According to \citet[Section 2.8]{sutton2018reinforcement}, the gradient bandit algorithm maintains a softmax distribution over actions $\probability{\left( a_t = a \right)} = \pi_{\theta_t}(a)$ such that $\pi_{\theta_t} = \softmax(\theta_t)$, where
\begin{align}
\label{eq:softmax}
    \pi_{\theta_t}(a) = \frac{ \exp\{ \theta_t(a) \} }{ \sum_{a^\prime \in [K]}{ \exp\{ \theta_t(a^\prime) } \} }, \mbox{ \quad   for all } a \in [K],
\end{align}
and $\theta_t \in \sR^{K}$ is the parameter vector to be updated. % in \cref{alg:gradient_bandit_algorithm_sampled_reward}.
%\bo{have an algorithm box here to illustrate stochastic gradient bandit.}

\begin{algorithm}[ht]
   \caption{Gradient bandit algorithm (without baselines)}
\begin{algorithmic}
   \STATE {\bfseries Input:} initial parameters $\theta_1 \in \sR^K$, learning rate $\eta > 0$.
   \STATE {\bfseries Output:} policies $\pi_{\theta_t} = \softmax(\theta_t)$.
   \WHILE{$t \ge 1$}
   \STATE Sample one action $a_t \sim \pi_{\theta_t}(\cdot)$.
   \STATE Observe one reward sample $R_t(a_t)\sim P_{a_t}$.
   \FOR{all $a \in [K]$}
   \IF{$a = a_t$}
   \STATE $\theta_{t+1}(a) \gets \theta_t(a) + \eta \cdot \left( 1 - \pi_{\theta_t}(a) \right) \cdot R_t(a_t)$.
   \ELSE
   \STATE $\theta_{t+1}(a) \gets \theta_t(a) - \eta \cdot \pi_{\theta_t}(a) \cdot R_t(a_t)$.
   \ENDIF
   \ENDFOR
   \ENDWHILE
\end{algorithmic}
\label{alg:gradient_bandit_algorithm_sampled_reward}
\end{algorithm}

\iffalse 

\begin{update_rule}[Gradient bandit algorithm]
\label{update:gradient_bandit_algorithm_sampled_reward}
For all $t \ge 1$, sample one action $a_t \sim \pi_{\theta_t}(\cdot)$,  observe reward sample $R_t(a_t)\sim R_{a_t}$.
Let $R_t(a)=0$ for all $a \not= a_t$. For all $a \in [K]$,
\begin{align*}
    \theta_{t+1}(a) \gets \theta_t(a) + \begin{cases}
		\eta \cdot \left( 1 - \pi_{\theta_t}(a) \right) \cdot R_t(a), & \text{if } a_t = a\, , \\
		- \eta \cdot \pi_{\theta_t}(a) \cdot R_t(a_t), & \text{otherwise}\,.
	\end{cases}
\end{align*}
\end{update_rule}

\fi

It is obvious that \cref{alg:gradient_bandit_algorithm_sampled_reward} is an instance of stochastic gradient ascent %\citep{sutton2018reinforcement} 
with an unbiased gradient estimator~\citep{nemirovski2009robust}, as shown below for completeness.
\begin{proposition}
\label{prop:gradient_bandit_algorithm_equivalent_to_stochastic_gradient_ascent_sampled_reward}
\cref{alg:gradient_bandit_algorithm_sampled_reward} is equivalent to the following stochastic gradient ascent update on $\pi_{\theta}^\top r$,
\begin{align}
    \theta_{t+1} &\gets  \theta_{t} + \eta \cdot \frac{d \pi_{\theta_t}^\top \hat{r}_t}{d \theta_t} \label{eq:stochastic_gradient_ascent_sampled_reward}\\
    &= \theta_t + \eta \cdot \left(  \diagonalmatrix{(\pi_{\theta_t})} - \pi_{\theta_t} \pi_{\theta_t}^\top \right) \hat{r}_t,
    \label{eq:stochastic_gradient_form}
\end{align}
where $\mathbb{E}_t{ \Big[ \frac{d \pi_{\theta_t}^\top \hat{r}_t }{d \theta_t} \Big] } = \frac{d \pi_{\theta_t}^\top r}{d \theta_t }$, and $\left( \frac{d \pi_{\theta}}{d \theta} \right)^\top = \diagonalmatrix{(\pi_{\theta})} - \pi_{\theta} \pi_{\theta}^\top $ is the Jacobian of $\theta \mapsto \pi_\theta \coloneqq \softmax(\theta)$, 
$\hat{r}_t(a) \coloneqq \frac{ \sI\left\{ a_t = a \right\} }{ \pi_{\theta_t}(a) } \cdot R_t(a)$ for all $a \in [K]$ is the importance sampling (IS) estimator, and we set $R_t(a)=0$ for all $a \not= a_t$. 
\end{proposition}
% \bo{please justify the gradient in Algorithm 1 is unbiased and connected to Eq. 4.}
Based on \cref{prop:gradient_bandit_algorithm_equivalent_to_stochastic_gradient_ascent_sampled_reward}, \citet[Section 2.8]{sutton2018reinforcement} assert that \cref{alg:gradient_bandit_algorithm_sampled_reward} has ``robust convergence properties'' toward stationary points without rigorous justification. 
However, as mentioned in \cref{sec:introduction}, convergence to stationary points is a very weak assertion in a MAB, since this does not guarantee sub-optimal local maxima are avoided. Hence this claim does not assure global convergence or sub-linear regret for the gradient bandit algorithm.

%\section{Standard Stochastic Gradient Analysis}
\section{Preliminary Stochastic Gradient Analysis}
\label{sec:standard_stochastic_gradient_analysis}
%\bo{for all lemmas, theorems, and corollaries, please refer the proofs in main text. }
% \bo{this is not "standard", which usually refers to~\citep{nemirovski2009robust} in convex setting. }
% As shown in \cref{prop:gradient_bandit_algorithm_equivalent_to_stochastic_gradient_ascent_sampled_reward}, \cref{alg:gradient_bandit_algorithm_sampled_reward} is an instance of stochastic gradient ascent. 
In this section, we start with local convergence of stochastic gradient bandit algorithm. The understanding of the behavior of the algorithm involves assessing whether optimization progress is able to overcome the effects of the sampling noise. This trade-off reveals inability of the vanilla analysis and inspires our refined analysis.

To illustrate the basic ideas, we first recall some known results about the form of $\pi_{\theta}^\top r$ and the behavior of \cref{alg:gradient_bandit_algorithm_sampled_reward} and make a preliminary attempt to establish convergence to a stationary point. 
\textbf{First}, $\pi_{\theta}^\top r$ is a $5/2$-smooth function of $\theta \in \sR^K$ \citep[Lemma 2]{mei2020global}, which implies that,
\begin{align*}
\MoveEqLeft
    \pi_{\theta_t}^\top r - \pi_{\theta_{t+1}}^\top r \le - \Big\langle \frac{d \pi_{\theta_t}^\top r}{d \theta_t}, \theta_{t+1} - \theta_t \Big\rangle + \frac{5}{4} \cdot \| \theta_{t+1} - \theta_{t} \|_2^2 \\
    &= - \eta \cdot \Big\langle \frac{d \pi_{\theta_t}^\top r}{d \theta_t}, \frac{d \pi_{\theta_t}^\top \hat{r}_t}{d \theta_t} \Big\rangle + \frac{5}{4} \cdot \eta^2 \cdot \bigg\| \frac{d \pi_{\theta_t}^\top \hat{r}_t}{d \theta_t} \bigg\|_2^2,
\end{align*}
where the last equation follows from \cref{eq:stochastic_gradient_ascent_sampled_reward}. 
\textbf{Second}, as is well known, the on-policy stochastic gradient is unbiased, and its variance / scale is uniformly bounded over all $\pi_{\theta}$.
\begin{proposition}[Unbiased stochastic gradient with bounded variance / scale]
\label{prop:unbiased_stochastic_gradient_bounded_scale_sampled_reward}
Using \cref{alg:gradient_bandit_algorithm_sampled_reward}, we have, for all $t \ge 1$,
\begin{align*}
    \mathbb{E}_t{ \bigg[ \frac{d \pi_{\theta_t}^\top \hat{r}_t }{d \theta_t} \bigg] } = \frac{d \pi_{\theta_t}^\top r}{d \theta_t }, \text{ and } \mathbb{E}_t{ \left[ \bigg\| \frac{d \pi_{\theta_t}^\top \hat{r}_t }{d \theta_t} \bigg\|_2^2 \right] } \le 2 \,  R_{\max}^2,
\end{align*}
where $\EEt{\cdot}$ is on randomness from the on-policy sampling $a_t \sim \pi_{\theta_t}(\cdot)$ and reward sampling $R_t(a_t)\sim P_{a_t}$.
\end{proposition}
Therefore, according to \cref{prop:unbiased_stochastic_gradient_bounded_scale_sampled_reward}, we have,
\begin{align}
\label{eq:smoothness_progress_special}
    \pi_{\theta_t}^\top r - \EEt{ \pi_{\theta_{t+1}}^\top r } \le - \eta \cdot \bigg\| \frac{d \pi_{\theta_t}^\top r}{d \theta_t} \bigg\|_2^2 + \frac{5 \, R_{\max}^2}{2} \cdot \eta^2.
\end{align}

Since the goal is to maximize $\pi_{\theta}^\top r$, we want the first term on the r.h.s. of \cref{eq:smoothness_progress_special} (``progress'') to overcome the second term (``noise'') to ensure that $\EEt{ \pi_{\theta_{t+1}}^\top r } > \pi_{\theta_t}^\top r$. 
Unfortunately, this is not achievable using a constant learning rate $\eta \in \Theta(1)$, since the ``progress'' contains a vanishing term of $\Big\| \frac{d \pi_{\theta_t}^\top r}{d \theta_t} \Big\|_2 \to 0$ as $t \to \infty$ while the ``noise'' term will remain at constant level. 
Therefore, based on this bound, it seems necessary to use a diminishing learning rate $\eta \to 0$ to control the effect of noise for local convergence. In fact, with appropriate learning rate control \citep{robbins1951stochastic, ghadimi2013stochastic,zhang2020global}, it can be shown that minimum gradient norm converge to zero. From \cref{eq:smoothness_progress_special}, by algebra and telescoping, we have,
% \vspace{-4mm}
{% \small
\begin{align*}
\resizebox{\columnwidth}{!}{%
    $\min\limits_{1 \le t \le T}{ \expectation{ \left[ \Big\| \frac{d \pi_{\theta_t}^\top r}{d \theta_t} \Big\|_2^2 \right]} } 
    %&\le \frac{ \sum_{t=1}^{T} \EE{ \pi_{\theta_{t+1}}^\top r } - \EE{ \pi_{\theta_t}^\top r} }{ \sum_{t=1}^{T}{\eta_t} } + \frac{5 \, R_{\max}^2}{2} \cdot \frac{ \sum_{t=1}^{T}{\eta_t^2} }{ \sum_{t=1}^{T}{\eta_t} } \\
    \le \frac{ \EE{ \pi_{\theta_{T+1}}^\top r } - \EE{ \pi_{\theta_1}^\top r} }{ \sum_{t=1}^{T}{\eta_t} } + \frac{5 \, R_{\max}^2}{2}  \frac{ \sum_{t=1}^{T}{\eta_t^2} }{ \sum_{t=1}^{T}{\eta_t} }$.
    }
\end{align*}
}%
Choosing $\eta_t \in \Theta(1/\sqrt{t})$, the r.h.s. of the above inequality is in $\tilde{O}(1/\sqrt{T})$, i.e.,  the minimum gradient norm approaches zero as $T \to \infty$ \citep{ghadimi2013stochastic,zhang2020global}.
%\bo{explain the bound can be used for averaging local convergence and cite here, emphasize not last iteration not global.}
However, the decaying learning rate will slow down the convergence as is seen. Next, we will present our novel technical characterization of the noise in stochastic gradient bandit that can allow us to avoid this choice.

%\section{A New Global Convergence Analysis}
%\label{sec:global_convergence_analysis}

%\subsection{Vanishing Noise: Without Noise Control}

%\section{New Analysis: Vanishing Noise Without Control}
\section{New Analysis: Noise Vanishes Automatically}
\label{sec:vanishing_noise}

As discussed in \cref{sec:standard_stochastic_gradient_analysis}, noise control is at the heart of standard stochastic gradient analysis, and different techniques (entropy or log-barrier regularization, learning rate schemes, momentum) are used to explicitly combat noise in stochastic updates. 
Here we take a different perspective by asking whether the sampling noise will automatically diminish in a way that there is no need to explicitly control it. In particular, we investigate whether the constant order of the second term in \cref{eq:smoothness_progress_special} (``noise'') is accurately characterizing the sampling noise, or whether this bound can be improved.

Note that the ``noise'' constant $5 \, R_{\max}^2$ in \cref{eq:smoothness_progress_special} arises from two quantities: the standard smoothness constant of $5/2$, % \citep{mei2020global}, 
and the variance upper bound of $2 \, R_{\max}^2$ in \cref{prop:unbiased_stochastic_gradient_bounded_scale_sampled_reward}. 
%We will improve both quantities.
It turns out that both of these quantities can be improved.

%\subsubsection{Non-uniform Smoothness: Landscape}
\subsection{Non-uniform Smoothness: Landscape Properties}

The first key observation is a landscape property originally derived in \citep{mei2021leveraging} for true policy gradient settings, which is also applicable for stochastic gradient update. 
%\dale{Wait a second.  We have to be a lot more careful about treating \citep{mei2021leveraging} as someone else's prior work.  We can't be trying to act like those results are new to this paper, or even due to us.} 

\begin{lemma}[Non-uniform smoothness (NS), Lemma 2 in \citep{mei2021leveraging}]
\label{lem:non_uniform_smoothness_softmax_special}
For all $\theta \in \sR^K$, and for all $r\in \sR^K$, the spectral radius of the Hessian matrix $\frac{d^2 \{ \pi_{\theta}^\top r \} }{d {\theta^2 }} \in \sR^{K \times K} $ is upper bounded by $3 \cdot \Big\| \frac{d \pi_{\theta}^\top r}{d {\theta}} \Big\|_2$, i.e., for all $y \in \sR^K$,
\begin{align}
    \left| y^\top \ \frac{d^2 \{ \pi_{\theta}^\top r \} }{d {\theta^2 }} \ y \right| \le 3 \cdot \bigg\| \frac{d \pi_{\theta}^\top r}{d {\theta}} \bigg\|_2 \cdot \| y \|_2^2.
\end{align}
\end{lemma}

It is useful to understand the intuition behind %\cref{lem:non_uniform_smoothness_softmax_special} 
this lemma.
Note that when the PG norm $\vphantom{\hat{\Big|}}\Big\| \frac{d \pi_{\theta}^\top r}{d {\theta}} \Big\|_2$ is small, the policy $\pi_{\theta}$ is close to a one-hot policy, and the objective $\pi_{\theta}^\top r$ has a flat local landscape; ultimately implying that the the Hessian magnitude is upper bounded by the gradient. 
%\citep{mei2021leveraging}. 
%\dale{Why are we explaining the intuition behind someone else's lemma?}

However, directly using \cref{lem:non_uniform_smoothness_softmax_special} remains challenging: Consider two iterates $\theta_t$ and $\theta_{t+1}$.  Then in a Taylor expansion the PG norm of an intermediate point $\theta_\zeta \coloneqq \theta_t + \zeta \cdot (\theta_{t+1} - \theta_t)$,  $\zeta \in [0, 1]$, will appear, which is undesirable since $\zeta$ is unknown. 
Therefore, we require an additional lemma to assert that for a sufficiently small learning rate, the PG norm of $\theta_\zeta$ will be controlled by that of $\theta_t$.
\begin{lemma}[NS between iterates]
\label{lem:non_uniform_smoothness_special_two_iterations}
Using \cref{alg:gradient_bandit_algorithm_sampled_reward} with $\eta \in \big(0, \frac{2}{9 \cdot R_{\max}} \big)$, we have, for all $t \ge 1$,
\begin{align}
    D(\theta_{t+1}, \theta_t) \le \frac{\beta(\theta_t)}{2} \cdot \| \theta_{t+1} - \theta_t \|_2^2,
\end{align}
where $D(\theta_{t+1}, \theta_t)$ is Bregman divergence defined as,
\begin{align*}
    D(\theta_{t+1}, \theta_t) \coloneqq \bigg| ( \pi_{\theta_{t+1}} - \pi_{\theta_t})^\top r - \Big\langle \frac{d \pi_{\theta_t}^\top r}{d \theta_t}, \theta_{t+1} - \theta_t \Big\rangle \bigg|,
\end{align*}
and $\beta(\theta_t) = \frac{6}{2 - 9 \cdot R_{\max} \cdot \eta } \cdot \Big\| \frac{d \pi_{\theta_t}^\top r}{d \theta_t} \Big\|_2$.
\end{lemma}

With the learning rate requirement, 
\cref{lem:non_uniform_smoothness_special_two_iterations} is no longer only a landscape property, but also depends on updates. Using \cref{lem:non_uniform_smoothness_special_two_iterations} rather than standard smoothness, $5 \, R_{\max}^2$ in \cref{eq:smoothness_progress_special} can be replaced with $\beta(\theta_t)$, which implies that the ``noise'' is also vanishing since $\Big\| \frac{d \pi_{\theta_t}^\top r}{d \theta_t} \Big\|_2 \to 0$ as $t \to \infty$. However, with simple constant upper bound of the variance of noise in \cref{prop:unbiased_stochastic_gradient_bounded_scale_sampled_reward}, the progress term in \cref{eq:smoothness_progress_special} still decays faster than the noise term since $\Big\| \frac{d \pi_{\theta_t}^\top r}{d {\theta_t}} \Big\|_2^2 \in o\Big( \Big\| \frac{d \pi_{\theta_t}^\top r}{d {\theta_t}} \Big\|_2\Big)$. % for all sufficiently large $t \ge 1$. 
Unfortunately, this suggests that a decaying learning rate is still necessary to control the noise. Therefore, a further refined analysis of the noise variance is required.

%\subsubsection{Growth Conditions: Softmax Jacobian}
\subsection{Growth Conditions: Softmax Jacobian Behavior}

\begin{figure*}[ht]
%\vskip 0.2in
\begin{center}
\centerline{\includegraphics[width=0.85\linewidth]{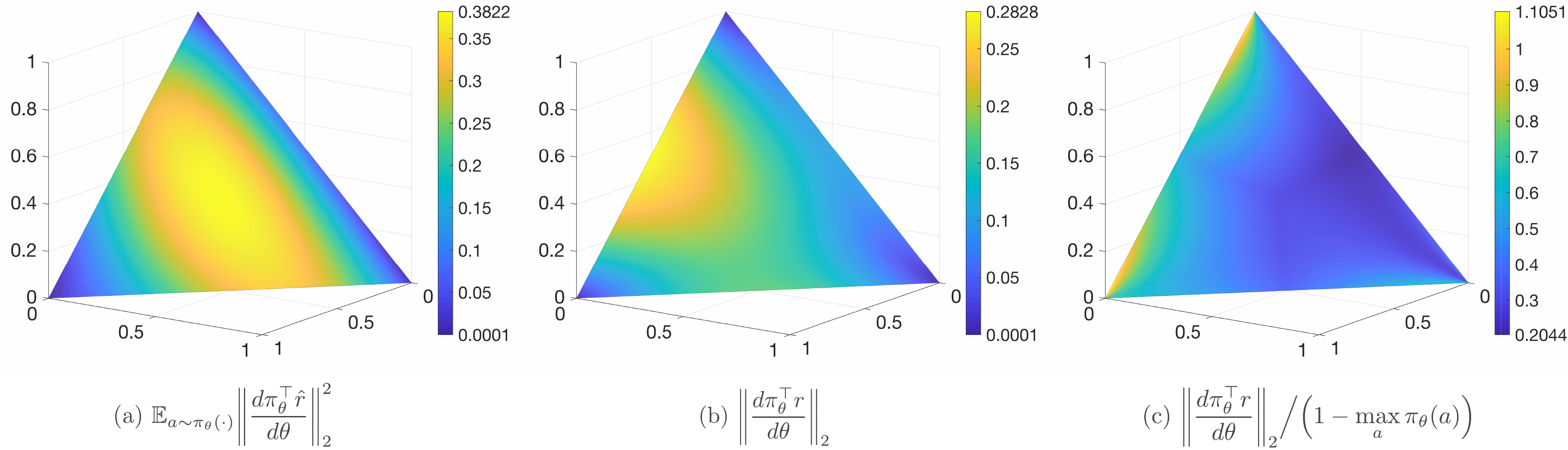}}
\caption{Visualization and intuition for  \cref{lem:strong_growth_conditions_sampled_reward}. \textbf{(a)} Stochastic gradient scale. \textbf{(b)} True gradient norm. \textbf{(c)} The ratio of true gradient norm over $1$ minus largest action probability. Color bars contain minimum and maximum values of corresponding quantities.}
\label{fig:growth_condition_visualization}
\end{center}
\vskip -0.2in
\end{figure*}

Since only \cref{lem:non_uniform_smoothness_softmax_special,lem:non_uniform_smoothness_special_two_iterations} are still insufficient to guarantee progress without learning rate control, we need to develop a more refined variance bound of the noise that was previously unknown for gradient bandit algorithms. 
We first consider an example to intuitively explain why a tighter bound on the noise variance might be possible.
% To see why a tighter bound on the noise variance might be possible, it is helpful to first consider an example.

{\bf Illustration.}
Consider \cref{fig:growth_condition_visualization}(a), which depicts the probability simplex containing all policies $\pi_{\theta}$ over $K = 3$ actions.
%for all $\pi_{\theta}$ in a probability simplex with $K = 3$ actions, 
\cref{fig:growth_condition_visualization}(a) shows the scale of the stochastic gradient
%We use $\pi_{\theta}$ to sample and construct $\hat{r}$, and the stochastic gradient scale 
$\expectation_{a \sim \pi_{\theta}(\cdot)} \Big\| \frac{d \pi_\theta^\top \hat{r}}{d \theta} \Big\|_2^2$, 
%\cref{fig:growth_condition_visualization}(a) shows 
illustrating
that when $\pi_{\theta}$ is close to any corner of the simplex, the stochastic gradient scale becomes close to $0$.
This suggests that the $2 \, R_{\max}^2$ in \cref{prop:unbiased_stochastic_gradient_bounded_scale_sampled_reward} is quite loose and improvable. \cref{fig:growth_condition_visualization}(b) presents a similar visualization for the true gradient norm $\Big\| \frac{d \pi_\theta^\top r}{d \theta} \Big\|_2$, demonstrating a similar behavior to the stochastic gradient scale. 

We formalize this observation by proving that the stochastic gradient scale is controlled by the true gradient norm, significantly improving \cref{prop:unbiased_stochastic_gradient_bounded_scale_sampled_reward}.

\begin{lemma}[Strong growth condition; self-bounding noise property]
\label{lem:strong_growth_conditions_sampled_reward}
Using \cref{alg:gradient_bandit_algorithm_sampled_reward}, we have, for all $t \ge 1$,
\begin{align}
    \mathbb{E}_t{ \left[ \bigg\| \frac{d \pi_{\theta_t}^\top \hat{r}_t}{d \theta_t} \bigg\|_2^2 \right] } 
    \le \frac{8 \cdot R_{\max}^3 \cdot K^{3/2} }{ \Delta^2 } \cdot  \bigg\| \frac{d \pi_{\theta_t}^\top r}{d \theta_t} \bigg\|_2,
\end{align}
where $\Delta \coloneqq \min_{i \not= j}{ | r(i) - r(j) | } $.
\end{lemma}
The proof sketch of \cref{lem:strong_growth_conditions_sampled_reward} is as follows. For any $t \ge 1$, let $k_t$ denote the action with the largest probability, i.e.,
\begin{align}
\label{eq:strong_growth_conditions_sampled_reward_intuition_1}
    k_t \coloneqq \argmax_{a \in [K]}{ \pi_{\theta_t}(a) }.
\end{align}
Note that $1 - \pi_{\theta_t}(k_t)$ characterizes how close $\pi_{\theta_t}$ is to any corner of the probability simplex. 
We first prove that,
\begin{align}
\label{eq:strong_growth_conditions_sampled_reward_intuition_2}
    \mathbb{E}_t{ \left[ \bigg\| \frac{d \pi_{\theta_t}^\top \hat{r}_t }{d \theta_t} \bigg\|_2^2 \right] } &\le 4 \cdot R_{\max}^2 \cdot \left( 1 - \pi_{\theta_t}(k_t) \right),
\end{align}
which formalizes the observation in \cref{fig:growth_condition_visualization}(a). 
Additionally, \cref{fig:growth_condition_visualization}(c) shows that $\Big\| \frac{d \pi_{\theta_t}^\top r}{d \theta_t} \Big\|_2  \big/ \left( 1 - \pi_{\theta_t}(k_t) \right)$ is larger than about $0.2044$, which suggests that the true gradient norm also characterizes a distance of $\pi_{\theta}$ from any corner of probability simplex since it has a ``variance-like'' structure (\cref{eq:softmax_policy_gradient_norm_squared}), which is formalized by proving that,
\begin{align}
\label{eq:strong_growth_conditions_sampled_reward_intuition_3}
     1 - \pi_{\theta_t}(k_t) \le \frac{2 \cdot R_{\max} \cdot K^{3/2} }{ \Delta^2 } \cdot
    \bigg\| \frac{d \pi_{\theta_t}^\top r}{d \theta_t} \bigg\|_2.
\end{align}
Combining \cref{eq:strong_growth_conditions_sampled_reward_intuition_2,eq:strong_growth_conditions_sampled_reward_intuition_3} allows one to establish \cref{lem:strong_growth_conditions_sampled_reward}, verifying the intuitive observation in \cref{fig:growth_condition_visualization}.

From this explanation, whenever the true PG norm $\Big\| \frac{d \pi_{\theta}^\top r}{d {\theta}} \Big\|_2$ is small and $\pi_{\theta}$ is close to a one-hot policy, the stochastic PG norm $\Big\| \frac{d \pi_{\theta}^\top \hat{r}}{d \theta} \Big\|_2$ in \cref{eq:stochastic_gradient_ascent_sampled_reward} will also be small. 
% The deeper reason for why this holds is because of the softmax Jacobian $\left(  \diagonalmatrix{(\pi_{\theta})} - \pi_{\theta} \pi_{\theta}^\top \right)$ is present 
A deeper explanation is that the softmax Jacobian $  \diagonalmatrix{(\pi_{\theta})} - \pi_{\theta} \pi_{\theta}^\top $ is involved in the stochastic PG in~\cref{eq:stochastic_gradient_form}, which cancels and annihilates the unbounded noise in reward estimator $\hat{r}$ arising from the use of importance sampling.

\begin{remark}
The ``strong growth condition'' was first proposed by \citet{schmidt2013fast} and later found to be satisfied in supervised learning with over-parameterized neural networks (NNs) \citep{allen2019convergence}. There, given a dataset $\gD \coloneqq \{ (x_1, y_1), (x_2, y_2), \cdots (x_N, y_N) \}$, the goal is to minimize a composite loss function,
\begin{align}
    \min_{w}{f(w)} = \min_{w} \sum_{i \in [N]}{ f_{i}( w ) }.
\end{align}
Since over-parameterized NNs fit all the data points, i.e., $f(w) = 0$ for some $w$, each individual loss is also $f_i(w) = 0$ since $f_i(w) \ge 0$ for all $i \in [N]$ (e.g., squared loss and cross entropy). This guarantees that when the true gradient $\nabla f(w) = \rvzero$, the stochastic gradient $\nabla f_i(w) = \rvzero$ for all $i \in [N]$. Hence the strong growth conditions are satisfied, and stochastic gradient descent (SGD) attains the same convergence speed as gradient descent (GD) with an over-parameterized NN \citep{allen2019convergence}.
\end{remark}

\begin{remark}
In probability and statistics, the property of a variance bounded by the expectation is also called a self-bounding property, which can be used to get strong variance bounds \citep{mcdiarmid2006concentration,boucheron2009concentration}.
%\dale{Is this right?  Refine if necessary.} %\textcolor{red}{Jincheng: The definition in the paper actually looks weird to me, not something like variance.}
%{\color{blue}Zixin: why do we need this remark?} \textcolor{red}{Alekh mentioend something existing works in probability and statistics. It seems people used this before to show faster convergences in other scenarios.}
\end{remark}

\cref{lem:strong_growth_conditions_sampled_reward}
proves that such strong growth conditions are also satisfied in the stochastic gradient bandit algorithm, but for a different reason. For over-parameterized NNs, since the model fits every data point, zero gradient implies that the stochastic gradient is also zero. Here, in \cref{lem:strong_growth_conditions_sampled_reward}, the strong growth condition alternatively arises because of the landscape; that is, the presence of the softmax Jacobian in the stochastic gradient update annihilates the sampling noise and leads to the strong growth condition being satisfied.
%{\color{blue}Zixin: the sentence "it is the fact that the model..." looks confusing}

%\subsubsection{No Learning Rate Decaying}
\subsection{No Learning Rate Decay}

Finally, from
\cref{lem:non_uniform_smoothness_special_two_iterations,lem:strong_growth_conditions_sampled_reward} we reach the result that expected progress can be guaranteed with a \emph{constant} learning rate for the stochastic gradient bandit update.
\begin{lemma}[Constant learning rates] 
\label{lem:constant_learning_rates_expected_progress_sampled_reward}
Using \cref{alg:gradient_bandit_algorithm_sampled_reward} with $\eta = \frac{\Delta^2}{40 \cdot K^{3/2} \cdot R_{\max}^3 }$, we have, for all $t \ge 1$,
\begin{align*}
    \pi_{\theta_t}^\top r - \EEt{ \pi_{\theta_{t+1}}^\top r } \le - \frac{\Delta^2}{80 \cdot K^{3/2} \cdot R_{\max}^3 } \cdot \bigg\| \frac{d \pi_{\theta_t}^\top r}{d \theta_t} \bigg\|_2^2.
\end{align*}
\end{lemma}
\cref{lem:non_uniform_smoothness_special_two_iterations} indicates that $\{ \pi_{\theta_t}^\top r \}_{t \ge 1}$ is a sub-martingale. 
Since the reward is bounded $r \in [ - R_{\max}, R_{\max } ]$, by Doob's super-martingale convergence theorem we have that, almost surely, $\pi_{\theta_t}^\top r \to c \in [ - R_{\max}, R_{\max } ]$ as $t \to \infty$.
\begin{corollary}%{corAlmostSureConvergenceStochasticNpgValueBaselineSpecial}
\label{cor:almost_sure_convergence_gradient_bandit_algorithms_sampled_reward}
   Using \cref{alg:gradient_bandit_algorithm_sampled_reward}, we have, the sequence $\{ \pi_{\theta_t}^\top r \}_{t\ge 1}$ converges w.p. $1$.
\end{corollary}
Therefore, from \cref{lem:constant_learning_rates_expected_progress_sampled_reward,cor:almost_sure_convergence_gradient_bandit_algorithms_sampled_reward} it follows that $\Big\| \frac{d \pi_{\theta_t}^\top r}{d \theta_t} \Big\|_2 \to 0$ as $t \to 0$ almost surely, which implies that convergence to a stationary point is achieved without a decaying learning rate \citep{robbins1951stochastic}. From \cref{lem:constant_learning_rates_expected_progress_sampled_reward}, using telescoping we immediately have,
\begin{align}
\label{eq:almost_sure_convergence_gradient_bandit_algorithms_sampled_reward_intuition_1}
    \frac{1}{T} \cdot \sum\limits_{1 \le t \le T}{ \expectation{ \left[ \bigg\| \frac{d \pi_{\theta_t}^\top r}{d \theta_t} \bigg\|_2^2 \right]} } \le \frac{ \EE{ \pi_{\theta_{T+1}}^\top r } - \EE{ \pi_{\theta_1}^\top r} }{ C \cdot T },
\end{align}
where $C \coloneqq \frac{\Delta^2}{80 \cdot K^{3/2} \cdot R_{\max}^3 }$. Comparing to \cref{sec:standard_stochastic_gradient_analysis}, \cref{eq:almost_sure_convergence_gradient_bandit_algorithms_sampled_reward_intuition_1} is a stronger result of averaged gradient norm approaches zero, in terms of faster $O(1/T)$ rate and constant learning rate. An interesting observation is that the average gradient convergence rate is independent with the initialization, which is different with the global convergence results in later sections.
%\bo{we can immediately obtain an average gradient converge to zero, as the preliminary results, which could be concrete and added to the corollary. In fact, an interesting observation is that the average gradient convergence rate is independent w.r.t. $C_t$. }
The key reason behind this outcome is that \cref{lem:non_uniform_smoothness_softmax_special,lem:strong_growth_conditions_sampled_reward} establish that the ``noise'' in \cref{eq:smoothness_progress_special} decays on the same order as the ``progress'' $\Big\| \frac{d \pi_{\theta_t}^\top r}{d \theta_t} \Big\|_2^2$, so that a constant learning rate is sufficient for the ``progress'' term to overcome the ``noise'' term (\cref{lem:constant_learning_rates_expected_progress_sampled_reward}).

%\subsection{Global Convergence}
\section{New Global Convergence Results}
\label{sec:global_convergence_analysis}

Given the refined stochastic analysis from \cref{sec:vanishing_noise} we are now ready to establish new global convergence results for the stochastic gradient bandit in \cref{alg:gradient_bandit_algorithm_sampled_reward}.

%\subsubsection{Asymptotic Global Convergence}
\subsection{Asymptotic Global Convergence}

First, note that the true gradient norm takes the following ``variance-like'' expression,
\begin{align}
\label{eq:softmax_policy_gradient_norm_squared}
    \bigg\| \frac{d \pi_{\theta}^\top r}{d \theta} \bigg\|_2^2 = \sum_{a \in [K]}{ \pi_{\theta}(a)^2 \cdot \left( r(a) - \pi_{\theta}^\top r \right)^2 }.    
\end{align}
According to $\Big\| \frac{d \pi_{\theta_t}^\top r}{d \theta_t} \Big\|_2 \to 0$ as $t \to 0$, we have that $\pi_{\theta_t}$ approaches a one-hot policy, i.e., $\pi_{\theta_t}(i) \to 1$ for some $i \in [K]$ as $t \to \infty$. Asymptotic global convergence is then proved by constructing contradictions against the assumption that the algorithm converges to a sub-optimal one-hot policy.
 
%\alekh{This should be a theorem?}
\begin{theorem}[Asymptotic global convergence]%{lemAsympGlobalConverg}
\label{thm:asymp_global_converg_gradient_bandit_sampled_reward}
Using 
% $\eta \in \Theta(1)$ in 
\cref{alg:gradient_bandit_algorithm_sampled_reward}, we have, almost surely, 
\begin{align}
    \pi_{\theta_t}(a^*) \to 1, \text{ as } t \to \infty,
\end{align}
which implies that $\inf_{t \ge 1}{ \pi_{\theta_t}(a^*) } > 0$.
\end{theorem}
It is highly challenging to prove almost surely global convergence because \textbf{(i)} the iteration $\{ \theta_t \}_{t \ge 1}$ is a stochastic process, which it different with the true gradient settings \citep{agarwal2021theory}, and \textbf{(ii)} the iteration $\theta_t \in \sR^K$ is unbounded, which makes Doob's super-martingale convergence results not applicable, unlike \cref{cor:almost_sure_convergence_gradient_bandit_algorithms_sampled_reward}. The strategy and insights of \cref{thm:asymp_global_converg_gradient_bandit_sampled_reward} are as follows. According to \cref{prop:unbiased_stochastic_gradient_bounded_scale_sampled_reward}, we have, for all $a \in [K]$,
\begin{align}
\label{eq:asymp_global_converg_gradient_bandit_sampled_reward_intuition_1}
    \EEt{ \theta_{t+1}(a) } = \theta_t(a) + \eta \cdot \pi_{\theta_t}(a) \cdot \left( r(a) - \pi_{\theta_t}^\top r \right).
\end{align}
Now we suppose that using \cref{alg:gradient_bandit_algorithm_sampled_reward}, there exists a sub-optimal action $i \in [K]$, $r(i) < r(a^*)$, such that,
\begin{align}
\label{eq:asymp_global_converg_gradient_bandit_sampled_reward_intuition_2}
    \pi_{\theta_t}(i) \to 1, \text{ as } t \to \infty,
\end{align}
which implies that,
\begin{align}
\label{eq:asymp_global_converg_gradient_bandit_sampled_reward_intuition_3}
    \pi_{\theta_t}^\top r \to r(i), \text{ as } t \to \infty.
\end{align}
Since $r(i) < r(a^*)$, there exists a ``good'' action set,
\begin{align}
\label{eq:asymp_global_converg_gradient_bandit_sampled_reward_intuition_4}
    \gA^+(i) &\coloneqq \left\{ a^+ \in [K]: r(a^+) > r(i) \right\}.
\end{align}
By \cref{eq:asymp_global_converg_gradient_bandit_sampled_reward_intuition_1,eq:asymp_global_converg_gradient_bandit_sampled_reward_intuition_3,eq:asymp_global_converg_gradient_bandit_sampled_reward_intuition_4}, for all large enough $t \ge 1$,
\begin{align}
\label{eq:asymp_global_converg_gradient_bandit_sampled_reward_intuition_5}
    \EEt{ \theta_{t+1}(a^+) } \ge \theta_t(a^+),
\end{align}
which means a ``good'' action's parameter $\{ \theta_t(a^+) \}_{t \ge 1}$ is a sub-martingale. The major part of the proofs are devoted to the following key results. We have, almost surely,
\begin{align}
\label{eq:asymp_global_converg_gradient_bandit_sampled_reward_intuition_6a}
    \inf_{t \ge 1}{ \theta_t(a^+) } &\ge c_1 > - \infty, \text{ and} \\
\label{eq:asymp_global_converg_gradient_bandit_sampled_reward_intuition_6b}
    \sup_{t \ge 1}{ \theta_t(i) } &\le c_2 < \infty.
\end{align}
\cref{eq:asymp_global_converg_gradient_bandit_sampled_reward_intuition_6a} is non-trivial since an unbounded sub-martingale is not necessarily lower bounded and could have positive probability of approaching negative infinity\footnote{Consider a random walk, $Y_{t+1} = Y_t + Z_t$, where $Z_t = \pm 1$ with equal chance of $1/2$. We have $\EEt{Y_{t+1}} \ge Y_t$. However, we also have $\liminf_{t \to \infty}{Y_t} = - \infty$, and with positive probability, $\inf_{t \ge 1}{Y_t} $ is not lower bounded.}, while \cref{eq:asymp_global_converg_gradient_bandit_sampled_reward_intuition_6b} is non-trivial since the behavior of $\theta_t(i)$ depends on different cases of how many times ``good'' actions are sampled as $t \to \infty$. With the above results, we have,
{\small
\begin{align*}
    \pi_{\theta_t}(i) %&= \frac{ e^{ \theta_t(i) } }{ e^{ \theta_t(i) } + \sum\limits_{a^+ \in \gA^+(i)}{ e^{ \theta_t(a^+) } } + \sum\limits_{a^- \in \gA^-(i)}{ e^{ \theta_t(a^-) } } } \\
    &< \frac{ e^{ \theta_t(i) } }{ e^{ \theta_t(i) } + \sum_{a^+ \in \gA^+(i)}{ e^{ \theta_t(a^+) } } } \qquad \big( e^{ \theta_t(a^-) } > 0 \big) \\
    &\le \frac{ e^{ \theta_t(i) } }{ e^{ \theta_t(i) } + e^{c_1} } \qquad \left( \text{by \cref{eq:asymp_global_converg_gradient_bandit_sampled_reward_intuition_6a}} \right) \\
    &\le \frac{ e^{c_2} }{ e^{c_2}  + e^{c_1} } \qquad \left( \text{by \cref{eq:asymp_global_converg_gradient_bandit_sampled_reward_intuition_6b}} \right) \\
    &\not\to 1,
\end{align*}
}%
which is a contradiction with the assumption of \cref{eq:asymp_global_converg_gradient_bandit_sampled_reward_intuition_2}. Therefore, the asymptotically convergent one-hot policy has to satisfy $r(i) = r(a^*)$, proving \cref{thm:asymp_global_converg_gradient_bandit_sampled_reward}.
The detailed proof is provided in \cref{pf:thm_asymp_global_converg_gradient_bandit_sampled_reward}.

\begin{remark}
\label{rmk:reward_no_ties_2}
As mentioned in \cref{rmk:reward_no_ties_1}, the arguments above \cref{thm:asymp_global_converg_gradient_bandit_sampled_reward}, i.e.,  $\pi_{\theta_t}$ approaches a one-hot policy, is based on \cref{assp:reward_no_ties}. With this result, \cref{thm:asymp_global_converg_gradient_bandit_sampled_reward} proves asymptotic global convergence by contradiction with the assumption of \cref{eq:asymp_global_converg_gradient_bandit_sampled_reward_intuition_2}. In general, without \cref{assp:reward_no_ties}, \cref{eq:softmax_policy_gradient_norm_squared} approaches zero can only imply $\pi_{\theta_t}$ approaches a ``generalized one-hot policy'' (rather than a strict one-hot policy). The definition of generalized one-hot policies can be found in \cref{eq:asymp_global_converg_gradient_bandit_sampled_reward_intermediate_1_b}.
\end{remark}

\begin{remark}
\label{rmk:reward_no_ties_3}
It is true that \cref{eq:softmax_policy_gradient_norm_squared} approaches zero is not enough for showing $\pi_{\theta_t}$ approaches a one-hot policy. However, \cref{alg:gradient_bandit_algorithm_sampled_reward} is special that it is always making one action's probability dominate others' when there are ties. Consider $r \in \sR^K$ with $r(1) = r(2)$. If $\pi_{\theta_t}(1) > \pi_{\theta_t}(2)$, then using the expected softmax PG update $\theta_{t+1}(a) \gets \theta_t(a) + \eta \cdot \pi_{\theta_t}(a) \cdot (r(a) - \pi_{\theta_t}^\top r )$, we have,
{\small
\begin{align*}
    \frac{\pi_{\theta_{t+1}}(1)}{\pi_{\theta_{t+1}}(2)} = \frac{\pi_{\theta_{t}}(1)}{\pi_{\theta_{t}}(2)} \cdot e^{\eta \cdot (\pi_{\theta_t}(1) - \pi_{\theta_t}(2)) \cdot ( r(1) - \pi_{\theta_t}^\top r) } > \frac{\pi_{\theta_{t}}(1)}{\pi_{\theta_{t}}(2)},
\end{align*}
}%
which means that after one update $\pi_{\theta_{t+1}}(1)$ will be even larger than $\pi_{\theta_{t+1}}(2)$. Therefore, we have,
\begin{align}
    \theta_{t+1}(1) - \theta_{t+1}(2) > \theta_{t}(1) - \theta_{t}(2) + C/t,
\end{align}
for some $C > 0$, which implies that,
\begin{align}
    \lim_{t \to \infty}{ \frac{\pi_{\theta_{t}}(1)}{\pi_{\theta_{t}}(2)} } = \lim_{t \to \infty}{ e^{\theta_{t}(1) - \theta_{t}(2)} } = \infty,
\end{align}
i.e., $\pi_{\theta_t}(1) \to 1$ as $t \to \infty$. The above arguments illustrate the ``self-reinforcing'' nature of \cref{alg:gradient_bandit_algorithm_sampled_reward}, such that whenever two (or more) actions have the same mean reward, the update will make only one one of their probabilities larger and larger, until one eventually dominates the others as $t \to \infty$. Generalizing the arguments to stochastic updates will remove \cref{assp:reward_no_ties} in the proofs for \cref{thm:asymp_global_converg_gradient_bandit_sampled_reward}.
\end{remark}

%\subsubsection{Convergence Rate}
\subsection{Convergence Rate}

Given \cref{thm:asymp_global_converg_gradient_bandit_sampled_reward}, a convergence rate result can then be proved using the following inequality \citep{mei2020global}.
\begin{lemma}[Non-uniform \L{}ojasiewicz (N\L{}), Lemma 3 of \citet{mei2020global}]
\label{lem:non_uniform_lojasiewicz_softmax_special}
Assume $r$ has a unique maximizing action $a^*$. Let $\pi^* = \argmax_{\pi \in \Delta}{ \pi^\top r}$. Then, 
\begin{align}
    \bigg\| \frac{d \pi_\theta^\top r}{d \theta} \bigg\|_2 \ge \pi_\theta(a^*) \cdot ( \pi^* - \pi_\theta )^\top r\,.
\end{align}
\end{lemma}

\begin{theorem}[Convergence rate and regret]
\label{thm:convergence_rate_and_regret_gradient_bandit_sampled_reward}
Using \cref{alg:gradient_bandit_algorithm_sampled_reward} with $\eta = \frac{\Delta^2}{40 \cdot K^{3/2} \cdot R_{\max}^3 }$, we have, for all $t \ge 1$,
{\small
\begin{align*}
    \EE{ \left( \pi^* - \pi_{\theta_t} \right)^\top r } &\le \frac{C}{t}, \quad \text{and} \\
    \expectation{ \bigg[ \sum_{t=1}^{T}{ \left( \pi^* - \pi_{\theta_t} \right)^\top r } \bigg] } &\le \min\{ \sqrt{2 \, R_{\max} \, C \, T}, C \, \log{T} + 1  \},
\end{align*}
}%
where $C \coloneqq \frac{80 \cdot K^{3/2} \cdot R_{\max}^3 }{\Delta^2 \cdot \EE{ c^2 }} $, and $c \coloneqq \inf_{t \ge 1}{\pi_{\theta_t}(a^*)} > 0$ is from \cref{thm:asymp_global_converg_gradient_bandit_sampled_reward}.
\end{theorem}
%\alekh{Seems like awkward writing, since $c$ is not just initialization, but trajectory dependent too. I feel that a more transparent way to write is defining $C_\theta = \frac{80 \cdot K^{3/2} \cdot R_{\max}^3 }{\Delta^2 \cdot \EE{ \pi_\theta(a^*)^2 }}$ and state the bound in terms of $\max_{\theta \in \{\theta_1,\ldots,\theta_T\}} C_\theta^2$}
In \cref{thm:convergence_rate_and_regret_gradient_bandit_sampled_reward}, the dependence on $t$ is optimal \citep{lai1985asymptotically}. 
However, the constant can be large, especially for a bad initialization. 
In short, the stochastic gradient algorithm inherits the initialization sensitivity and sub-optimal plateaus from the true policy gradient algorithm with softmax parameterization \citep{mei2020escaping}.
The detailed proof of \cref{thm:convergence_rate_and_regret_gradient_bandit_sampled_reward} is elaborated in \cref{pf:thm_convergence_rate_and_regret_gradient_bandit_sampled_reward}.

\begin{theorem}
\label{thm:ecaping_time_lower_bound}
There exists a problem, initialization $\theta_1 \in \sR^K$, and a positive constant $C > 0$, such that, for all $t \le t_0 \coloneqq \frac{C}{\delta \cdot \pi_{\theta_1}(a^*)}$, we have
\begin{align}
    \EE{ \left( \pi^* - \pi_{\theta_t} \right)^\top r } \ge 0.9 \cdot \Delta,
\end{align}
where $\Delta \coloneqq r(a^*) - \max_{a \not= a^*}{ r(a) }$ is the reward gap of $r$.
\end{theorem}

\section{The Effect of Baselines}
\label{sec:effect_of_baselines}

The original gradient bandit algorithm \citep{sutton2018reinforcement} uses a baseline, which is a slightly modification of \cref{alg:gradient_bandit_algorithm_sampled_reward}. The difference is that $R_t(a_t)$ in \cref{alg:gradient_bandit_algorithm_sampled_reward} is replaced with $R_t(a_t) - B_t$, where $B_t \in \sR$ is an action independent baseline, as shown in \cref{alg:gradient_bandit_algorithm_sampled_reward_baselines} in \cref{sec:proofs_for_gradient_bandit_algorithm_sampled_reward_baselines}.
%\dale{Where is \cref{alg:gradient_bandit_algorithm_sampled_reward_baselines}?}

It is well known that action independent baselines do not introduce bias in the gradient estimate \citep{sutton2018reinforcement}. 
The utility of adding a baseline has typically been considered to be reducing the variance of the gradient estimates \citep{greensmith2004variance,bhatnagar2007incremental,tucker2018mirage,mao2018variance,wu2018variance}. 
Here we show that a similar effect manifests itself through improvements in the strong growth condition.
\begin{lemma}[Strong growth condition, Self-bounding noise property]
\label{lem:strong_growth_conditions_sampled_reward_baselines}
Using \cref{alg:gradient_bandit_algorithm_sampled_reward_baselines}, we have, for all $t \ge 1$,
{\small
\begin{align*}
    \mathbb{E}_t{ \left[ \bigg\| \frac{d \pi_{\theta_t}^\top \big( \hat{r}_t - \hat{b}_t \big) }{d \theta_t} \bigg\|_2^2 \right] } 
    \le \frac{8 \, \bar{R}_{\max}^2 \, R_{\max} \, K^{3/2} }{ \Delta^2 } \cdot  \bigg\| \frac{d \pi_{\theta_t}^\top r}{d \theta_t} \bigg\|_2,
\end{align*}
}%
where $\Delta \coloneqq \min_{i \not= j}{ | r(i) - r(j) | } $, $\hat{b}_t(a) \coloneqq \frac{ \sI\left\{ a_t = a \right\} }{ \pi_{\theta_t}(a) } \cdot B_t \,$ for all $a \in [K]$, and $R_t(a_t) - B_t \in [ - \bar{R}_{\max}, \bar{R}_{\max} ]$.
\end{lemma}
Note that $\bar{R}_{\max}$ denotes the range of $R_t(a_t) - B_t$ after minus a baseline from sampled reward. %\dale{Where is $\bar{R}$ defined?} 
Comparing \cref{lem:strong_growth_conditions_sampled_reward_baselines} to \cref{lem:strong_growth_conditions_sampled_reward} shows that the only difference is that a constant factor of $R_{\max}^2$ is changed to $\bar{R}_{\max}^2$. 
This indicates that a deeper reason for the variance reduction effect of adding a baseline is to reduce the effective reward range. The same improved constant will carry over to all the similar results, including larger constant learning rates, larger progress, and better constants in the global convergence results.

It is worth noting that the effect of baseline differs between algorithms. Here we see that without any baseline \cref{alg:gradient_bandit_algorithm_sampled_reward} already achieves global convergence, while adding a baseline provides constant improvements. For a different natural policy gradient (NPG) method \citep{kakade2002natural,agarwal2021theory}, it is known that without using baselines, on-policy NPG can fail by converging to a sub-optimal deterministic policy \citep{chung2020beyond,mei2021understanding}, while adding a value baseline $\pi_{\theta_t}^\top r$ restores a guarantee of global convergence by reducing the update aggressiveness. % of on-policy NPG \citep{mei2022role}.

%\textcolor{red}{baseline should be average reward of the last part of the policies, or $\pi_{\theta_t}^\top r$?}

%\vspace{-2mm}
\section{Simulation Results}
\label{sec:simulation_results}

\iffalse

\begin{figure*}[ht]
\vskip 0.2in
\begin{center}
\centerline{\includegraphics[width=0.9\linewidth]{figs/gradient_bandit_result_1.pdf}}
\caption{Simulation results. \textbf{(a)} Optimal action's probability. \textbf{(b)} Sub-optimal gap. \textbf{(c)} Sub-optimal gap in $\log$ scale.}
\label{fig:gradient_bandit_simulation_result_1}
\end{center}
\vskip -0.2in
\end{figure*}

\fi

In this section, we conduct several simulations to empirically verify the theoretical findings of asymptotic global convergence and convergence rate. 

%\vspace{-2mm}
\subsection{Asymptotic Global Convergence}

We first design experiments to justify the asymptotic global convergence. We run \cref{alg:gradient_bandit_algorithm_sampled_reward} on stochastic bandit problems with $K = 10$ actions. The mean reward $r$ is random generated in $(0, 1)^K$. 
For each sampled action $a_t \sim \pi_{\theta_t}(\cdot)$, the observed reward is generated as $R_t(a_t) = r(a_t) + Z_t$, where $Z_t \sim \gN(0, 1)$ is Gaussian noise. For the baseline in \cref{alg:gradient_bandit_algorithm_sampled_reward_baselines}, we use average reward as suggested in \citep{sutton2018reinforcement}, i.e., $B_t \coloneqq \sum_{s=1}^{t-1}{ R_s(a_s)/(t-1) }$ for all $t > 1$. The learning rate is $\eta = 0.01$. We use adversarial initialization, such that $\pi_{\theta_1}(a^*) < 1 / K$. 

As shown in \cref{fig:asymptotic_global_convergence_results}, $\pi_{\theta_t}(a^*) \to 1$ eventually, even if its initial value $\pi_{\theta_1}(a^*) $ is very small, verifying the asymptotic global convergence in \cref{thm:asymp_global_converg_gradient_bandit_sampled_reward}. On the other hand, the long plateaus observed in \cref{fig:asymptotic_global_convergence_results} verify \cref{thm:ecaping_time_lower_bound}.

\begin{figure}[ht]
%\vskip -0.1in
%\vspace{-4mm}
\begin{center}
\centerline{\includegraphics[width=\columnwidth]{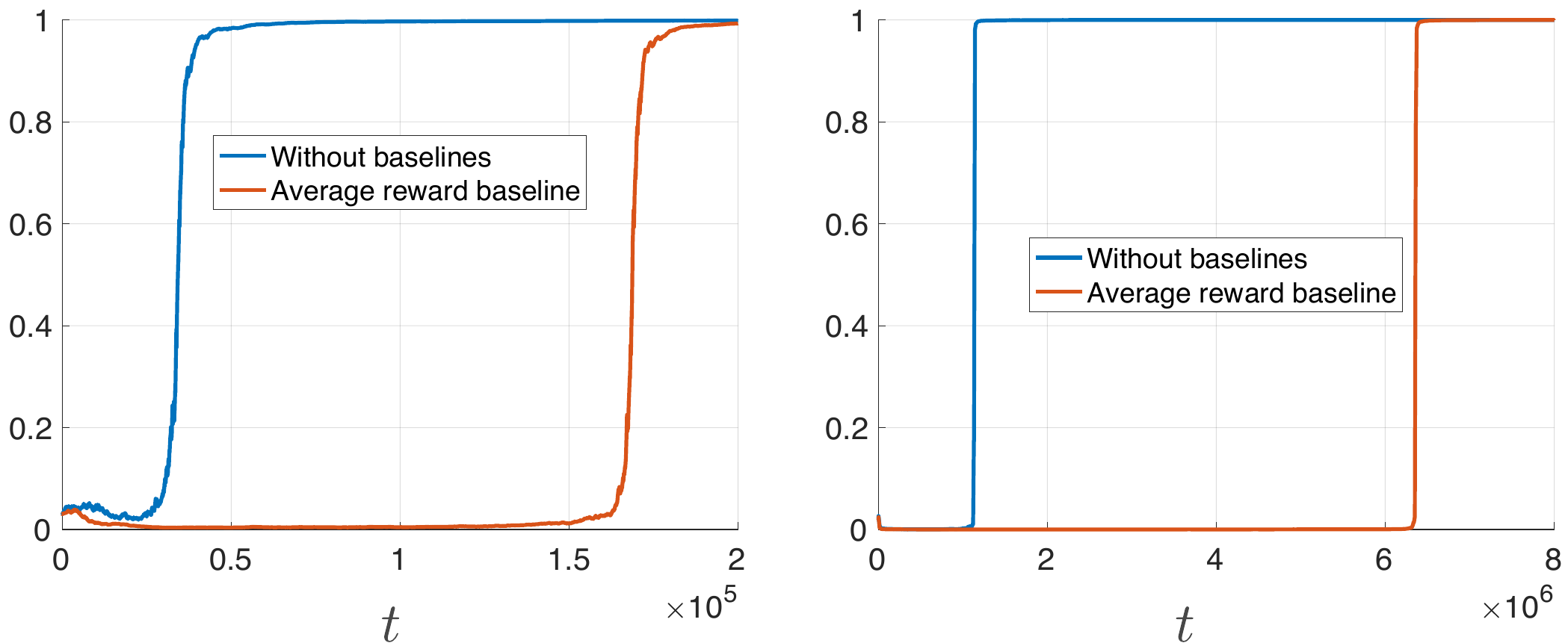}}
\vspace{-3mm}
\caption{Both subfigures show results for $\pi_{\theta_t}(a^*)$. The left is for $\pi_{\theta_1}(a^*) = 0.03$, and the right is for $\pi_{\theta_1}(a^*) = 0.02$.}
\label{fig:asymptotic_global_convergence_results}
\end{center}
\vskip -0.2in
\vspace{-3mm}
\end{figure}

One unexpected observation in \cref{fig:asymptotic_global_convergence_results} is that average reward baselines have worse performances, which is different with \citet{sutton2018reinforcement}. After checking numerical values, we found that since the initialization is bad, a sub-optimal action with $r(i) < r(a^*)$ will be pulled for most of the time, which results in $B_t \approx r(i)$. This implies that when $a^*$ is pulled, $\theta_t(a^*)$ is increased less than without baselines, since the reward gap is also relatively small. Therefore, the average reward baseline might not be a baseline that is universally beneficial, which raises the question to design adaptive baseline, which is out of the scope of this paper, and we leave as our future work.

%\vspace{-2mm}
\subsection{Convergence Rate}\label{subsec:convergence_rate}

We further check the convergence rate empirically in the same problem settings. We use uniform initialization i.e., $\pi_{\theta_1}(a) = 1/K$ for all $ a \in [K]$ and the results are shown in \cref{fig:convergence_rate_results}. 
Each curve is an average from $10$ independent runs, and the total iteration number is $T = 2 \times 10^6$. As shown in
\cref{fig:convergence_rate_results}(b), the slope in $\log$ scale is close to $-1$, which implies that $\log{ \left( \pi^* - \pi_{\theta_t} \right)^\top r} \approx - \log{t} + C$. Equivalently, we have $\left( \pi^* - \pi_{\theta_t} \right)^\top r \approx C^\prime / t$, verifying the $O(1/t)$ convergence rate in \cref{thm:convergence_rate_and_regret_gradient_bandit_sampled_reward}.

\begin{figure}[h]
%\vskip -0.1in
%\vspace{-4mm}
\begin{center}
\centerline{\includegraphics[width=\columnwidth]{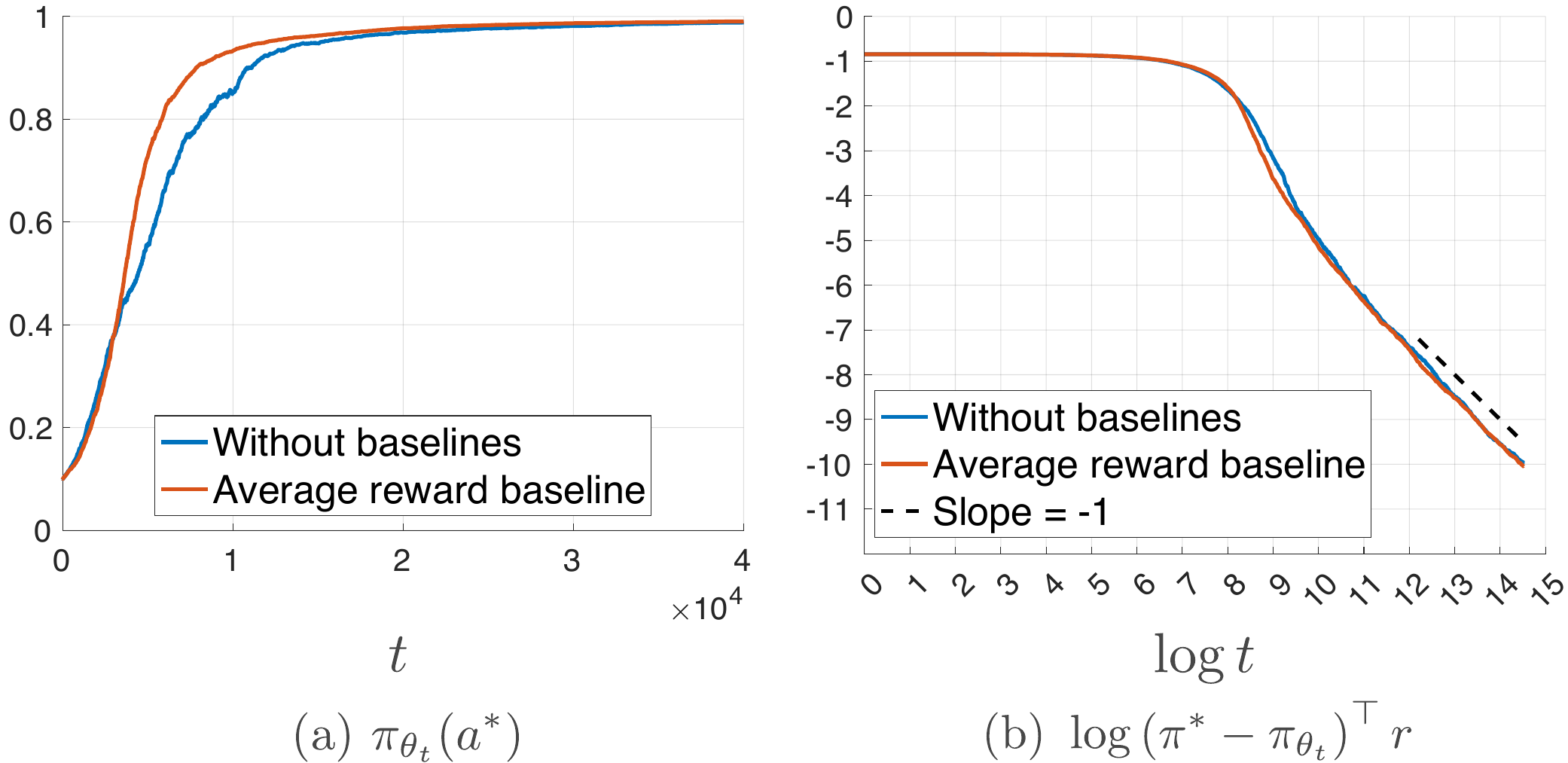}}
\vspace{-3mm}
\caption{Figure (a) shows the optimal action's probability and (b) shows $\log$ sub-optimal gap, which justifies our global convergence rate in~\cref{thm:convergence_rate_and_regret_gradient_bandit_sampled_reward}.}
\label{fig:convergence_rate_results}
\end{center}
\vskip -0.2in
\vspace{-3mm}
\end{figure}

%\vspace{-2mm}
\subsection{Average Gradient Norm Convergence}

In this section, we empirically verify the finite-step convergence rate in terms of average gradient norm. We follow exactly the same experimental settings in~\cref{subsec:convergence_rate}, but evaluate the average gradient norm along the algorithm iterations. We illustrate the results in $\log$-scale in~\cref{fig:avg_grad_norm_results}. It obviously aligned well with our convergence rate in terms of average gradient norm in~\cref{eq:almost_sure_convergence_gradient_bandit_algorithms_sampled_reward_intuition_1}.

\begin{figure}[h]
% \vskip -0.1in
%\vspace{-2mm}
\begin{center}
\centerline{\includegraphics[width=0.6\columnwidth]{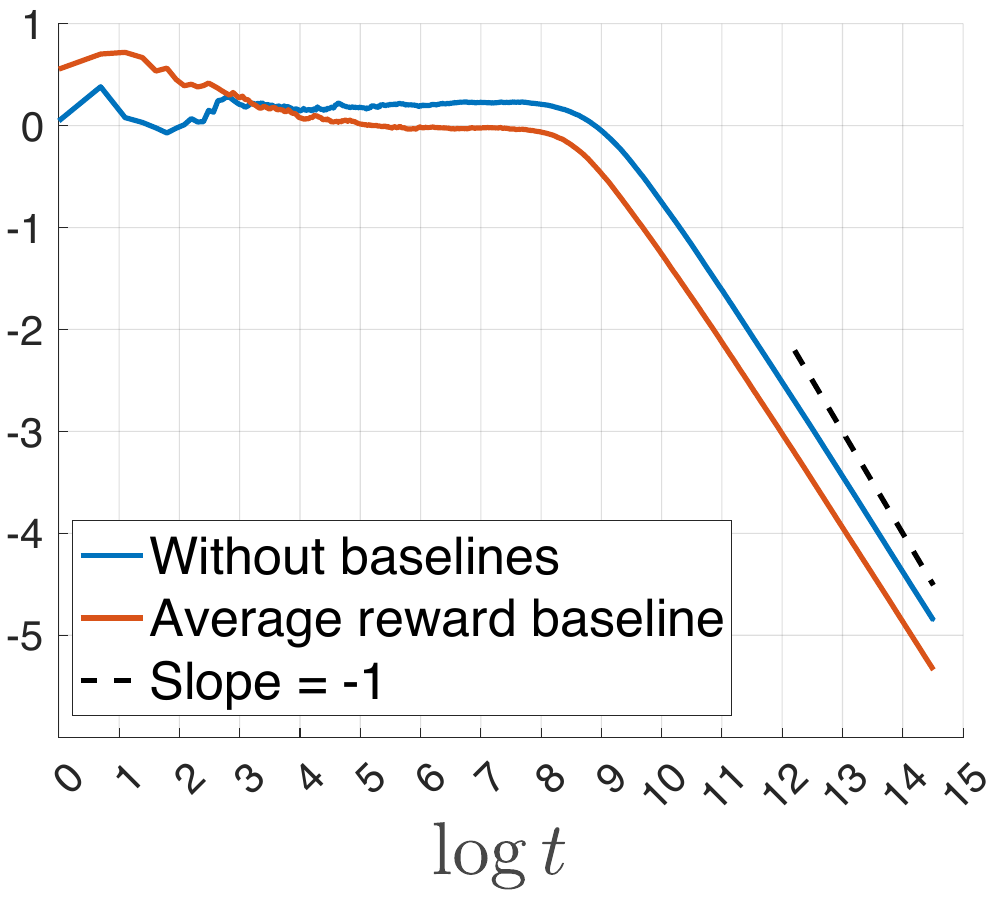}}
\vspace{-5mm}
\caption{Average squared gradient norm $\frac{1}{t} \cdot \sum\limits_{1 \le s \le t}{ \expectation{ \bigg[ \Big\| \frac{d \pi_{\theta_s}^\top r}{d \theta_s} \Big\|_2^2 \bigg]} }$ in $\log$ scale (l.h.s. of \cref{eq:almost_sure_convergence_gradient_bandit_algorithms_sampled_reward_intuition_1}).}
\label{fig:avg_grad_norm_results}
\end{center}
\vskip -0.2in
\vspace{-5mm}
\end{figure}

% \section{Further Discussion}
\section{Boltzmann Exploration}
\label{sec:further_discussions}

%\textcolor{red}{ask csaba if thompson sampling should be discussed here.\\ EXP3?}

The softmax parameterization used in gradient bandit algorithms is also called Boltzmann distribution \citep{sutton2018reinforcement}, based on which a classic algorithm EXP3 uses $O(1/\sqrt{t})$ learning rate and achieves a $O(1/\sqrt{t})$ rate \citep{auer2002nonstochastic}. The Botlzmann distribution has also been used in other policy gradient based algorithms. For example, \citet{lan2022block} show that mirror decent (MD) or NPG with strongly convex regularizers, increasing \emph{batch sizes} and $O(1/t)$ learning rates achieves a $O(\log{t} /t) $ rate. The convergence in~\citep{lan2022block} heavily relies on batch observation for an accurate estimation of full gradient approximation, which is impossible in stochastic bandit setting, and thus not applicable. 

There are also existing results revealing the weakness of Boltzmann distribution. \citet{cesa2017boltzmann} show that without count based bonuses, ``Boltzmann exploration done wrong''. In particular, there exists a $2$-armed stochastic bandit problem with rewards bounded in $[0, 1]$, when using $\probability{\left( a_t = a \right)} \propto \exp\{ \eta_t \cdot \hat{\mu}_{t, a} \}$, where  
\begin{align*}
    \hat{\mu}_{t, a} \coloneqq \frac{ \sum_{s=1}^{t}{ \sI\left\{ a_s = a \right\} \cdot R_s(a) } }{  \sum_{s=1}^{t}{ \sI\left\{ a_s = a \right\} } }
\end{align*}
is the empirical mean estimator for $r(a)$, with $\eta_t > 2 \, \log{t} $ for all $t \ge 1$, would incur linear regret $\Omega(T)$. Instead of the aggressive update of parameters in softmax policy in~\citep{cesa2017boltzmann}, the stochastic gradient bandit can be understood as a better way for parameter updates (with weights diminishing if the action is not selected) to ensure global convergence. 
%The main reason for ``Boltzmann exploration done wrong'' is that this update is very aggressive such that sub-optimal actions have chance to be pull forever. This can be seen by considering the following thought experiments. Let $K = 2$, and $r(1) = 1.0$, and $r(2) = 0.5$. Let action $2$ be pulled for the first $t$ iterations. Then we have, $\hat{\mu}_{t, 1} = 0$ and $\hat{\mu}_{t, 2} \approx 0.5$. Using Boltzmann exploration with $\eta_t = 2 \, \log{t}$, we have $\probability{\left( a_t = a \right)} \approx \frac{t}{t+1}$ 

\citet{cesa2017boltzmann} also claim that the Boltzmann exploration is equivalent to the rule selecting
\begin{align*}
    a_{t+1} = \argmax_{a\in[K]} ( \hat{\mu}_{t, a} + Z_{t,a} ),
\end{align*}
which is the widely used ``Gumbel-Softmax'' trick \citep{jang2016categorical}, 
where $Z_{t,a}$ is a Gumbel random variable independent for all $a\in[K]$. \citet{cesa2017boltzmann} show a $O(\log T)$ regret when replacing $Z_{t,a}$ by $\beta_{t,a} \cdot Z_{t,a}$, where $\beta_{t,a}$ is determined by count information $\sum_{s=1}^t\sI\{a_t=a\}$.
This inspires us that we may incorporate other techniques into gradient bandit algorithm and further improve it, especially for the poor initialization and problem dependent constant in \cref{thm:asymp_global_converg_gradient_bandit_sampled_reward} as also observed in \cref{fig:asymptotic_global_convergence_results}. 

\vspace{-3mm}
\section{Conclusions}
\label{sec:conclusions}

This work provides the first global convergence result for the gradient bandit algorithms \citep{sutton2018reinforcement} using constant learning rates. The main technical finding is that the noise in stochastic gradient updates automatically vanishes such that noise control is unnecessary for global convergence. This work uncover a new understanding of stochastic gradient itself manages to achieve ``weak exploraton'' in the sense that the distribution over the arms almost surely concentrates  asymptotically on a globally optimal action. One important future direction is to improve stochastic gradient to achieve ``strong'' exploration with finite-time optimal rates. Another direction of interest is to generalize the ideas and techniques to reinforcement learning.

\section*{Acknowledgements}

The authors would like to thank anonymous reviewers for their valuable comments. Jincheng Mei would like to thank Ramki Gummadi for providing feedback on a draft of this manuscript. Csaba Szepesv\'ari, Dale Schuurmans and Zixin Zhong gratefully acknowledge funding from the Canada CIFAR AI Chairs Program, Amii and NSERC.

\bibliography{icml_ref}
\bibliographystyle{icml2023}

%%%%%%%%%%%%%%%%%%%%%%%%%%%%%%%%%%%%%%%%%%%%%%%%%%%%%%%%%%%%%%%%%%%%%%%%%%%%%%%
%%%%%%%%%%%%%%%%%%%%%%%%%%%%%%%%%%%%%%%%%%%%%%%%%%%%%%%%%%%%%%%%%%%%%%%%%%%%%%%
% APPENDIX
%%%%%%%%%%%%%%%%%%%%%%%%%%%%%%%%%%%%%%%%%%%%%%%%%%%%%%%%%%%%%%%%%%%%%%%%%%%%%%%
%%%%%%%%%%%%%%%%%%%%%%%%%%%%%%%%%%%%%%%%%%%%%%%%%%%%%%%%%%%%%%%%%%%%%%%%%%%%%%%
\newpage
\appendix
\onecolumn
% \etocdepthtag.toc{mtappendix}
% \etocsettagdepth{mtmainpaper}{none}
% \etocsettagdepth{mtappendix}{subsubsection}
% \begingroup
% \parindent=0em
% \etocsettocstyle{\rule{\linewidth}{\tocrulewidth}\vskip0.5\baselineskip}{\rule{\linewidth}{\tocrulewidth}}
%\etocsettocstyle{}{}
% \tableofcontents 
% \endgroup

% \section{Proofs}

\section{Proofs for \cref{alg:gradient_bandit_algorithm_sampled_reward}}
\label{sec:proofs_for_gradient_bandit_algorithm_sampled_reward}

\textbf{\cref{prop:gradient_bandit_algorithm_equivalent_to_stochastic_gradient_ascent_sampled_reward}.} \cref{alg:gradient_bandit_algorithm_sampled_reward} is equivalent to the following stochastic gradient ascent update,
\begin{align}
    \theta_{t+1} &\gets  \theta_{t} + \eta \cdot \frac{d \pi_{\theta_t}^\top \hat{r}_t}{d \theta_t} \\
    &= \theta_t + \eta \cdot \left(  \diagonalmatrix{(\pi_{\theta_t})} - \pi_{\theta_t} \pi_{\theta_t}^\top \right) \hat{r}_t,
\end{align}
where $\mathbb{E}_t{ \Big[ \frac{d \pi_{\theta_t}^\top \hat{r}_t }{d \theta_t} \Big] } = \frac{d \pi_{\theta_t}^\top r}{d \theta_t }$, and $\left( \frac{d \pi_{\theta}}{d \theta} \right)^\top = \diagonalmatrix{(\pi_{\theta})} - \pi_{\theta} \pi_{\theta}^\top $ is the Jacobian of $\theta \mapsto \pi_\theta \coloneqq \softmax(\theta)$, and $\hat{r}_t(a) \coloneqq \frac{ \sI\left\{ a_t = a \right\} }{ \pi_{\theta_t}(a) } \cdot R_t(a)$ for all $a \in [K]$ is the importance sampling (IS) estimator, and we set $R_t(a)=0$ for all $a \not= a_t$.
\begin{proof}
Using the definition of softmax Jacobian and $\hat{r}_t$, we have, for all $a \in [K]$,
\begin{align}
    \theta_{t+1}(a) &\gets \theta_t(a) + \eta \cdot \pi_{\theta_t}(a) \cdot \left( \hat{r}_t(a) - \pi_{\theta_t}^\top \hat{r}_t \right) \\
    &= \theta_t(a) + \eta \cdot \pi_{\theta_t}(a) \cdot \left( \hat{r}_t(a) - R_t(a_t) \right) \\
    &= \theta_t(a) + \begin{cases}
		\eta \cdot \left( 1 - \pi_{\theta_t}(a) \right) \cdot R_t(a), & \text{if } a_t = a\, , \\
		- \eta \cdot \pi_{\theta_t}(a) \cdot R_t(a_t), & \text{otherwise}\,.
    \end{cases}
    \qedhere
\end{align}
\end{proof}

\textbf{\cref{prop:unbiased_stochastic_gradient_bounded_scale_sampled_reward}} (Unbiased stochastic gradient with bounded variance / scale)\textbf{.}
Using \cref{alg:gradient_bandit_algorithm_sampled_reward}, we have, for all $t \ge 1$,
\begin{align}
\label{eq:unbiased_stochastic_gradient_bounded_scale_sampled_reward_result_1_appendix}
    &\mathbb{E}_t{ \bigg[ \frac{d \pi_{\theta_t}^\top \hat{r}_t }{d \theta_t} \bigg] } = \frac{d \pi_{\theta_t}^\top r}{d \theta_t }, \\
\label{eq:unbiased_stochastic_gradient_bounded_scale_sampled_reward_result_2_appendix}
    &\mathbb{E}_t{ \left[ \bigg\| \frac{d \pi_{\theta_t}^\top \hat{r}_t }{d \theta_t} \bigg\|_2^2 \right] } \le 2 \, R_{\max}^2,
\end{align}
where $\EEt{\cdot}$ is on randomness from the on-policy sampling $a_t \sim \pi_{\theta_t}(\cdot)$ and reward sampling $R_t(a_t)\sim P_{a_t}$.
\begin{proof}
\textbf{First part, \cref{eq:unbiased_stochastic_gradient_bounded_scale_sampled_reward_result_1_appendix}.} For all action $a \in [K]$, the true softmax PG is,
\begin{align}
\label{eq:unbiased_stochastic_gradient_bounded_scale_sampled_reward_result_1_intermediate_1}
    \frac{d \pi_{\theta_t}^\top r}{d \theta_t(a)} = \pi_{\theta_t}(a) \cdot \left( r(a) - \pi_{\theta_t}^\top r \right).
\end{align}
For all $a \in [K]$, the stochastic softmax PG is,
\begin{align}
\label{eq:unbiased_stochastic_gradient_bounded_scale_sampled_reward_result_1_intermediate_2}
    \frac{d \pi_{\theta_t}^\top \hat{r}_t}{d \theta_t(a)} &= \pi_{\theta_t}(a) \cdot \left( \hat{r}_t(a) - \pi_{\theta_t}^\top \hat{r}_t \right) \\
    &= \pi_{\theta_t}(a) \cdot \left( \hat{r}_t(a) - R_t(a_t) \right) \\
    &= \left( \sI\left\{ a_t = a \right\} - \pi_{\theta_t}(a) \right) \cdot R_t(a_t).
\end{align}
For the sampled action $a_t$, we have,
\begin{align}
\label{eq:unbiased_stochastic_gradient_bounded_scale_sampled_reward_result_1_intermediate_3}
    \expectation_{ R_t(a_t) \sim P_{a_t} }{ \bigg[ \frac{d \pi_{\theta_t}^\top \hat{r}_t}{d \theta_t(a_t)} \bigg] } &= \expectation_{ R_t(a_t) \sim P_{a_t} }{ \Big[ \left( 1 - \pi_{\theta_t}(a_t) \right) \cdot R_t(a_t) \Big] } \\
    &= \left( 1 - \pi_{\theta_t}(a_t) \right) \cdot \expectation_{ R_t(a_t) \sim P_{a_t} }{ \Big[ R_t(a_t) \Big] } \\
    &= \left( 1 - \pi_{\theta_t}(a_t) \right) \cdot r(a_t).
\end{align}
For any other not sampled action $a \not= a_t$, we have,
\begin{align}
\label{eq:unbiased_stochastic_gradient_bounded_scale_sampled_reward_result_1_intermediate_4}
    \expectation_{ R_t(a_t) \sim P_{a_t} }{ \bigg[ \frac{d \pi_{\theta_t}^\top \hat{r}_t}{d \theta_t(a)} \bigg] } &= \expectation_{ R_t(a_t) \sim P_{a_t} }{ \Big[ - \pi_{\theta_t}(a) \cdot R_t(a_t) \Big] } \\
    &= - \pi_{\theta_t}(a) \cdot \expectation_{ R_t(a_t) \sim P_{a_t} }{ \Big[ R_t(a_t) \Big] } \\
    &= - \pi_{\theta_t}(a) \cdot r(a_t).
\end{align}
Combing \cref{eq:unbiased_stochastic_gradient_bounded_scale_sampled_reward_result_1_intermediate_3,eq:unbiased_stochastic_gradient_bounded_scale_sampled_reward_result_1_intermediate_4}, we have, for all $a \in [K]$,
\begin{align}
\label{eq:unbiased_stochastic_gradient_bounded_scale_sampled_reward_result_1_intermediate_5}
    \expectation_{ R_t(a_t) \sim P_{a_t} }{ \bigg[ \frac{d \pi_{\theta_t}^\top \hat{r}_t}{d \theta_t(a)} \bigg] } = \left( \sI\left\{ a_t = a \right\} - \pi_{\theta_t}(a) \right) \cdot r(a_t).
\end{align}
Taking expectation over $a_t \sim \pi_{\theta_t}(\cdot)$, we have,
\begin{align}
\label{eq:unbiased_stochastic_gradient_bounded_scale_sampled_reward_result_1_intermediate_6}
    \mathbb{E}_t{ \bigg[ \frac{d \pi_{\theta_t}^\top \hat{r}_t }{d \theta_t(a)} \bigg] } &= \probability{\left( a_t = a \right) } \cdot \expectation_{ R_t(a_t) \sim P_{a_t} }{ \bigg[ \frac{d \pi_{\theta_t}^\top \hat{r}_t}{d \theta_t(a)} \ \Big| \ a_t = a \bigg] } + \probability{\left( a_t \not= a \right) } \cdot \expectation_{ R_t(a_t) \sim P_{a_t} }{ \bigg[ \frac{d \pi_{\theta_t}^\top \hat{r}_t}{d \theta_t(a)} \ \Big| \ a_t \not= a \bigg] } \\
    &= \pi_{\theta_t}(a) \cdot \left( 1 - \pi_{\theta_t}(a) \right) \cdot r(a) + \sum_{a^\prime \not= a}{ \pi_{\theta_t}(a^\prime) \cdot \left( - \pi_{\theta_t}(a) \right) \cdot r(a^\prime) } \\
    &= \pi_{\theta_t}(a) \cdot \sum_{a^\prime \not= a}{ \pi_{\theta_t}(a^\prime) \cdot \left( r(a) - r(a^\prime) \right) } \\
    &= \pi_{\theta_t}(a) \cdot \left( r(a) - \pi_{\theta_t}^\top r \right).
\end{align}
Combining \cref{eq:unbiased_stochastic_gradient_bounded_scale_sampled_reward_result_1_intermediate_1,eq:unbiased_stochastic_gradient_bounded_scale_sampled_reward_result_1_intermediate_6}, we have, for all $a \in [K]$,
\begin{align}
\label{eq:unbiased_stochastic_gradient_bounded_scale_sampled_reward_result_1_intermediate_7}
    \mathbb{E}_t{ \bigg[ \frac{d \pi_{\theta_t}^\top \hat{r}_t }{d \theta_t(a)} \bigg] } = \frac{d \pi_{\theta_t}^\top r}{d \theta_t(a)},
\end{align}
which implies \cref{eq:unbiased_stochastic_gradient_bounded_scale_sampled_reward_result_1_appendix} since $a \in [K]$ is arbitrary.

\textbf{Second part, \cref{eq:unbiased_stochastic_gradient_bounded_scale_sampled_reward_result_2_appendix}.} The squared stochastic PG norm is,
\begin{align}
\label{eq:unbiased_stochastic_gradient_bounded_scale_sampled_reward_result_2_intermediate_1}
    \bigg\| \frac{d \pi_{\theta_t}^\top \hat{r}_t }{d \theta_t} \bigg\|_2^2 &= \sum_{a \in [K]}{ \left( \frac{d \pi_{\theta_t}^\top \hat{r}_t}{d \theta_t(a)} \right)^2 } \\
    &= \sum_{a \in [K]}{ \left( \sI\left\{ a_t = a \right\} - \pi_{\theta_t}(a) \right)^2 \cdot R_t(a_t)^2 } \qquad \left( \text{by \cref{eq:unbiased_stochastic_gradient_bounded_scale_sampled_reward_result_1_intermediate_2}} \right) \\
    &\le R_{\max}^2 \cdot \sum_{a \in [K]}{ \left( \sI\left\{ a_t = a \right\} - \pi_{\theta_t}(a) \right)^2 } \qquad \left( \text{by \cref{eq:true_mean_reward_expectation_bounded_sampled_reward}} \right) \\
    &= R_{\max}^2 \cdot \bigg[ \left( 1 - \pi_{\theta_t}(a_t) \right)^2 + \sum_{a \not= a_t}{ \pi_{\theta_t}(a)^2 } \bigg] \\
    &\le R_{\max}^2 \cdot \bigg[ \left( 1 - \pi_{\theta_t}(a_t) \right)^2 + \Big( \sum_{a \not= a_t}{ \pi_{\theta_t}(a) } \Big)^2 \bigg] \qquad \left( \left\| x \right\|_2 \le \left\| x \right\|_1 \right) \\
    &= 2 \cdot R_{\max}^2 \cdot \left( 1 - \pi_{\theta_t}(a_t) \right)^2.
\end{align}
Therefore, we have, for all $a \in [K]$, conditioning on $a_t = a$,
\begin{align}
\label{eq:unbiased_stochastic_gradient_bounded_scale_sampled_reward_result_2_intermediate_2}
    \Bigg[ \bigg\| \frac{d \pi_{\theta_t}^\top \hat{r}_t }{d \theta_t} \bigg\|_2^2 \ \Big| \ a_t = a \Bigg] \le 2 \cdot R_{\max}^2 \cdot \left( 1 - \pi_{\theta_t}(a) \right)^2.
\end{align}
Taking expectation over $a_t \sim \pi_{\theta_t}(\cdot)$, we have,
\begin{align}
\label{eq:unbiased_stochastic_gradient_bounded_scale_sampled_reward_result_2_intermediate_3}
    \mathbb{E}_t{ \left[ \bigg\| \frac{d \pi_{\theta_t}^\top \hat{r}_t }{d \theta_t} \bigg\|_2^2 \right] } &= \sum_{a \in [K]}{ \probability{\left( a_t = a \right) } \cdot \Bigg[ \bigg\| \frac{d \pi_{\theta_t}^\top \hat{r}_t }{d \theta_t} \bigg\|_2^2 \ \Big| \ a_t = a \Bigg] } \\
    &\le \sum_{a \in [K]}{ \pi_{\theta_t}(a) \cdot 2 \cdot R_{\max}^2 \cdot \left( 1 - \pi_{\theta_t}(a) \right)^2 } \\
    &\le 2 \cdot R_{\max}^2 \cdot \sum_{a \in [K]}{ \pi_{\theta_t}(a)} \qquad \left( \pi_{\theta_t}(a) \in (0, 1) \text{ for all } a \in [K] \right) \\
    &= 2 \,  R_{\max}^2. \qedhere
\end{align}
\end{proof}

\textbf{\cref{lem:non_uniform_smoothness_softmax_special}} (Non-uniform smoothness (NS), \citet[Lemma 2]{mei2021leveraging})\textbf{.}
For all $\theta \in \sR^K$, the spectral radius of Hessian matrix $\frac{d^2 \{ \pi_{\theta}^\top r \} }{d {\theta^2 }} \in \sR^{K \times K} $ is upper bounded by $3 \cdot \Big\| \frac{d \pi_{\theta}^\top r}{d {\theta}} \Big\|_2$, i.e., for all $y \in \sR^K$,
\begin{align}
    \left| y^\top \ \frac{d^2 \{ \pi_{\theta}^\top r \} }{d {\theta^2 }} \ y \right| \le 3 \cdot \bigg\| \frac{d \pi_{\theta}^\top r}{d {\theta}} \bigg\|_2 \cdot \| y \|_2^2.
\end{align}
\begin{proof}
See the proof in \citet[Lemma 2]{mei2021leveraging}. We include a  proof for completeness.    

Let $S \coloneqq S(r,\theta)\in \R^{K \times K}$ be 
the second derivative of the map $\theta \mapsto \pi_\theta^\top r$. Denote $H(\pi_\theta) \coloneqq \diagonalmatrix(\pi_\theta) - \pi_\theta \pi_\theta^\top$ as the softmax Jacobian. By definition we have,
\begin{align}
    S &= \frac{d }{d \theta } \left\{ \frac{d \pi_\theta^\top r}{d \theta} \right\} \\
    &= \frac{d }{d \theta } \left\{ H(\pi_\theta) r \right\}  \\
    &= \frac{d }{d \theta } \left\{ ( \diagonalmatrix(\pi_\theta) - \pi_\theta \pi_\theta^\top) r \right\}.
\end{align}
Continuing with our calculation fix $i, j \in [K]$. Then, 
\begin{align}
    S_{(i, j)} &= \frac{d \{ \pi_\theta(i) \cdot  ( r(i) - \pi_\theta^\top r ) \} }{d \theta(j)} \\
    &= \frac{d \pi_\theta(i) }{d \theta(j)} \cdot ( r(i) - \pi_\theta^\top r ) + \pi_\theta(i) \cdot \frac{d \{ r(i) - \pi_\theta^\top r \} }{d \theta(j)} \\
    &= (\delta_{ij} \cdot \pi_\theta(j) -  \pi_\theta(i) \cdot \pi_\theta(j) ) \cdot ( r(i) - \pi_\theta^\top r ) - \pi_\theta(i) \cdot ( \pi_\theta(j) \cdot r(j) - \pi_\theta(j) \cdot \pi_\theta^\top r ) \\
    &= \delta_{ij} \cdot \pi_\theta(j) \cdot ( r(i) - \pi_\theta^\top r ) -  \pi_\theta(i) \cdot \pi_\theta(j) \cdot ( r(i) - \pi_\theta^\top r ) - \pi_\theta(i) \cdot \pi_\theta(j) \cdot ( r(j) -  \pi_\theta^\top r ),
\end{align}
where $\delta_{ij}$ is the Kronecker's $\delta$-function defined as,
\begin{align}
    \delta_{ij} = \begin{cases}
		1, & \text{if } i = j, \\
		0, & \text{otherwise}.
	\end{cases}
\end{align}
To show the bound on 
the spectral radius of $S$, pick $y \in \sR^K$. Then,
\begin{align}
\label{eq:non_uniform_smoothness_softmax_special_Hessian_spectral_radius}
\MoveEqLeft
    \left| y^\top S y \right| = \bigg| \sum\limits_{i=1}^{K}{ \sum\limits_{j=1}^{K}{ S_{(i,j)} \cdot y(i) \cdot y(j)} } \bigg| \\
    &= \bigg| \sum_{i}{ \pi_\theta(i) \cdot ( r(i) - \pi_\theta^\top r ) \cdot y(i)^2 } - 2 \cdot \sum_{i} \pi_\theta(i) \cdot ( r(i) - \pi_\theta^\top r ) \cdot y(i) \cdot \sum_{j} \pi_\theta(j) \cdot y(j) \bigg| \\
    &= \left| \left( H(\pi_\theta) r \right)^\top \left( y \odot y \right) - 2 \cdot \left( H(\pi_\theta) r \right)^\top y \cdot \left( \pi_\theta^\top y \right) \right| \\
    &\le \left| \left( H(\pi_\theta) r \right)^\top \left( y \odot y \right) \right| + 2 \cdot \left| \left( H(\pi_\theta) r \right)^\top y \right| \cdot \left| \pi_\theta^\top y \right| \qquad \left( \text{triangle inequality} \right) \\
    &\le \left\| H(\pi_\theta) r \right\|_\infty \cdot \left\| y \odot y \right\|_1 + 2 \cdot \left\| H(\pi_\theta) r \right\|_2 \cdot \left\| y \right\|_2 \cdot \left\| \pi_\theta \right\|_1 \cdot \left\| y \right\|_\infty \qquad \left( \text{H{\" o}lder's inequality} \right) \\
    &\le 3 \cdot \left\| H(\pi_\theta) r \right\|_2 \cdot \left\| y \right\|_2^2,
\end{align}
where $\odot$ is Hadamard (component-wise) product, and the last inequality uses $\left\| H(\pi_\theta) r \right\|_\infty \le \left\| H(\pi_\theta) r \right\|_2$, $\| y \odot y \|_1 = \| y \|_2^2$, $\| \pi_\theta \|_1 = 1$, and $\| y \|_\infty \le \| y \|_2$. Therefore, we have,
\begin{align}
    \left| y^\top S y \right| &\le 3 \cdot \left\| H(\pi_\theta) r \right\|_2 \cdot \left\| y \right\|_2^2 \qquad \left( \text{by \cref{eq:non_uniform_smoothness_softmax_special_Hessian_spectral_radius}} \right) \\
    &= 3 \cdot \left\| \left( \diagonalmatrix(\pi_\theta) - \pi_\theta \pi_\theta^\top \right) r \right\|_2 \cdot \left\| y \right\|_2^2 \qquad \left( H(\pi_\theta) \coloneqq \diagonalmatrix(\pi_\theta) - \pi_\theta \pi_\theta^\top \right) \\
    &= 3 \cdot \bigg\| \frac{d \pi_{\theta}^\top r}{d {\theta}} \bigg\|_2 \cdot \left\| y \right\|_2^2. \qquad \left( \text{by \cref{eq:stochastic_gradient_ascent_sampled_reward}} \right) \qedhere
\end{align}
\end{proof}

\textbf{\cref{lem:non_uniform_smoothness_special_two_iterations}} (NS between iterates)\textbf{.}
Using \cref{alg:gradient_bandit_algorithm_sampled_reward} with $\eta \in \big(0, \frac{2}{9 \cdot R_{\max}} \big)$, we have, for all $t \ge 1$,
\begin{align}
    D(\theta_{t+1}, \theta_t) \coloneqq \left| ( \pi_{\theta_{t+1}} - \pi_{\theta_t})^\top r - \Big\langle \frac{d \pi_{\theta_t}^\top r}{d \theta_t}, \theta_{t+1} - \theta_t \Big\rangle \right| \le \frac{\beta(\theta_t)}{2} \cdot \| \theta_{t+1} - \theta_t \|_2^2,
\end{align}
where
\begin{align}
    \beta(\theta_t) = \frac{6}{2 - 9 \cdot R_{\max} \cdot \eta } \cdot \bigg\| \frac{d \pi_{\theta_t}^\top r}{d \theta_t} \bigg\|_2.
\end{align}
\begin{proof}
Denote $\theta_\zeta \coloneqq \theta_t + \zeta \cdot (\theta_{t+1} - \theta_t)$ with some $\zeta \in [0,1]$. According to Taylor's theorem, we have,
\begin{align}
\label{eq:non_uniform_smoothness_special_two_iterations_intermediate_1}
    \left| ( \pi_{\theta_{t+1}} - \pi_{\theta_t})^\top r - \Big\langle \frac{d \pi_{\theta_t}^\top r}{d \theta_t}, \theta_{t+1} - \theta_t \Big\rangle \right| &= \frac{1}{2} \cdot \left| \left( \theta_{t+1} - \theta_t \right)^\top \frac{d^2 \pi_{\theta_\zeta}^\top r}{d {\theta_\zeta}^2} \left( \theta_{t+1} - \theta_t \right) \right| \\
    &\le \frac{3 }{2} \cdot \bigg\| \frac{d \pi_{\theta_\zeta}^\top r}{d {\theta_\zeta}} \bigg\|_2 \cdot \| \theta_{t+1} - \theta_t \|_2^2. \qquad \left( \text{by \cref{lem:non_uniform_smoothness_softmax_special}} \right)
\end{align}
Denote $\theta_{\zeta_1} \coloneqq \theta_t + \zeta_1 \cdot (\theta_\zeta - \theta_t)$ with some $\zeta_1 \in [0,1]$. We have,
\begin{align}
\label{eq:non_uniform_smoothness_special_two_iterations_intermediate_2}
\MoveEqLeft
    \bigg\| \frac{d \pi_{\theta_{\zeta}}^\top r}{d {\theta_{\zeta}}} - \frac{d \pi_{\theta_t}^\top r}{d \theta_t} \bigg\|_2 = \left\| \int_{0}^{1} \bigg\langle \frac{d^2 \{ \pi_{\theta_{\zeta_1}}^\top r \} }{d {\theta_{\zeta_1}^2 }}, \theta_{\zeta} - \theta_t \bigg\rangle d \zeta_1 \right\|_2 \qquad \left( \text{Fundamental theorem of calculus} \right) \\
    &\le \int_{0}^{1} \left\| \frac{d^2 \{ \pi_{\theta_{\zeta_1}}^\top r \} }{d {\theta_{\zeta_1}^2 }} \right\|_2 \cdot \left\| \theta_{\zeta} - \theta_t \right\|_2 d \zeta_1 \qquad \left( \text{by Cauchy–Schwarz} \right) \\
    &\le \int_{0}^{1} 3 \cdot \bigg\| \frac{d  \pi_{\theta_{\zeta_1}}^\top r }{d {\theta_{\zeta_1} }} \bigg\|_2 \cdot \left\| \theta_{\zeta} - \theta_t \right\|_2 d \zeta_1 \qquad \left( \text{by \cref{lem:non_uniform_smoothness_softmax_special}} \right)  \\
    &= \int_{0}^{1} 3 \cdot \bigg\| \frac{d  \pi_{\theta_{\zeta_1}}^\top r }{d {\theta_{\zeta_1} }} \bigg\|_2 \cdot \zeta \cdot \left\| \theta_{t+1} - \theta_t \right\|_2 d \zeta_1 \qquad \left( \theta_\zeta \coloneqq \theta_t + \zeta \cdot (\theta_{t+1} - \theta_t) \right)  \\
    &\le \int_{0}^{1} 3 \cdot \bigg\| \frac{d  \pi_{\theta_{\zeta_1}}^\top r }{d {\theta_{\zeta_1} }} \bigg\|_2 \cdot \eta \cdot \bigg\| \frac{d \pi_{\theta_t}^\top \hat{r}_t }{d \theta_t} \bigg\|_2 \ d \zeta_1, \qquad \left( \zeta \in [0, 1], \text{ and } \theta_{t+1} = \theta_t + \eta \cdot \frac{d \pi_{\theta_t}^\top \hat{r}_t }{d \theta_t} \right)
\end{align}
where the second inequality is because of the Hessian is symmetric, and its operator norm is equal to its spectral radius. Therefore, we have,
\begin{align}
\label{eq:non_uniform_smoothness_special_two_iterations_intermediate_3}
    \bigg\| \frac{d \pi_{\theta_{\zeta}}^\top r}{d {\theta_{\zeta}}} \bigg\|_2 &\le \bigg\| \frac{d \pi_{\theta_t}^\top r}{d \theta_t} \bigg\|_2 + \bigg\| \frac{d \pi_{\theta_{\zeta}}^\top r}{d {\theta_{\zeta}}} - \frac{d \pi_{\theta_t}^\top r}{d \theta_t} \bigg\|_2 \qquad \left( \text{by triangle inequality} \right) \\
    &\le \bigg\| \frac{d \pi_{\theta_t}^\top r}{d \theta_t} \bigg\|_2 + 3 \,  \eta \cdot \bigg\| \frac{d \pi_{\theta_t}^\top \hat{r}_t }{d \theta_t} \bigg\|_2 \cdot \int_{0}^{1} \bigg\| \frac{d  \pi_{\theta_{\zeta_1}}^\top r }{d {\theta_{\zeta_1} }} \bigg\|_2 d \zeta_1. \qquad \left( \text{by \cref{eq:non_uniform_smoothness_special_two_iterations_intermediate_2}} \right)
\end{align}
Denote $\theta_{\zeta_2} \coloneqq \theta_t + \zeta_2 \cdot (\theta_{\zeta_1} - \theta_t)$ with $\zeta_2 \in [0,1]$. Using similar calculation in \cref{eq:non_uniform_smoothness_special_two_iterations_intermediate_2}, we have,
\begin{align}
\label{eq:non_uniform_smoothness_special_two_iterations_intermediate_4}
    \bigg\| \frac{d  \pi_{\theta_{\zeta_1}}^\top r }{d {\theta_{\zeta_1} }} \bigg\|_2 &\le \bigg\| \frac{d \pi_{\theta_t}^\top r}{d \theta_t} \bigg\|_2 + \bigg\| \frac{d \pi_{\theta_{\zeta_1}}^\top r}{d {\theta_{\zeta_1}}} - \frac{d \pi_{\theta_t}^\top r}{d \theta_t} \bigg\|_2 \\
    &\le \bigg\| \frac{d \pi_{\theta_t}^\top r}{d \theta_t} \bigg\|_2 + 3 \,  \eta \cdot \bigg\| \frac{d \pi_{\theta_t}^\top \hat{r}_t }{d \theta_t} \bigg\|_2 \cdot \int_{0}^{1}  \bigg\| \frac{d  \pi_{\theta_{\zeta_2}}^\top r }{d {\theta_{\zeta_2} }} \bigg\|_2 d \zeta_2.
\end{align}
Combining \cref{eq:non_uniform_smoothness_special_two_iterations_intermediate_3,eq:non_uniform_smoothness_special_two_iterations_intermediate_4}, we have,
\begin{align}
\label{eq:non_uniform_smoothness_special_two_iterations_intermediate_5}
    \bigg\| \frac{d \pi_{\theta_{\zeta}}^\top r}{d {\theta_{\zeta}}} \bigg\|_2 \le \left( 1 + 3 \, \eta \cdot \bigg\| \frac{d \pi_{\theta_t}^\top \hat{r}_t }{d \theta_t} \bigg\|_2 \right) \cdot \bigg\| \frac{d \pi_{\theta_t}^\top r}{d \theta_t} \bigg\|_2 + \left( 3 \,  \eta \cdot \bigg\| \frac{d \pi_{\theta_t}^\top \hat{r}_t }{d \theta_t} \bigg\|_2 \right)^2 \cdot \int_{0}^{1} \int_{0}^{1} \bigg\| \frac{d  \pi_{\theta_{\zeta_2}}^\top r }{d {\theta_{\zeta_2} }} \bigg\|_2 d \zeta_2 d \zeta_1,
\end{align}
which, by recurring the above arguments, implies that,
\begin{align}
\label{eq:non_uniform_smoothness_special_two_iterations_intermediate_6}
    \bigg\| \frac{d \pi_{\theta_{\zeta}}^\top r}{d {\theta_{\zeta}}} \bigg\|_2 &\le \sum_{i = 0}^{\infty}{ \left( 3 \,  \eta \cdot \bigg\| \frac{d \pi_{\theta_t}^\top \hat{r}_t }{d \theta_t} \bigg\|_2 \right)^i } \cdot \bigg\| \frac{d \pi_{\theta_t}^\top r}{d \theta_t} \bigg\|_2.
\end{align}
Next, we have,
\begin{align}
\label{eq:non_uniform_smoothness_special_two_iterations_intermediate_7}
    3 \, \eta \cdot \bigg\| \frac{d \pi_{\theta_t}^\top \hat{r}_t }{d \theta_t} \bigg\|_2 &\le 3 \, \eta \cdot \sqrt{ 2 \cdot R_{\max}^2 \cdot \left( 1 - \pi_{\theta_t}(a_t) \right)^2 } \qquad \left( \text{by \cref{eq:unbiased_stochastic_gradient_bounded_scale_sampled_reward_result_2_intermediate_1}} \right) \\
    &< \frac{3 \cdot 2}{9 \cdot R_{\max}} \cdot \sqrt{2} \cdot R_{\max} \qquad \left( \pi_{\theta_t}(a_t) \in (0, 1), \text{ and } \eta < \frac{2}{9 \cdot R_{\max}} \right) \\
    &< 1.
\end{align}
Combining \cref{eq:non_uniform_smoothness_special_two_iterations_intermediate_6,eq:non_uniform_smoothness_special_two_iterations_intermediate_7}, we have,
\begin{align}
\label{eq:non_uniform_smoothness_special_two_iterations_intermediate_8}
    \bigg\| \frac{d \pi_{\theta_{\zeta}}^\top r}{d {\theta_{\zeta}}} \bigg\|_2 &\le \frac{1}{1 - 3 \,  \eta \cdot \Big\| \frac{d \pi_{\theta_t}^\top \hat{r}_t }{d \theta_t} \Big\|_2 } \cdot \bigg\| \frac{d \pi_{\theta_t}^\top r}{d \theta_t} \bigg\|_2 \qquad \left( 3 \, \eta \cdot \bigg\| \frac{d \pi_{\theta_t}^\top \hat{r}_t }{d \theta_t} \bigg\|_2 \in ( 0 , 1) \text{ from \cref{eq:non_uniform_smoothness_special_two_iterations_intermediate_7}}\right) \\
    &\le \frac{1}{1 - 3 \, \eta \cdot \sqrt{2} \cdot R_{\max} } \cdot \bigg\| \frac{d \pi_{\theta_t}^\top r}{d \theta_t} \bigg\|_2 \qquad \left( \text{by \cref{eq:unbiased_stochastic_gradient_bounded_scale_sampled_reward_result_2_intermediate_1}} \right)  \\
    &< \frac{1}{1 - \frac{9}{2} \cdot R_{\max} \cdot \eta } \cdot \bigg\| \frac{d \pi_{\theta_t}^\top r}{d \theta_t} \bigg\|_2.
\end{align}
Combining \cref{eq:non_uniform_smoothness_special_two_iterations_intermediate_1,eq:non_uniform_smoothness_special_two_iterations_intermediate_8}, we have,
\begin{align}
    \left| ( \pi_{\theta_{t+1}} - \pi_{\theta_t})^\top r - \Big\langle \frac{d \pi_{\theta_t}^\top r}{d \theta_t}, \theta_{t+1} - \theta_t \Big\rangle \right| &\le \frac{3 }{2} \cdot \bigg\| \frac{d \pi_{\theta_\zeta}^\top r}{d {\theta_\zeta}} \bigg\|_2 \cdot \| \theta_{t+1} - \theta_t \|_2^2 \\
    &\le \frac{3}{2 - 9 \cdot R_{\max} \cdot \eta } \cdot \bigg\| \frac{d \pi_{\theta_t}^\top r}{d \theta_t} \bigg\|_2 \cdot \| \theta_{t+1} - \theta_t \|_2^2. \qedhere
\end{align}
\end{proof}

\textbf{\cref{lem:strong_growth_conditions_sampled_reward}} (Strong growth conditions / Self-bounding noise property)\textbf{.}
Using \cref{alg:gradient_bandit_algorithm_sampled_reward}, we have, for all $t \ge 1$,
\begin{align}
    \mathbb{E}_t{ \left[ \bigg\| \frac{d \pi_{\theta_t}^\top \hat{r}_t}{d \theta_t} \bigg\|_2^2 \right] } 
    \le \frac{8 \cdot R_{\max}^3 \cdot K^{3/2} }{ \Delta^2 } \cdot  \bigg\| \frac{d \pi_{\theta_t}^\top r}{d \theta_t} \bigg\|_2,
\end{align}
where $\Delta \coloneqq \min_{i \not= j}{ | r(i) - r(j) | } $.
\begin{proof}
Given $t \ge 1$, denote $k_t$ as the action with largest probability, i.e., $k_t \coloneqq \argmax_{a \in [K]}{ \pi_{\theta_t}(a) }$. We have,
\begin{align}
\label{eq:strong_growth_conditions_sampled_reward_intermedita_1}
    \pi_{\theta_t}(k_t) \ge \frac{1}{K}.
\end{align}
According to \cref{eq:unbiased_stochastic_gradient_bounded_scale_sampled_reward_result_2_intermediate_3}, we have,
\begin{align}
\label{eq:strong_growth_conditions_sampled_reward_intermedita_2}
    \mathbb{E}_t{ \left[ \bigg\| \frac{d \pi_{\theta_t}^\top \hat{r}_t }{d \theta_t} \bigg\|_2^2 \right] } &= \sum_{a \in [K]}{ \probability{\left( a_t = a \right) } \cdot \Bigg[ \bigg\| \frac{d \pi_{\theta_t}^\top \hat{r}_t }{d \theta_t} \bigg\|_2^2 \ \Big| \ a_t = a \Bigg] } \\
    &\le 2 \cdot R_{\max}^2 \cdot \sum_{a \in [K]}{ \pi_{\theta_t}(a) \cdot  \left( 1 - \pi_{\theta_t}(a) \right)^2 } \\
    &= 2 \cdot R_{\max}^2 \cdot \bigg[ \pi_{\theta_t}(k_t) \cdot \left( 1 - \pi_{\theta_t}(k_t) \right)^2 + \sum_{a \not= k_t}{ \pi_{\theta_t}(a) \cdot \left( 1 - \pi_{\theta_t}(a) \right)^2  } \bigg] \\
    &\le 2 \cdot R_{\max}^2 \cdot \bigg[ 1 - \pi_{\theta_t}(k_t) + \sum_{a \not= k_t}{ \pi_{\theta_t}(a) } \bigg] \qquad \left( \pi_{\theta_t}(a) \in (0, 1) \text{ for all } a \in [K] \right) \\
    &= 4 \cdot R_{\max}^2 \cdot \left( 1 - \pi_{\theta_t}(k_t) \right).
\end{align}
On the other hand, we have,
\begin{align}
\label{eq:strong_growth_conditions_sampled_reward_intermedita_3}
    \bigg\| \frac{d \pi_{\theta_t}^\top r}{d \theta_t} \bigg\|_2^2 &= \sum_{a \in [K] }{ \pi_{\theta_t}(a)^2 \cdot (r(a) - \pi_{\theta_t}^\top r)^2} \\
    &= \sum_{a^\prime \in [K]}{ (r(a^\prime) - \pi_{\theta_t}^\top r)^2 } \cdot \sum_{a \in [K] }{ \pi_{\theta_t}(a)^2 \cdot \frac{ (r(a) - \pi_{\theta_t}^\top r)^2 }{ \sum_{a^\prime \in [K]}{ (r(a^\prime) - \pi_{\theta_t}^\top r)^2 } } } \\
    &\ge \sum_{a^\prime \in [K]}{ (r(a^\prime) - \pi_{\theta_t}^\top r)^2 } \cdot \Bigg[ \sum_{a \in [K] }{ \pi_{\theta_t}(a) \cdot \frac{ (r(a) - \pi_{\theta_t}^\top r)^2 }{ \sum_{a^\prime \in [K]}{ (r(a^\prime) - \pi_{\theta_t}^\top r)^2 } } } \Bigg]^2 \qquad \left( \text{Jensen's inequality} \right) \\
    &= \frac{1}{ \sum_{a^\prime \in [K]}{ (r(a^\prime) - \pi_{\theta_t}^\top r)^2 } } \cdot \Bigg[ \sum_{a \in [K] }{ \pi_{\theta_t}(a) \cdot (r(a) - \pi_{\theta_t}^\top r)^2 } \Bigg]^2 \\
    &\ge \frac{1}{4 \cdot K \cdot R_{\max}^2 } \cdot \Bigg[ \sum_{a \in [K] }{ \pi_{\theta_t}(a) \cdot (r(a) - \pi_{\theta_t}^\top r)^2 } \Bigg]^2, \qquad \left( r \in [-R_{\max}, R_{\max}]^K \right)
\end{align}
which implies that,
\begin{align}
\label{eq:strong_growth_conditions_sampled_reward_intermedita_4}
    \bigg\| \frac{d \pi_{\theta_t}^\top r}{d \theta_t} \bigg\|_2 \ge \frac{1}{2 \cdot \sqrt{K} \cdot R_{\max} } \cdot \sum_{a \in [K] }{ \pi_{\theta_t}(a) \cdot (r(a) - \pi_{\theta_t}^\top r)^2 }.
\end{align}
Using similar calculations in the proofs for \citet[Lemma 2]{mei2021understanding}, we have,
\begin{align}
\MoveEqLeft
\label{eq:strong_growth_conditions_sampled_reward_intermedita_5}
    \sum_{a \in [K] }{ \pi_{\theta_t}(a) \cdot (r(a) - \pi_{\theta_t}^\top r)^2 } = \sum_{i=1}^{K}{ \pi_{\theta_t}(i) \cdot r(i)^2 } - \bigg[ \sum_{i=1}^{K}{ \pi_{\theta_t}(i) \cdot r(i) } \bigg]^2 \\
    &= \sum_{i=1}^{K}{ \pi_{\theta_t}(i) \cdot r(i)^2 } - \sum_{i=1}^{K}{ \pi_{\theta_t}(i)^2 \cdot r(i)^2 } - 2 \cdot \sum_{i=1}^{K-1}{ \pi_{\theta_t}(i) \cdot r(i) \cdot \sum_{j = i+1}^{K}{ \pi_{\theta_t}(j) \cdot r(j) } } \\
    &= \sum_{i=1}^{K}{ \pi_{\theta_t}(i) \cdot r(i)^2 \cdot \left( 1 - \pi_{\theta_t}(i) \right) } - 2 \cdot \sum_{i=1}^{K-1}{ \pi_{\theta_t}(i) \cdot r(i) \cdot \sum_{j = i+1}^{K}{ \pi_{\theta_t}(j) \cdot r(j) } } \\
    &= \sum_{i=1}^{K}{ \pi_{\theta_t}(i) \cdot r(i)^2 \cdot \sum_{j \not= i}{\pi_{\theta_t}(j)}  } - 2 \cdot \sum_{i=1}^{K-1}{ \pi_{\theta_t}(i) \cdot r(i) \cdot \sum_{j = i+1}^{K}{ \pi_{\theta_t}(j) \cdot r(j) } } \\
    &= \sum_{i=1}^{K-1}{ \pi_{\theta_t}(i) \cdot \sum_{j = i+1}^{K}{\pi_{\theta_t}(j) \cdot \left( r(i)^2 + r(j)^2 \right) }  } - 2 \cdot \sum_{i=1}^{K-1}{ \pi_{\theta_t}(i) \cdot r(i) \cdot \sum_{j = i+1}^{K}{ \pi_{\theta_t}(j) \cdot r(j) } } \\
    &= \sum_{i=1}^{K-1}{ \pi_{\theta_t}(i) \cdot \sum_{j = i+1}^{K}{\pi_{\theta_t}(j) \cdot \left( r(i) - r(j) \right)^2 }  },
\end{align}
which implies that,
\begin{align}
\MoveEqLeft
\label{eq:strong_growth_conditions_sampled_reward_intermedita_6}
    \sum_{a \in [K] }{ \pi_{\theta_t}(a) \cdot (r(a) - \pi_{\theta_t}^\top r)^2 } \ge \sum_{i=1}^{k_t}{ \pi_{\theta_t}(i) \cdot \sum_{j = i+1}^{K}{\pi_{\theta_t}(j) \cdot \left( r(i) - r(j) \right)^2 }  } \qquad \left( \text{fewer terms} \right) \\
    &\ge \sum_{i=1}^{k_t-1}{ \pi_{\theta_t}(i) \cdot \pi_{\theta_t}(k_t) \cdot \left( r(i) - r(k_t) \right)^2 } + \pi_{\theta_t}(k_t) \cdot \sum_{j = k_t+1}^{K}{\pi_{\theta_t}(j) \cdot \left( r(k_t) - r(j) \right)^2 } \qquad \left( \text{fewer terms} \right) \\
    &= \pi_{\theta_t}(k_t) \cdot \sum_{a \not= k_t}{ \pi_{\theta_t}(a) \cdot \left( r(a) - r(k_t) \right)^2 } \\
    &\ge \frac{ \Delta^2 }{K} \cdot \left( 1 - \pi_{\theta_t}(k_t) \right), \qquad \left( \text{by \cref{eq:strong_growth_conditions_sampled_reward_intermedita_1}} \right)
\end{align}
where $\Delta \coloneqq \min_{i \not= j}{ | r(i) - r(j) | } $. Therefore, we have,
\begin{align}
    \mathbb{E}_t{ \left[ \bigg\| \frac{d \pi_{\theta_t}^\top \hat{r}_t }{d \theta_t} \bigg\|_2^2 \right] } &\le 4 \cdot R_{\max}^2 \cdot \left( 1 - \pi_{\theta_t}(k_t) \right) \qquad \left( \text{by \cref{eq:strong_growth_conditions_sampled_reward_intermedita_2}} \right) \\
    &\le \frac{4 \cdot R_{\max}^2 \cdot K}{ \Delta^2 } \cdot \sum_{a \in [K] }{ \pi_{\theta_t}(a) \cdot (r(a) - \pi_{\theta_t}^\top r)^2 } \qquad \left( \text{by \cref{eq:strong_growth_conditions_sampled_reward_intermedita_6}} \right) \\
    &\le \frac{4 \cdot R_{\max}^2 \cdot K}{ \Delta^2 } \cdot 2 \cdot \sqrt{K} \cdot R_{\max} \cdot \bigg\| \frac{d \pi_{\theta_t}^\top r}{d \theta_t} \bigg\|_2 \qquad \left( \text{by \cref{eq:strong_growth_conditions_sampled_reward_intermedita_4}} \right) \\
    &= \frac{8 \cdot R_{\max}^3 \cdot K^{3/2} }{ \Delta^2 } \cdot \bigg\| \frac{d \pi_{\theta_t}^\top r}{d \theta_t} \bigg\|_2. \qedhere
\end{align}
\end{proof}

\textbf{\cref{lem:constant_learning_rates_expected_progress_sampled_reward}} (Constant learning rates)\textbf{.}
Using \cref{alg:gradient_bandit_algorithm_sampled_reward} with $\eta = \frac{\Delta^2}{40 \cdot K^{3/2} \cdot R_{\max}^3 }$, we have, for all $t \ge 1$,
\begin{align}
    \pi_{\theta_t}^\top r - \EEt{ \pi_{\theta_{t+1}}^\top r } \le - \frac{\Delta^2}{80 \cdot K^{3/2} \cdot R_{\max}^3 } \cdot \bigg\| \frac{d \pi_{\theta_t}^\top r}{d \theta_t} \bigg\|_2^2.
\end{align}
\begin{proof}
Using the learning rate,
\begin{align}
\label{eq:constant_learning_rates_expected_progress_sampled_reward_intermediate_0a}
    \eta &= \frac{\Delta^2}{40 \cdot K^{3/2} \cdot R_{\max}^3 } \\
    &= \frac{4}{45 \cdot R_{\max} } \cdot \frac{\Delta^2}{R_{\max}^2} \cdot \frac{1}{K^{3/2}} \cdot \frac{45}{4} \cdot \frac{1}{40} \\
    &\le \frac{4}{45 \cdot R_{\max} } \cdot 4 \cdot \frac{1}{2 \cdot \sqrt{2}} \cdot  \frac{45}{4} \cdot \frac{1}{40}, \qquad \left( \Delta \le 2 \cdot R_{\max}, \text{ and } K \ge 2 \right) \\
\label{eq:constant_learning_rates_expected_progress_sampled_reward_intermediate_0b}
    &< \frac{4}{45 \cdot R_{\max} },
\end{align}
we have $\eta \in \big(0, \frac{2}{9 \cdot R_{\max}} \big)$. According to \cref{lem:non_uniform_smoothness_special_two_iterations}, we have,
\begin{align}
\MoveEqLeft
\label{eq:constant_learning_rates_expected_progress_sampled_reward_intermediate_1}
    \left| ( \pi_{\theta_{t+1}} - \pi_{\theta_t})^\top r - \Big\langle \frac{d \pi_{\theta_t}^\top r}{d \theta_t}, \theta_{t+1} - \theta_t \Big\rangle \right| \le \frac{3}{2 - 9 \cdot R_{\max} \cdot \eta } \cdot \bigg\| \frac{d \pi_{\theta_t}^\top r}{d \theta_t} \bigg\|_2 \cdot \| \theta_{t+1} - \theta_t \|_2^2 \\
    &\le \frac{3}{2 - 9 \cdot R_{\max} \cdot \frac{4}{45 \cdot R_{\max} } } \cdot \bigg\| \frac{d \pi_{\theta_t}^\top r}{d \theta_t} \bigg\|_2 \cdot \| \theta_{t+1} - \theta_t \|_2^2 \qquad \left( \text{by \cref{eq:constant_learning_rates_expected_progress_sampled_reward_intermediate_0b}} \right) \\
    &= \frac{5}{2} \cdot \bigg\| \frac{d \pi_{\theta_t}^\top r}{d \theta_t} \bigg\|_2 \cdot \| \theta_{t+1} - \theta_t \|_2^2,
\end{align}
which implies that,
\begin{align}
\label{eq:constant_learning_rates_expected_progress_sampled_reward_intermediate_2}
    \pi_{\theta_t}^\top r - \pi_{\theta_{t+1}}^\top r &\le - \Big\langle \frac{d \pi_{\theta_t}^\top r}{d \theta_t}, \theta_{t+1} - \theta_t \Big\rangle +  \frac{5}{2} \cdot \bigg\| \frac{d \pi_{\theta_t}^\top r}{d \theta_t} \bigg\|_2 \cdot \| \theta_{t+1} - \theta_{t} \|_2^2 \\
    &= - \eta \cdot \Big\langle \frac{d \pi_{\theta_t}^\top r}{d \theta_t}, \frac{d \pi_{\theta_t}^\top \hat{r}_t}{d \theta_t} \Big\rangle + \frac{5}{2} \cdot \bigg\| \frac{d \pi_{\theta_t}^\top r}{d \theta_t} \bigg\|_2 \cdot \eta^2 \cdot \bigg\| \frac{d \pi_{\theta_t}^\top \hat{r}_t}{d \theta_t} \bigg\|_2^2,
\end{align}
where the last equation uses \cref{alg:gradient_bandit_algorithm_sampled_reward}.
Taking expectation over $a_t \sim \pi_{\theta_t}(\cdot)$ and $R_t(a_t) \sim P_{a_t}$, we have,
\begin{align}
\MoveEqLeft
\label{eq:constant_learning_rates_expected_progress_sampled_reward_intermediate_3}
    \pi_{\theta_t}^\top r - \EEt{ \pi_{\theta_{t+1}}^\top r } \le - \eta \cdot \Big\langle \frac{d \pi_{\theta_t}^\top r}{d \theta_t}, \mathbb{E}_t{ \bigg[ \frac{d \pi_{\theta_t}^\top \hat{r}_t}{d \theta_t} \bigg] } \Big\rangle + \frac{5}{2} \cdot \bigg\| \frac{d \pi_{\theta_t}^\top r}{d \theta_t} \bigg\|_2 \cdot \eta^2 \cdot \mathbb{E}_t{ \Bigg[ \bigg\| \frac{d \pi_{\theta_t}^\top \hat{r}_t}{d \theta_t} \bigg\|_2^2  \Bigg] } \\
    &= - \eta \cdot \bigg\| \frac{d \pi_{\theta_t}^\top r}{d \theta_t} \bigg\|_2^2 + \frac{5}{2} \cdot \bigg\| \frac{d \pi_{\theta_t}^\top r}{d \theta_t} \bigg\|_2 \cdot \eta^2 \cdot \mathbb{E}_t{ \Bigg[ \bigg\| \frac{d \pi_{\theta_t}^\top \hat{r}_t}{d \theta_t} \bigg\|_2^2  \Bigg] } \qquad \left( \text{by \cref{prop:unbiased_stochastic_gradient_bounded_scale_sampled_reward}} \right) \\
    &\le - \eta \cdot \bigg\| \frac{d \pi_{\theta_t}^\top r}{d \theta_t} \bigg\|_2^2 + \frac{5}{2} \cdot \bigg\| \frac{d \pi_{\theta_t}^\top r}{d \theta_t} \bigg\|_2 \cdot \eta^2 \cdot \frac{8 \cdot R_{\max}^3 \cdot K^{3/2} }{ \Delta^2 } \cdot \bigg\| \frac{d \pi_{\theta_t}^\top r}{d \theta_t} \bigg\|_2 \qquad \left( \text{by \cref{lem:strong_growth_conditions_sampled_reward}} \right) \\
    &= \left( - \eta + \eta^2 \cdot \frac{20 \cdot R_{\max}^3 \cdot K^{3/2} }{ \Delta^2 } \right) \cdot \bigg\| \frac{d \pi_{\theta_t}^\top r}{d \theta_t} \bigg\|_2^2 \\
    &= - \frac{\Delta^2}{80 \cdot K^{3/2} \cdot R_{\max}^3 } \cdot \bigg\| \frac{d \pi_{\theta_t}^\top r}{d \theta_t} \bigg\|_2^2. \qquad \left( \text{by \cref{eq:constant_learning_rates_expected_progress_sampled_reward_intermediate_0a}} \right) \qedhere
\end{align}
\end{proof}

\textbf{\cref{cor:almost_sure_convergence_gradient_bandit_algorithms_sampled_reward}.}
Using \cref{alg:gradient_bandit_algorithm_sampled_reward}, we have, the sequence $\{ \pi_{\theta_t}^\top r \}_{t\ge 1}$ converges w. p. $1$.
   
% \corAlmostSureConvergenceStochasticNpgValueBaselineSpecial*

\begin{proof}
Setting $Y_t =  r(a^*) - \pi_{\theta_t}^\top r$, we have $Y_t \in [- R_{\max}, R_{\max}]$ by \cref{eq:true_mean_reward_expectation_bounded_sampled_reward}. 
Define $\cF_t$ as the $\sigma$-algebra generated by
$\{ a_1, R_1(a_1), a_2, R_2(a_2), \dots, a_{t-1}, R_{t-1}(a_{t-1}) \}$.
Note that $Y_t$ is $\cF_t$-measurable since $\theta_t$ is a deterministic function of $a_1, R_1(a_1), \ldots, a_{t-1}, R_{t-1}(a_{t-1})$.
According to \cref{lem:constant_learning_rates_expected_progress_sampled_reward}, 
using \cref{alg:gradient_bandit_algorithm_sampled_reward}, we have, for all $t \ge 1$,
$\pi_{\theta_t}^\top r - \EEt{ \pi_{\theta_{t+1}}^\top r } \le 0$, 
which indicates that $\EE{ Y_{t+1}|\cF_t }\le Y_t$.
Hence, the conditions of Doob's super-martingale theorem (\cref{thm:smc}) are satisfied and the result follows.
\end{proof}

 \subsection{Proof of \cref{thm:asymp_global_converg_gradient_bandit_sampled_reward}}
\label{pf:thm_asymp_global_converg_gradient_bandit_sampled_reward}
\textbf{\cref{thm:asymp_global_converg_gradient_bandit_sampled_reward}} (Asymptotic global convergence)\textbf{.}
Using \cref{alg:gradient_bandit_algorithm_sampled_reward}, we have, almost surely, 
\begin{align}
    \pi_{\theta_t}(a^*) \to 1, \text{ as } t \to \infty,
\end{align}
which implies that $\inf_{t \ge 1}{ \pi_{\theta_t}(a^*) } > 0$. 

\begin{proof}
According to \cref{alg:gradient_bandit_algorithm_sampled_reward}, for each $a \in [K]$, the update is,
\begin{align}
    \theta_{t+1}(a) 
    &= \theta_t(a) + \eta \cdot \pi_{\theta_t}(a) \cdot \left( \frac{ \sI\left\{ a_t = a \right\} }{ \pi_{\theta_t}(a) } \cdot R_t(a) - R_t(a_t) \right).
\end{align}
Given $i \in [K]$, define the following set $\gP(i)$ of ``generalized one-hot policy'',
\begin{align}
\label{eq:asymp_global_converg_gradient_bandit_sampled_reward_intermediate_1_a}
    \gA(i) &\coloneqq \left\{ j \in [K]: r(j) = r(i) \right\}, \\
\label{eq:asymp_global_converg_gradient_bandit_sampled_reward_intermediate_1_b}
    \gP(i) &\coloneqq \bigg\{ \pi \in \Delta(K): \sum_{j \in \gA(i)}{ \pi(j) } = 1 \bigg\}.
\end{align}
We make the following two claims.
\begin{claim}
\label{cl:approaching_generalized_one_hot_policy}
Almost surely, $\pi_{\theta_t}$ approaches one ``generalized one-hot policy'', i.e., 
there exists (a possibly random) $i \in [K]$, such that $\sum_{ j \in \gA(i) }{\pi_{\theta_t}(j) } \to 1$ almost surely as $t \to \infty$.
\end{claim}

\begin{claim}
\label{cl:contradiction_approaching_sub_optimal_generalized_one_hot_policy}
Almost surely, $\pi_{\theta_t}$ cannot approach any ``sub-optimal generalized one-hot policies'', 
i.e., $i$ in the previous claim must be an optimal action. 
\end{claim}

From \cref{cl:contradiction_approaching_sub_optimal_generalized_one_hot_policy}, it follows that $\sum_{ j \in \gA(a^*) }{\pi_{\theta_t}(j) } \to 1$ almost surely, as $t \to \infty$ and thus the policy sequence obtained
almost surely convergences to a globally optimal policy $\pi^*$.

\textbf{Proof of \cref{cl:approaching_generalized_one_hot_policy}}. %Recall that this claim stated that there exists at least one $i \in [K]$, such that $\lim_{t \to \infty} \sum_{ j \in \gA(i) }{\pi_{\theta_t}(j) } = 1$.
According to \cref{cor:almost_sure_convergence_gradient_bandit_algorithms_sampled_reward}, we have that for some (possibly random) $c \in [- R_{\max}, R_{\max}]$, almost surely,
\begin{align}
\label{eq:asymp_global_converg_gradient_bandit_sampled_reward_claim_1_intermediate_1}
    \lim_{t \to \infty}{ \pi_{\theta_t}^\top r} = c\,.
\end{align}
Thanks to $\pi_{\theta_t}^\top r \in [-R_{\max},R_{\max}]$ and
$\pi_{\theta_t}^\top r - \EEt{ \pi_{\theta_{t+1}}^\top r }\le 0 $ by \cref{lem:constant_learning_rates_expected_progress_sampled_reward},
 we have that
$X_t = \pi_{\theta_t}^\top r$ ($t\ge 1$) satisfies the conditions of
Corollary 3 in \citep{mei2022role}. Hence, by this result, 
almost surely, 
\begin{align}
\label{eq:asymp_global_converg_gradient_bandit_sampled_reward_claim_1_intermediate_2}
    \lim_{t \to \infty}\,\,{  \EEt{\pi_{\theta_{t+1}}^\top r} - \pi_{\theta_{t+1}}^\top r } & = 0\,,
\end{align}
which,
combined with 
\cref{eq:asymp_global_converg_gradient_bandit_sampled_reward_claim_1_intermediate_1}
 also gives that $\lim_{t\to\infty} \EEt{\pi_{\theta_{t+1}}^\top r} = c$ almost surely.
Hence, 
\begin{align}
\label{eq:asymp_global_converg_gradient_bandit_sampled_reward_claim_1_intermediate_3}
    \lim_{t \to \infty} \,\, \EEt{\pi_{\theta_{t+1}}^\top r} - \pi_{\theta_{t}}^\top r  & = c-c = 0,\, 
    \qquad \text{a.s.}
\end{align}
According to \cref{lem:constant_learning_rates_expected_progress_sampled_reward}, we have,
\begin{align}
\label{eq:asymp_global_converg_gradient_bandit_sampled_reward_claim_1_intermediate_4}
    \EEt{ \pi_{\theta_{t+1}}^\top r }  - \pi_{\theta_t}^\top r 
    &\ge \frac{\Delta^2}{80 \cdot K^{3/2} \cdot R_{\max}^3 } \cdot \bigg\| \frac{d \pi_{\theta_t}^\top r}{d \theta_t} \bigg\|_2^2 \\
    &= \frac{\Delta^2}{80 \cdot K^{3/2} \cdot R_{\max}^3 } \cdot \sum_{i = 1}^{K} \pi_{\theta_t}(i)^2 \cdot \left( r(i) - \pi_{\theta_t}^\top r \right)^2. \qquad \left(\text{by \cref{eq:softmax_policy_gradient_norm_squared}}\right)
\end{align}
Combining \cref{eq:asymp_global_converg_gradient_bandit_sampled_reward_claim_1_intermediate_3,eq:asymp_global_converg_gradient_bandit_sampled_reward_claim_1_intermediate_4}, we have, with probability $1$, 
\begin{align}
\label{eq:asymp_global_converg_gradient_bandit_sampled_reward_claim_1_intermediate_5}
    \lim_{t \to \infty}{  \sum_{i = 1}^{K} \pi_{\theta_t}(i)^2 \cdot \left( r(i) - \pi_{\theta_t}^\top r \right)^2 } = 0, 
\end{align}
which implies that, for all $i \in [K]$, almost surely,
\begin{align}
\label{eq:asymp_global_converg_gradient_bandit_sampled_reward_claim_1_intermediate_6}
    \lim_{t \to \infty}{ \pi_{\theta_t}(i)^2 \cdot \left( r(i) - \pi_{\theta_t}^\top r \right)^2 } = 0.
\end{align}

We claim that $c$, the almost sure limit of $\pi_{\theta_t}^\top r$, is such that almost surely, for some (possibly random)
 $i\in [K]$, $c= r(i)$ almost surely. 
We prove this by contradiction.
Let $\mathcal{E}_i = \{ c = r(i) \}$. 
Hence, our goal is to show that $\mathbb{P}( \cup_i \mathcal{E}_i ) = 1$.
Clearly, this follows from $\mathbb{P}( \cap_i \mathcal{E}_i^c ) = 0$, hence, we prove this.
On $\mathcal{E}_i^c$, since $\lim_{t\to\infty} \pi_{\theta_t}^\top r \ne r(i)$, 
we also have
\begin{align}
\label{eq:asymp_global_converg_gradient_bandit_sampled_reward_claim_1_intermediate_7}
    \lim_{t \to \infty}{ \left( r(i) - \pi_{\theta_t}^\top r \right)^2 } > 0, \quad \text{ almost surely on } \mathcal{E}_i^c\,.
\end{align}
This, together with \cref{eq:asymp_global_converg_gradient_bandit_sampled_reward_claim_1_intermediate_6} gives that almost surely on $\mathcal{E}_i^c$,
\begin{align}
\label{eq:asymp_global_converg_gradient_bandit_sampled_reward_claim_1_intermediate_8}
    \lim_{t \to \infty}{ \pi_{\theta_t}(i)^2 } = 0.
\end{align}
Hence, on $\cap_i \mathcal{E}_i^c$, almost surely, for all $i\in [K]$,
$\lim_{t \to \infty}{ \pi_{\theta_t}(i)^2 } = 0$. This contradicts
with that $\sum_i \pi_{\theta_t}(i)=1$ holds for all $t\ge 1$, and hence we must have that 
$\mathbb{P}(\cap_i \mathcal{E}_i^c)=0$, finishing the proof that 
$\mathbb{P}(\cup_i \mathcal{E}_i)=1$.

Now, let $i\in [K]$ be the (possibly random) index of the action for which $c=r(i)$ almost surely.
%\textbf{Case (ii)} There exist at least two actions $i, j \in [K]$, such that $r(i) \not= r(j)$, and $\lim_{t \to \infty}{ \pi_{\theta_t}^\top r} = r(i)$ and $\lim_{t \to \infty}{ \pi_{\theta_t}^\top r} = r(j)$, which is not possible because of \cref{eq:non_vanishing_nl_coefficient_stochastic_npg_value_baseline_special_claim_1_intermediate_1}.
Recall that $\gA(i)$ contains all actions $j$ with $r(j)=r(i)$
(cf. \cref{eq:asymp_global_converg_gradient_bandit_sampled_reward_intermediate_1_a}).
Clearly, it holds that 
%  $\lim_{t \to \infty}{ \pi_{\theta_t}^\top r}= r(j)$ almost surely for all $j\in \gA(i)$ and also that
%  for $j\not\in \gA(i)$, 
%  $\lim_{t \to \infty}{ \pi_{\theta_t}^\top r}\ne r(j)$ almost surely.
%  
%\textbf{Combining} \textbf{(i)} and \textbf{(ii)}, we have, there exists one $i \in [K]$, such that, 
for all $j \in \gA(i)$,
\begin{align}
\label{eq:non_vanishing_nl_coefficient_stochastic_npg_value_baseline_special_claim_1_intermediate_9}
    \lim_{t \to \infty}{ \pi_{\theta_t}^\top r } = r(j), \qquad \text{a.s.},
\end{align}
and we have, for all $k \not\in \gA(i)$,
\begin{align}
\label{eq:non_vanishing_nl_coefficient_stochastic_npg_value_baseline_special_claim_1_intermediate_10}
    \lim_{t \to \infty}{ \left( r(k) - \pi_{\theta_t}^\top r \right)^2 } > 0, \qquad \text{a.s.},
\end{align}
which implies that,
\begin{align}
\label{eq:non_vanishing_nl_coefficient_stochastic_npg_value_baseline_special_claim_1_intermediate_11}
    \lim_{t \to \infty}{ \sum_{k \not\in \gA(i)} \pi_{\theta_t}(k)^2 } = 0, \qquad \text{a.s.}
\end{align}
Therefore, we have,
\begin{align}
\label{eq:non_vanishing_nl_coefficient_stochastic_npg_value_baseline_special_claim_1_intermediate_12}
    \lim_{t \to \infty}{ \sum_{j \in \gA(i)} \pi_{\theta_t}(j) } = 1, \qquad \text{a.s.},
\end{align}
which means $\pi_{\theta_t}$ a.s. approaches the ``generalized one-hot policy'' $\gP(i)$ in \cref{eq:asymp_global_converg_gradient_bandit_sampled_reward_intermediate_1_b} as $t \to \infty$, finishing the proof of the first claim.

\textbf{Proof of \cref{cl:contradiction_approaching_sub_optimal_generalized_one_hot_policy}}.
Recall that this claim stated that  
$\lim_{t \to \infty} \sum_{ j \in \gA(a^*) }{\pi_{\theta_t}(j) } = 1$.
The brief sketch of the proof is as follows:
By \cref{cl:approaching_generalized_one_hot_policy},  there exists a (possibly random) $i\in [K]$ such that 
$\sum_{ j \in \gA(i) }{\pi_{\theta_t}(j) } \to 1$ almost surely, as $t \to \infty$.
If $i=a^*$ almost surely, \cref{cl:contradiction_approaching_sub_optimal_generalized_one_hot_policy} follows. 
Hence, it suffices to consider the event that $\{ i\not = a^* \}$ and show that this event has zero probability mass. Hence,  in the rest of the proof we assume that we are on the event when $i\not =a^*$.

Since $i \not= a^*$, there exists at least one ``good'' action $a^+ \in [K]$ such that $r(a^+) > r(i)$. The two cases are as follows.
\begin{description}%[style=unboxed,leftmargin=0cm]
    \item[2a)] All ``good'' actions are sampled finitely many times as $t \to \infty$. \label{cl:contradiction_approaching_sub_optimal_generalized_one_hot_policy:a}
    \item[2b)] At least one ``good'' action is sampled infinitely many times as $t \to \infty$.
\end{description}
In both cases, we show that $\sum_{ j \in \gA(i) }{ \exp\{ \theta_t(j) \} } < \infty$ as $t \to \infty$ (but for different reasons), \textcolor{red}{which is a contradiction with the assumption of $\sum_{ j \in \gA(i) }{\pi_{\theta_t}(j) } \to 1$ as $t \to \infty$}, given that a ``good'' action's parameter is almost surely lower bounded. Hence, $i\ne a^*$ almost surely does not happen, which means that almost surely $i=a^*$.
Let us now turn to the details of the proof. We start with 
 some useful extra notation.
For each action $a \in [K]$, for $t \ge 2$, we have the following decomposition,
\begin{align}
\label{eq:non_vanishing_nl_coefficient_stochastic_npg_value_baseline_special_claim_2_intermediate_1}
    \theta_{t}(a) = \underbrace{ \theta_{t}(a) - \chE_{t-1}{[ \theta_t(a)]} }_{ W_t(a) } + \underbrace { \chE_{t-1}{[ \theta_t(a)]} - \theta_{t-1}(a) }_{ P_{t-1}(a) } + \theta_{t-1}(a),
\end{align}
while we also have,
\begin{align}
\label{eq:non_vanishing_nl_coefficient_stochastic_npg_value_baseline_special_claim_2_intermediate_2}
    \theta_1(a) = \underbrace{ \theta_{1}(a) - \EE{\theta_{1}(a)} }_{ W_1(a) } + \EE{\theta_{1}(a)},
\end{align}
where $\EE{\theta_{1}(a)}$ accounts for possible randomness in initialization of $\theta_1$.

Define the following notations,
\begin{align}
\label{eq:non_vanishing_nl_coefficient_stochastic_npg_value_baseline_special_claim_2_intermediate_3a}
    Z_t(a) &\coloneqq W_1(a) + \cdots + W_t(a), \qquad \left( \text{``cumulative noise''} \right) \\
\label{eq:non_vanishing_nl_coefficient_stochastic_npg_value_baseline_special_claim_2_intermediate_3b}
    W_t(a) &\coloneqq \theta_t(a) - \chE_{t-1}{[ \theta_t(a)]}, \qquad \left( \text{``noise''} \right) \\
\label{eq:non_vanishing_nl_coefficient_stochastic_npg_value_baseline_special_claim_2_intermediate_3c}
    P_t(a) &\coloneqq \EEt{\theta_{t+1}(a)} - \theta_t(a). \qquad \left( \text{``progress''} \right)
\end{align}
Recursing \cref{eq:non_vanishing_nl_coefficient_stochastic_npg_value_baseline_special_claim_2_intermediate_1} gives,
\begin{align}
\label{eq:non_vanishing_nl_coefficient_stochastic_npg_value_baseline_special_claim_2_intermediate_4}
    \theta_t(a) = \EE{\theta_1(a)} + Z_t(a) + \underbrace{ P_1(a) + \cdots + P_{t-1}(a)}_{\text{``cumulative progress''}}.
\end{align}
Let
\begin{align}
\label{eq:non_vanishing_nl_coefficient_stochastic_npg_value_baseline_special_claim_2_intermediate_5}
    I_t(a) = \begin{cases}
		1, & \text{if } a_t = a\, , \\
		0, & \text{otherwise}\,.
	\end{cases}
\end{align}
The update rule (cf. \cref{alg:gradient_bandit_algorithm_sampled_reward}) is,
\begin{align}
\label{eq:non_vanishing_nl_coefficient_stochastic_npg_value_baseline_special_claim_2_intermediate_6}
    \theta_{t+1}(a) = \theta_{t}(a) + \eta \cdot \pi_{\theta_t}(a) \cdot \left( \frac{ I_t(a) }{ \pi_{\theta_t}(a) } \cdot R_t(a) - R_t(a_t) \right),
\end{align}
where $a_t \sim \pi_{\theta_t}(\cdot)$, and $x_t(a) \sim P_a$. Let $\gF_t$ be the $\sigma$-algebra generated by $a_1$, $x_1(a_1)$, $\cdots$, $a_{t-1}$, $x_{t-1}(a_{t-1})$:
\begin{align}
\cF_t = \sigma( \{ 
 a_1, R_1(a_1), \cdots, a_{t-1}, R_{t-1}(a_{t-1}) \} )\,.
\end{align}
Note that $\theta_{t},I_t$ are $\gF_t$-measurable and $\hat{x}_t$ is $\gF_{t+1}$-measurable for all $t\ge 1$. Let $\EEt{\cdot}$ denote the conditional expectation with respect to $\gF_t$: $\mathbb{E}_t[X] = \mathbb{E}[X|\gF_t]$. We have,
\begin{align}
    \EEt{W_{t+1}(a)}=0, \qquad \text{for } t=0, 1, \dots
\end{align}
Using the above notations, we have,
\begin{align}
\MoveEqLeft
\label{eq:non_vanishing_nl_coefficient_stochastic_npg_value_baseline_special_claim_2_intermediate_7}
    W_{t+1}(a) = \theta_{t+1}(a) - \EEt{\theta_{t+1}(a)} \\
    &=  \theta_{t}(a) + \eta \cdot \left( I_t(a) - \pi_{\theta_t}(a) \right) \cdot R_t(a_t)   - \left( \theta_{t}(a) + \eta \cdot \pi_{\theta_t}(a) \cdot \left( r(a) - \pi_{\theta_t}^\top r \right) \right) \\
    &= \eta \cdot \left( I_t(a) - \pi_{\theta_t}(a) \right) \cdot R_t(a_t)   - \eta \cdot \pi_{\theta_t}(a) \cdot \left( r(a) - \pi_{\theta_t}^\top r \right),
\end{align}
which implies that,
\begin{align}
\label{eq:non_vanishing_nl_coefficient_stochastic_npg_value_baseline_special_claim_2_intermediate_8}
    Z_t(a) &= W_1(a) + \cdots + W_t(a)
    % \\&
    % =\sum_{s=1}^t \eta \cdot \big[ \underbrace{  I_s(a) \cdot R_s(a_s)  - \pi_{\theta_s}(a) \cdot r(a) }_{(a)} \big] - \eta \cdot \underbrace{ \pi_{\theta_s}(a) \cdot \left( R_s(a_s) - \pi_{\theta_s}^\top r \right) }_{(b)}.
    \\&=\sum_{s=1}^t  \eta \cdot \left( I_s(a) - \pi_{\theta_s}(a) \right) \cdot R_s(a_s)   - \eta \cdot \pi_{\theta_s}(a) \cdot \left( r(a) - \pi_{\theta_s}^\top r \right).
\end{align}
We also have,
\begin{align}
\label{eq:non_vanishing_nl_coefficient_stochastic_npg_value_baseline_special_claim_2_intermediate_9}
    P_t(a) &= \EEt{\theta_{t+1}(a)} - \theta_t(a) 
    = \eta \cdot \pi_{\theta_t}(a) \cdot \left(  r(a) - \pi_{\theta_t}^\top r \right).
\end{align}
%
% {\color{blue}
We observe that $|W_{t+1}(a)| \le 3 \, \eta \cdot R_{\max}$, and, 
\begin{align}
\label{eq:non_vanishing_nl_coefficient_stochastic_npg_value_baseline_special_claim_2_intermediate_10}
    \mathrm{Var}[ W_{t+1}(a) | \cF_t ] &\coloneqq \bbE_t [( W_{t+1}(a) )^2 ]  \\
    &\le 2 \ \eta^2 \cdot \bbE_t \left[ \left( I_t(a) - \pi_{\theta_t}(a) \right)^2 \cdot R_t(a_t)^2 \right] + 2 \ \eta^2 \cdot \pi_{\theta_t}(a)^2 \cdot \left( r(a) - \pi_{\theta_t}^\top r \right)^2,
\end{align}
where the inequality is by $(a+b)^2 \le 2 \ a^2 + 2 \ b^2$. Next, we have,
\begin{align}
\label{eq:non_vanishing_nl_coefficient_stochastic_npg_value_baseline_special_claim_2_intermediate_11}
    \bbE_t \left[ \left( I_t(a) - \pi_{\theta_t}(a) \right)^2 \cdot R_t(a_t)^2 \right] &= \pi_t(a) \cdot ( 1 - \pi_t(a) )^2 \cdot r(a)^2 + \sum_{a'\ne a} \pi_{\theta_t}(a^\prime) \cdot \pi_t(a)^2 \cdot r(a^\prime)^2 \\
    &\le R_{\max}^2 \cdot \Big( \pi_t(a) \cdot ( 1 - \pi_t(a) )^2 + ( 1 - \pi_t(a) ) \cdot \pi_t(a)^2  \Big) \\
    &= R_{\max}^2 \cdot  \pi_{\theta_t}(a) \cdot ( 1 - \pi_{\theta_t}(a) ),
\end{align}
and,
\begin{align}
\label{eq:non_vanishing_nl_coefficient_stochastic_npg_value_baseline_special_claim_2_intermediate_12}
    \left| r(a) - \pi_{\theta_t}^\top r  \right|  
    &= \bigg| \sum_{a^\prime \ne a} \pi_{\theta_t}(a^\prime) \cdot \left( r(a)-r(a^\prime) \right) \bigg| \\ 
    &\le \sum_{a^\prime \ne a} \pi_{\theta_t}(a^\prime) \cdot \left|  r(a)-r(a^\prime)  \right| 
    \\
    &\le 2 \ R_{\max} \cdot \sum_{a^\prime \ne a} \pi_{\theta_t}(a^\prime) \\
    &= 2 \ R_{\max} \cdot \left( 1-\pi_{\theta_t}(a) \right).
\end{align}
Combining \cref{eq:non_vanishing_nl_coefficient_stochastic_npg_value_baseline_special_claim_2_intermediate_10,eq:non_vanishing_nl_coefficient_stochastic_npg_value_baseline_special_claim_2_intermediate_11,eq:non_vanishing_nl_coefficient_stochastic_npg_value_baseline_special_claim_2_intermediate_12}, we have,
\begin{align}
\label{eq:non_vanishing_nl_coefficient_stochastic_npg_value_baseline_special_claim_2_intermediate_13}
    \mathrm{Var}[ W_{t+1}(a) | \cF_t ]
    &\le 2 \ \eta^2 \cdot R_{\max}^2 \cdot \pi_{\theta_t}(a) \cdot (1-\pi_{\theta_t}(a)) + 8 \ \eta^2 \cdot R_{\max}^2 \cdot \pi_{\theta_t}(a)^2 \cdot ( 1-\pi_{\theta_t}(a) )^2 \\
    &\le 10 \ \eta^2 \cdot R_{\max}^2 \cdot \pi_{\theta_t}(a) \cdot (1-\pi_{\theta_t}(a)).
\end{align}
Let $X_{t+1}(a) \coloneqq \frac{ W_{t+1}(a) }{  6 \ \eta \cdot R_{\max} }$. Then we have, $|X_{t+1}(a)| \le 1/2$ and $\mathrm{Var}[ X_{t+1}(a) | \cF_t ] \le \frac{5}{18} \cdot \pi_{\theta_t}(a) \cdot (1-\pi_{\theta_t}(a))$.
According to \cref{thm:conc_new}, there exists event $\calE_1 $ such that $\bbP(\calE_1 ) \ge 1- \delta$, and when $\calE_1 $ holds, 
\begin{align}
\label{eq:non_vanishing_nl_coefficient_stochastic_npg_value_baseline_special_claim_2_intermediate_14}
    &  \forall t:\quad \left| \sum_{s=1}^t X_{s+1}(a) \right| \le  6 \ \sqrt{  (V_t(a)+4/3)\log \left( \frac{ V_t(a)+1  }{ \delta } \right) } + 2 \ \log(1/\delta)   + \frac{4}{3} \log 3,
\end{align}
which implies that, 
\begin{align}
\MoveEqLeft
    \forall t:\quad \left| Z_{t+1}(a) \right| = \left| \sum_{s=1}^t W_{s+1}(a) \right| \\
    &\le  36 \ \eta \ R_{\max} \ \sqrt{  (V_t(a)+4/3)\log \left( \frac{  V_t(a)+1  }{ \delta } \right) } + 12 \ \eta \ R_{\max} \ \log(1/\delta) +  8 \ \eta \ R_{\max}\log 3,
    \label{eq:global_conv_apply_conc_1}
\end{align}
where $V_t(a) \coloneqq \frac{5}{18} \cdot \sum_{s=1}^t  \pi_{\theta_s}(a)\cdot (1-\pi_{\theta_s}(a))$.

Recall that $i$ is the index of the (random) action $I\in [K]$  with 
\begin{align}
\label{eq:non_vanishing_nl_coefficient_stochastic_npg_value_baseline_special_claim_2_intermediate_15}
    \lim_{t \to \infty} \sum_{ j \in \gA(I) }{\pi_{\theta_t}(j) } = 1, \qquad \text{a.s.}
\end{align}
As noted earlier we consider the event $\{ I \ne a^* \}$, where $a^*$ is the index of an optimal action and we will show that this event has zero probability.
Since $\{ I \ne a^* \} = \cup_{i\in [K]} \{ I=i, i\ne a^* \}$, it suffices to show that for any fixed $i\in [K]$ index with $r(i)<r(a^*)$, $\{ I=i, i\ne a^* \}$ has zero probability.
Hence, in what follows we fix such a suboptimal action's index $i\in [K]$ and consider the event $\{I=i,i\ne a^*\}$.

Partition the action set $[K]$ into three parts using $r(i)$ as follows,
\begin{align}
\label{eq:non_vanishing_nl_coefficient_stochastic_npg_value_baseline_special_claim_2_intermediate_16a}
    \gA(i) &\coloneqq \left\{ j \in [K]: r(j) = r(i) \right\}, \qquad \left( \text{from \cref{eq:asymp_global_converg_gradient_bandit_sampled_reward_intermediate_1_a}} \right) \\
\label{eq:non_vanishing_nl_coefficient_stochastic_npg_value_baseline_special_claim_2_intermediate_16b}
    \gA^+(i) &\coloneqq \left\{ a^+ \in [K]: r(a^+) > r(i) \right\}, \\
\label{eq:non_vanishing_nl_coefficient_stochastic_npg_value_baseline_special_claim_2_intermediate_16c}
    \gA^-(i) &\coloneqq \left\{ a^- \in [K]: r(a^-) < r(i) \right\}.
\end{align}
Because $i$ was the index of a sub-optimal action, we have $\gA^+(i)\ne \emptyset$.
According to \cref{eq:non_vanishing_nl_coefficient_stochastic_npg_value_baseline_special_claim_2_intermediate_15}, 
on $\{I=i\} \supset \{I=i,i\ne a^* \}$,
we have $\pi_{\theta_t}^\top r \to r(i)$ as $t \to \infty$ because 
\begin{align}
\label{eq:non_vanishing_nl_coefficient_stochastic_npg_value_baseline_special_claim_2_intermediate_17}
    \left| r(i) - \pi_{\theta_t}^\top r \right| \le 2 \ R_{\max} \cdot \left( 1-\pi_{\theta_t}(i) \right). \qquad \left (\text{by \cref{eq:non_vanishing_nl_coefficient_stochastic_npg_value_baseline_special_claim_2_intermediate_12}} \right)
\end{align}
Therefore, there exists $\tau\ge 1$ such that
almost surely on $\{I = i,i\ne a^* \}$ 
 $\tau < \infty$ while we also have
\begin{align}
\label{eq:non_vanishing_nl_coefficient_stochastic_npg_value_baseline_special_claim_2_intermediate_18}
    r(a^+) - c' \ge \pi_{\theta_t}^\top r \ge r(a^-) + c', \qquad \text{for all } t \ge \tau,
\end{align}
for all $a^+ \in \gA^+(i)$, $a^- \in \gA^-(i)$, where $c' > 0$.
Hence, for all $t\ge \tau$, $a^+ \in \gA^+(i)$, $a^- \in \gA^-(i)$, we have, $P_t(a^+) > 0 > P_t(a^-)$, according to the definition of $P_t(a)$ in \cref{eq:non_vanishing_nl_coefficient_stochastic_npg_value_baseline_special_claim_2_intermediate_9}. We have, for all $a^+ \in \gA^+(i)$, and for all $t 
\ge 1$, 
\begin{align} 
\label{eq:non_vanishing_nl_coefficient_stochastic_npg_value_baseline_special_claim_2_intermediate_19}
    V_t(a^+) &\coloneqq \frac{5}{18} \cdot \sum_{s=1}^{t-1}  \pi_{\theta_s}(a^+) \cdot (1-\pi_{\theta_s}(a^+)) \le \frac{5}{18} \cdot \sum_{s=1}^{t-1}  \pi_{\theta_s}(a^+).
\end{align}
On the other hand, we have, for all $a^+ \in \gA^+(i)$, when $t > \tau$,
\begin{align}
\label{eq:non_vanishing_nl_coefficient_stochastic_npg_value_baseline_special_claim_2_intermediate_20}
    \sum_{s=\tau}^t P_s(a^+) &=  \sum_{s=\tau}^t \eta \cdot \pi_{\theta_s}(a^+) \cdot ( r(a^+) - \pi_{\theta_s}^\top r ) \ge  \eta \cdot  c' \cdot \sum_{s=\tau}^t   \pi_{\theta_s}(a^+).
    % \sum_{s=1}^t P_s(a^+) & = \sum_{s=1}^t \eta \cdot \pi_{\theta_s}(a^+) \cdot [ r(a^+) - \pi_{\theta_s}^\top r ]
    % \\& = \sum_{s=1}^{\tau-1} \eta \cdot \pi_{\theta_s}(a^+) \cdot [ r(a^+) - \pi_{\theta_s}^\top r ] +  \sum_{s=\tau}^t \eta \cdot \pi_{\theta_s}(a^+) \cdot [ r(a^+) - \pi_{\theta_s}^\top r ]
    % \\& \ge \sum_{s=1}^{\tau-1} \eta \cdot \pi_{\theta_s}(a^+) \cdot [ r(a^+) - \pi_{\theta_s}^\top r ] + \eta \cdot  c' \cdot \sum_{s=\tau}^t   \pi_{\theta_s}(a^+).
\end{align}
%
% {\color{blue}
Hence, when $\gE_1 $ holds, we have, for all $t \ge 1$,
\begin{align}
    \theta_{t}(a^+) &= \EE{\theta_1(a^+)} + Z_t(a^+) + P_1(a^+) + \cdots + P_{\tau-1}(a^+)  + P_{\tau}(a^+) + \cdots +  P_{t-1}(a^+)\qquad \big( \text{by \cref{eq:non_vanishing_nl_coefficient_stochastic_npg_value_baseline_special_claim_2_intermediate_4}} \big) \\
    & 
    \ge  \EE{\theta_1(a^+)}  - %\underbrace{ 
     \left\{
     36 \ \eta \ R_{\max} \ \sqrt{  (V_{t-1}(a^+)+4/3) \cdot \log \left( \frac{  V_{t-1}(a^+)+1  }{ \delta } \right) } + 12 \ \eta \ R_{\max} \ \log(1/\delta)
    \right\}  
    \\
    &\qquad + P_1(a^+) + \cdots + P_{\tau-1}(a^+)  + 
     P_{\tau}(a^+) + \cdots +  P_{t-1}(a^+)  
    \qquad \left( \text{by \cref{eq:global_conv_apply_conc_1}} \right)\\
    & 
    \ge  \EE{\theta_1(a^+)}  - \underbrace{ 
     \left\{
     36 \ \eta \ R_{\max} \ \sqrt{  \left(\frac{4}{3}+  \sum_{s=1}^{t-1} \pi_{\theta_s}(a^+) \right) \cdot \log \left( \frac{  1 +   \sum_{s=1}^{t-1} \pi_{\theta_s}(a^+)  }{ \delta } \right) } + 12 \ \eta \ R_{\max} \log(1/\delta) +  8 \ \eta \ R_{\max}\log 3
    \right\}
    }_{ (\spadesuit) }
    \\
    % & 
    % \ge  \EE{\theta_1(a^+)}  - \underbrace{ 
    %  \left\{
    %  36 \ \eta \ R_{\max} \ \sqrt{  \left(\frac{4}{3}+ \frac{5}{18}\sum_{s=1}^{t-1} \pi_{\theta_s}(a^+) \right) \cdot \log \left( \frac{  1 + \frac{5}{18} \sum_{s=1}^{t-1} \pi_{\theta_s}(a^+)  }{ \delta } \right) } + 12 \ \eta \ R_{\max} \log(1/\delta) +  8 \ \eta \ R_{\max}\log 3
    % \right\}
    % }_{ (\spadesuit) }
    % \\
    &\qquad + P_1(a^+) + \cdots + P_{\tau-1}(a^+)  + 
    \underbrace{ \eta \cdot  c' \cdot \sum_{s=\tau}^t   \pi_{\theta_s}(a^+) }_{ (\heartsuit) }. \qquad \left( \text{by \cref{eq:non_vanishing_nl_coefficient_stochastic_npg_value_baseline_special_claim_2_intermediate_19,eq:non_vanishing_nl_coefficient_stochastic_npg_value_baseline_special_claim_2_intermediate_20}} \right)
    \label{eq:global_conv_apply_conc_2}
\end{align}
% }
If $\sum_{s=1}^\infty \pi_{\theta_s}(a^+) < \infty$, then $\theta_t(a^+)$ is always finite and $\inf_{t\ge 1} \theta_t(a^+) > -\infty$. If $\sum_{s=1}^\infty \pi_{\theta_s}(a^+) = \infty$, we have $(\heartsuit)$ goes to $\infty$ faster than $(\spadesuit)$, and also $\inf_{t\ge 1} \theta_t(a^+) > -\infty$.

Now take any $\omega \in \gE := \{ I = i, i \ne a^*, \cap_{a^+\in \calA^+(i) } N_\infty(a^+)=\infty \}$. Because $\PP{\left( \gE  \setminus \left( \gE  \cap \calE_{1}  \right) \right)} \le \PP{\left( \Omega \setminus \calE_{1}  \right)} \le \delta \to 0$  as $\delta\to 0$, we have that $\PP$-almost surely for all $\omega \in \gE $
there exists $\delta > 0$ such that $\omega \in \gE \cap \gE_1 $
while
\cref{eq:global_conv_apply_conc_2} also holds for this $\delta$.
Take such a $\delta$. By \cref{eq:global_conv_apply_conc_2},
\begin{align}
    \inf_{t\ge 1} \theta_t(a^+)(\omega) > -\infty.
\end{align}
Hence, almost surely on $\gE $,
\begin{align}
    c_1(a^+) &\coloneqq \inf_{t \ge 1}{ \theta_t(a^+) } >- \infty.
\end{align}
Furthermore,
\begin{align}
    c_1 &\coloneqq \min_{a^+\in \gA^+(i)} \inf_{t \ge 1}{ \theta_t(a^+) } = \min_{a^+\in \gA^+(i)} c_1(a^+) >- \infty.
     \label{eq:global_conv_theta_low_good}
\end{align}
Similarly, we can show that, almost surely on $\gE $,
\begin{align}
     c_2 &\coloneqq \max_{a^-\in \gA^-(i)} \sup_{t \ge 1}{ \theta_t(a^-) } < \infty.
     \label{eq:global_conv_theta_up_bad}
\end{align}

\textbf{First case. 2a).} Consider the event,
\begin{align}
\label{eq:non_vanishing_nl_coefficient_stochastic_npg_value_baseline_special_claim_2_case_1_intermediate_1}
    \gE_0 \coloneqq \bigcap\limits_{a^+ \in \gA^+(i)} 
    \underbrace{
    \left\{ N_\infty(a^+) < \infty \right\}}_{\mathcal{E}_0(a^+)},
\end{align}
i.e., any ``good'' action $a^+ \in \gA^+(i)$ has finitely many updates as $t \to \infty$. Pick $a^+ \in \gA^+(i)$, such that $\PP{\left( N_\infty(a^+) < \infty \right) } > 0$. According to the extended Borel-Cantelli lemma (\cref{lem:ebc}), we have, almost surely,
\begin{align}
\label{eq:non_vanishing_nl_coefficient_stochastic_npg_value_baseline_special_claim_2_case_1_intermediate_2}
    \Big\{ \sum_{t \ge 1} \pi_{\theta_t}(a^+)=\infty \Big\} = \left\{ N_\infty(a^+)=\infty \right\}.
\end{align}
Hence, taking complements, we have,
\begin{align}
\label{eq:non_vanishing_nl_coefficient_stochastic_npg_value_baseline_special_claim_2_case_1_intermediate_3}
    \Big\{ \sum_{t \ge 1} \pi_{\theta_t}(a^+)<\infty \Big\} = \left\{N_\infty(a^+)<\infty\right\}
\end{align}
also holds almost surely. 
\iffalse Note that, for all ``good'' action $a^+$,
\begin{align}
\label{eq:good_action_update}
    \theta_{t+1}(a^+) \gets \theta_t(a^+) + \begin{cases}
	    \eta \cdot \left( 1 - \pi_{\theta_t}(a^+) \right) \cdot R_t(a^+), & \text{if } a_t = a^+\, , \\
		- \eta \cdot \pi_{\theta_t}(a^+) \cdot R_t(a_t), & \text{otherwise}\, .
	\end{cases}
\end{align}
Since $N_\infty(a^+)<\infty$, for a given ``good'' action $a^+$, the first update in \cref{eq:good_action_update} will be conducted finitely many times as $t \to \infty$. On the other hand, the second update in \cref{eq:good_action_update} will be conducted for infinitely many times, we have
\begin{align}
    c_3 &\coloneqq \sup_{t \ge 1}{ \theta_t(a^+) } < \infty,
\end{align}
according to $| R_t(a_t) | \le R_{\max}$ and $\sum_{t \ge 1} \pi_{\theta_t}(a^+)<\infty$ by \cref{eq:non_vanishing_nl_coefficient_stochastic_npg_value_baseline_special_claim_2_case_1_intermediate_3}. 
\fi
Next, we have,
\begin{align}
    1 - \sum_{ j \in \gA(i) }{\pi_{\theta_t}(j) } &= \frac{ \sum_{a^+ \in \gA^+(i)}{ e^{\theta_t(a^+)} } + \sum_{a^- \in \gA^-(i)}{ e^{\theta_t(a^-)} } }{ \sum_{a \in [K]}{ e^{\theta_t(a)} } } \\
    &\le \frac{ \sum_{a^+ \in \gA^+(i)}{ e^{\theta_t(a^+)} } + \sum_{a^- \in \gA^-(i)}{ e^{c_2} } }{ \sum_{a \in [K]}{ e^{\theta_t(a)} } } \\
    &= \frac{ \sum_{a^+ \in \gA^+(i)}{ e^{\theta_t(a^+)} } + e^{c_2 - c_1} \cdot \sum_{a^- \in \gA^-(i)}{ e^{c_1} } }{ \sum_{a \in [K]}{ e^{\theta_t(a)} } } \\
    &= \frac{ \sum_{a^+ \in \gA^+(i)}{ e^{\theta_t(a^+)} } + e^{c_2 - c_1} \cdot \frac{|\gA^-(i)| }{|\gA^+(i)|} \cdot \sum_{a^+ \in \gA^+(i)}{ e^{c_1} } }{ \sum_{a \in [K]}{ e^{\theta_t(a)} } } \qquad \left( |\gA^+(i)| \ge 1 \right) \\
    &\le \frac{ \sum_{a^+ \in \gA^+(i)}{ e^{\theta_t(a^+)} } + e^{c_2 - c_1} \cdot \frac{|\gA^-(i)| }{|\gA^+(i)|} \cdot \sum_{a^+ \in \gA^+(i)}{ e^{\theta_t(a^+)} } }{ \sum_{a \in [K]}{ e^{\theta_t(a)} } } \\
    &= \left( 1 + e^{c_2 - c_1} \cdot \frac{|\gA^-(i)| }{|\gA^+(i)|} \right) \cdot \sum_{a^+ \in \gA^+(i)}{ \pi_{\theta_t}(a^+) }.
\end{align}
According to \cref{eq:non_vanishing_nl_coefficient_stochastic_npg_value_baseline_special_claim_2_case_1_intermediate_2},  $N_\infty (a^+) < \infty $ for all $a^+\in \gA^+(i)$. Since (from the above derivation)
\begin{align}
     \sum_{t\ge 1} \sum_{a^- \in \gA^-(i)}{ \pi_{\theta_t}(a^-) } \le   e^{c_2 - c_1} \cdot \frac{|\gA^-(i)| }{|\gA^+(i)|} \cdot  \sum_{t\ge 1}\sum_{a^+ \in \gA^+(i)}{ \pi_{\theta_t}(a^+) } < \infty,
\end{align}
we have that $N_\infty (a^-) < \infty $ for all $a^- \in \gA^-(i)$ according to \cref{eq:non_vanishing_nl_coefficient_stochastic_npg_value_baseline_special_claim_2_case_1_intermediate_3}. Therefore, we have, $\sum_{j\in \gA(i)} N_\infty (i) = \infty $, which indicates that for all $j \in \gA(i)$, the first update in the following \cref{eq:dominating_action_update} will be conducted for infinitely many times, and the second update in \cref{eq:dominating_action_update} will be conducted for finitely many times. 
\begin{align}
\label{eq:dominating_action_update}
    \theta_{t+1}(j) \gets \theta_t(j) + \begin{cases}
	    \eta \cdot \left( 1 - \pi_{\theta_t}(j) \right) \cdot R_t(j), & \text{if } a_t = j\, , \\
		- \eta \cdot \pi_{\theta_t}(j) \cdot R_t(a_t), & \text{otherwise}\, .
	\end{cases}
\end{align}
According to \cref{assp:reward_no_ties}, we have {\color{red} $| \gA(i) | =1$,} and,
\begin{align}
    \theta_{t} (i) &\le \theta_1(i) + c_3 + \eta \cdot R_{\max} \cdot \sum_{s=1}^{t-1}  ( 1 - \pi_{\theta_s}(i)  ) \\
    &\le \theta_1(i) + c_3 + \eta \cdot R_{\max} \cdot\left( 1 + e^{c_2 - c_1} \cdot \frac{|\gA^-(i)| }{|\gA^+(i)|} \right) \cdot \sum_{t\ge 1} \sum_{a^+ \in \gA^+(i)}{ \pi_{\theta_t}(a^+) } < \infty,
\end{align}
where $c_3 < \infty$ upper bounds the cumulative sum of the second update in \cref{eq:dominating_action_update}, since every single update is bounded, and the second update in \cref{eq:dominating_action_update} is conducted for finitely many times. Hence, we have
\begin{align}
\label{eq:non_vanishing_nl_coefficient_stochastic_npg_value_baseline_special_claim_2_case_1_intermediate_21}
    c_4 \coloneqq \sup_{t \ge 1}{\theta_t(i)} < \infty.
\end{align}
%
%%%%%%%%%%%%%%%%%%%%%%%%%%%%%%%%%%%%%%%%%%%%%
Therefore, we have, for all $t \ge 1$, %, almost surely on $\gE^\prime$,
\begin{align}
\label{eq:non_vanishing_nl_coefficient_stochastic_npg_value_baseline_special_claim_2_case_1_intermediate_22}
    \sum_{j \in \gA(i)}{ \pi_{\theta_t}(j) } &= \frac{ \sum_{j \in \gA(i)}{ e^{ \theta_t(j) } } }{ \sum_{j \in \gA(i)}{ e^{ \theta_t(j) } } + \sum_{a^+ \in \gA^+(i)}{ e^{ \theta_t(a^+) } } + \sum_{a^- \in \gA^-(i)}{ e^{ \theta_t(a^-) } } } \\
    &\le \frac{ \sum_{j \in \gA(i)}{ e^{ \theta_t(j) } } }{ \sum_{j \in \gA(i)}{ e^{ \theta_t(j) } } + \sum_{a^+ \in \gA^+(i)}{ e^{ \theta_t(a^+) } } } \qquad \big( e^{ \theta_t(a^-) } > 0 \big) \\
    &\le \frac{ \sum_{j \in \gA(i)}{ e^{ \theta_t(j) } } }{ \sum_{j \in \gA(i)}{ e^{ \theta_t(j) } } + e^{c_1} \cdot \left|\gA^+(i) \right| } \qquad \left( \text{by \cref{eq:global_conv_theta_low_good}} \right) \\
    &\le \frac{ e^{c_4} \cdot \left|\gA(i) \right| }{ e^{c_4} \cdot \left|\gA(i) \right| + e^{c_1} \cdot \left|\gA^+(i) \right| } \qquad \left( \text{by \cref{eq:non_vanishing_nl_coefficient_stochastic_npg_value_baseline_special_claim_2_case_1_intermediate_21}} \right) \\
    &\not\to 1,
\end{align}
which is a contradiction with the assumption of \cref{eq:non_vanishing_nl_coefficient_stochastic_npg_value_baseline_special_claim_2_intermediate_15},
showing that 
% {\color{blue}
$\mathbb{P}(\gE_0 \cap \gE )=0$.
% }

\textbf{Second case. 2b).} Consider the complement $\gE_0^c$ of $\gE_0$, where $\gE_0$ is by \cref{eq:non_vanishing_nl_coefficient_stochastic_npg_value_baseline_special_claim_2_case_1_intermediate_1}. $\gE_0^c$ indicates the event for at least one ``good'' action $a^+ \in \gA^+(i)$ has infinitely many updates as $t \to \infty$.
%
%%%%%%%%%%%%%%%%%%%%%%%%%%%%%%%%%%%%%%%
% We now show that also $\PP(\gE^{\prime\prime})=0$ where
%  $\gE^{\prime\prime}=\gE_0^c \cap \{ I=i,i\ne a^* \}
%  = (\cup_{a^+\in \cA(i)} \{ N_\infty(a^+)=\infty \}) \cap \{ I=i,i\ne a^* \}$.
% It suffices to show that for any $a^+\in \gA^+(i)$, 
% $\PP(  \{ N_\infty(a^+)=\infty \}) \cap \{ I=i,i\ne a^* \} )=0$.

% Thus, fix an arbitrary $a^+ \in \gA^+(i)$ and let
% \[
% \gE^\prime\coloneqq\gE_\infty(a^+)\cap \{I=i,i\ne a^*\},
% \]
% where for $a\in [K]$, $\gE_\infty(a) = \{ N_\infty(a)=\infty \}$. With this notation, the goal is to show that $\PP(\gE^\prime)=0$.%
% \footnote{Here, $\gE^\prime$ is redefined to minimize clutter; the previous definition is not used in this part of the proof.}
% %%%%%%%%%%%%%%%%%%%%%%%%%%%%%%%%%%%%%%%%
We now show that also $\PP(\gE^{\prime})=0$ where
 $\gE^{\prime }=\gE_0^c \cap \{ I=i,i\ne a^* \}
 = (\cup_{a^+\in \cA(i)} \{ N_\infty(a^+)=\infty \}) \cap \{ I=i,i\ne a^* \}$.%
\footnote{Here, $\gE^\prime$ is redefined to minimize clutter; the previous definition is not used in this part of the proof.}
Let $\tilde{\gA}^+(i) \coloneqq \{ a^+ \in \gA^+(i): N_\infty(a^+)=\infty   \}$, and 
\begin{align}
    \gE_\infty \coloneqq\cup_{a^+\in \cA(i)} \{ N_\infty(a^+)=\infty \} = \cup_{a^+\in \tilde{\gA}^+(i)} \{ N_\infty(a^+)=\infty \}.
\end{align}
Then it suffices to show that for any $a^+ \in \tilde{\gA}^+(i)$, $\mathbb{P}{\left( \gE_{\infty}(a^+) \cap \{I=i,i\ne a^*\} \right) } = 0$, where $\gE_{\infty}(a^+) \coloneqq \{ N_\infty(a^+)=\infty \}$. Hence, assume that $\mathbb{P}{\left( \gE_{\infty}(a^+) \cap \{I=i,i\ne a^*\} \right) } > 0$. 

% {\color{blue}
Fix $\delta \in [0, 1]$. 
Using a similar calculation to that of \cref{eq:global_conv_apply_conc_2}, there exists an event 
$\gE_{\delta}$ such that $\PP{\left( \gE_{\delta} \right)} \ge 1 - 2\delta$, and on $\gE_{\delta}$, for all $t\ge \tau$, for all $a^+ \in \tilde{\gA}^+(i)$,
\begin{align}
    \theta_{t}(a^+) &= \EE{\theta_1(a^+)} + Z_t(a^+) + P_1(a^+) + \cdots + P_{\tau-1}(a^+) \qquad \big( \text{by \cref{eq:non_vanishing_nl_coefficient_stochastic_npg_value_baseline_special_claim_2_intermediate_4}} \big) \\
    &\qquad + P_{\tau}(a^+) + \cdots +  P_{t-1}(a^+) \\
    &
    \ge  \EE{\theta_1(a^+)}  - \underbrace{ 
     \left\{
     36 \ \eta \ R_{\max} \ \sqrt{  \left(\frac{4}{3}+  \sum_{s=1}^{t-1} \pi_{\theta_s}(a^+) \right) \cdot \log \left( \frac{  1+  \sum_{s=1}^{t-1} \pi_{\theta_s}(a^+)  }{ \delta } \right) } + 12 \ \eta \ R_{\max} \ \log(1/\delta) +  8 \ \eta \ R_{\max}\log 3
    \right\}
    }_{ (\spadesuit) }
    \\
    % &
    % \ge  \EE{\theta_1(a^+)}  - \underbrace{ 
    %  \left\{
    %  36 \ \eta \ R_{\max} \ \sqrt{  \left(\frac{4}{3}+ \frac{5}{18}\sum_{s=1}^{t-1} \pi_{\theta_s}(a^+) \right) \cdot \log \left( \frac{  1+ \frac{5}{18} \sum_{s=1}^{t-1} \pi_{\theta_s}(a^+)  }{ \delta } \right) } + 12 \ \eta \ R_{\max} \ \log(1/\delta) +  8 \ \eta \ R_{\max}\log 3
    % \right\}
    % }_{ (\spadesuit) }
    % \\
    &\qquad + P_1(a^+) + \cdots + P_{\tau-1}(a^+)  + 
    \underbrace{ \eta \cdot  c' \cdot \sum_{s=\tau}^t   \pi_{\theta_s}(a^+) }_{ (\heartsuit) } .
    \qquad \left( \text{by \cref{eq:global_conv_apply_conc_1,eq:non_vanishing_nl_coefficient_stochastic_npg_value_baseline_special_claim_2_intermediate_19,eq:non_vanishing_nl_coefficient_stochastic_npg_value_baseline_special_claim_2_intermediate_20}} \right)
%    \label{eq:global_conv_apply_concc_5}
\end{align}
% }
On $\gE_{\infty}(a^+) \cap \gE_{\delta}$, $N_\infty(a^+) = \infty$, which indicates that $\sum_{t\ge 1} \pi_{\theta_t}(a^+) = \infty$ according to \cref{eq:non_vanishing_nl_coefficient_stochastic_npg_value_baseline_special_claim_2_case_1_intermediate_2}. 
% Since  
% \begin{align}
%     \sum_{t\ge \tau} P_t(a^+) 
%     &= \sum_{t\ge \tau} \eta \cdot \pi_{\theta_t}(a^+) \cdot ( r(a^+) - \pi_{\theta_t}^\top r )
%      \qquad \left(\text{by  \cref{eq:non_vanishing_nl_coefficient_stochastic_npg_value_baseline_special_claim_2_intermediate_9}}\right)
%      \\&
%      \ge c' \cdot \eta \cdot  \sum_{t\ge \tau}  \pi_{\theta_t}(a^+) \rightarrow \infty
%      \qquad \left(\text{by \cref{eq:non_vanishing_nl_coefficient_stochastic_npg_value_baseline_special_claim_2_intermediate_18}}\right),
% \end{align}
When $t\rightarrow \infty$, both $(\spadesuit)$ and $(\heartsuit)$ go to infinity while $(\heartsuit)$ goes to infinity faster than $(\spadesuit)$. 
Hence, 
we have $\theta_t(a^+) \to \infty$ as $t \to \infty$.

Since $\PP{\left( \gE_{\infty}(a^+) \setminus \left( \gE_{\infty}(a^+) \cap \gE_{\delta} \right) \right)} \to 0$ as $\delta \to 0$, with an argument parallel to that used in the previous analysis (cf. the argument after \cref{eq:global_conv_apply_conc_2}), we have, almost surely on $\gE_{\infty}(a^+)$, 
\begin{align}
\label{eq:non_vanishing_nl_coefficient_stochastic_npg_value_baseline_special_claim_2_case_2_intermediate_3}
    \lim_{t \to \infty}{ \theta_t(a^+) } = \infty,
\end{align}
which implies that there exists $\tau_1\ge 1$ such that on 
$\gE_\infty(a^+)\cap \{I=i,i\ne a^*\}$, we have almost surely that $\tau_1<+\infty$ while we also have that for all $t\ge \tau_1$,  for all $a^+ \in \tilde{\gA}^+(i)$,
\begin{align}
\label{eq:non_vanishing_nl_coefficient_stochastic_npg_value_baseline_special_claim_2_case_2_intermediate_3b}
    \sum_{a^- \in \gA^-(i)}  \frac{ r(i) - r(a^-) }{ \exp\{ \theta_t(a^+) - c_2 \}} < c^\prime \coloneqq \frac{r(a^+)-r(i)}{2},
\end{align}
where $c_2 \le \infty$ is from \cref{eq:global_conv_theta_up_bad}. For all $\bar{a}^+ \in \gA^+(i) \setminus \tilde{\gA}^+(i)$, we have, $ \gN_\infty(\bar{a}^+) < \infty $. Note that,
\begin{align}
\label{eq:good_action_update}
    \theta_{t+1}(\bar{a}^+) \gets \theta_t(\bar{a}^+) + \begin{cases}
	    \eta \cdot \left( 1 - \pi_{\theta_t}(\bar{a}^+) \right) \cdot R_t(\bar{a}^+), & \text{if } a_t = \bar{a}^+\, , \\
		- \eta \cdot \pi_{\theta_t}(\bar{a}^+) \cdot R_t(a_t), & \text{otherwise}\, .
	\end{cases}
\end{align}
Since $N_\infty(\bar{a}^+)<\infty$, the first update in \cref{eq:good_action_update} will be conducted finitely many times as $t \to \infty$. On the other hand, the second update in \cref{eq:good_action_update} will be conducted for infinitely many times. According to $| R_t(a_t) | \le R_{\max}$ and $\sum_{t \ge 1} \pi_{\theta_t}(\bar{a}^+)<\infty$ by \cref{eq:non_vanishing_nl_coefficient_stochastic_npg_value_baseline_special_claim_2_case_1_intermediate_3}, we have,
\begin{align}
    \max_{ \bar{a}^+ \in \gA^+(i) \setminus \tilde{\gA}^+(i) } \sup_{t \ge 1} \theta_t(\bar{a}^+) < \infty.
\end{align}
Recall in \cref{eq:global_conv_theta_up_bad}, we show that  
   $c_2 \coloneqq \max_{a^-\in \gA^-(i)} \sup_{t \ge 1}{ \theta_t(a^-) } < \infty$. Therefore, we have,
\begin{align}
    \max_{ a \in ( \gA^+(i) \cup \gA^-(i) ) \setminus \tilde{\gA}^+(i) } \sup_{t \ge 1} \theta_t(a) < \infty.
\end{align}
Fix any $\tilde{a}^+\in \tilde{\gA}^+(i)$, since $ \lim_{t\rightarrow\infty}  \theta_t(\tilde{a}^+) = \infty$ by \cref{eq:non_vanishing_nl_coefficient_stochastic_npg_value_baseline_special_claim_2_case_2_intermediate_3}, there exists $\tau_2<\infty$ such that when $t\ge \tau_2$,
\begin{align}
    \sum_{ a \in ( \gA^+(i) \cup \gA^-(i) ) \setminus \tilde{\gA}^+(i) } \exp\{ \theta_t(a) \} &\le \frac{1}{2} \cdot   \exp\{\theta_t(\tilde{a}^+)\}, \quad \text{and} \\
    \sum_{ a \in ( \gA^+(i) \cup \gA^-(i) ) \setminus \tilde{\gA}^+(i) } \pi_{\theta_t}(a) &\le \frac{1}{2} \cdot \pi_{\theta_t}(\tilde{a}^+),
\end{align}
which implies that,
\begin{align}
    1 - \sum_{j\in \gA(i)} \pi_t(j) &= \sum_{ a^+ \in \tilde{\gA}^+(i) } \pi_{\theta_t}(a^+) + \sum_{ a \in ( \gA^+(i) \cup \gA^-(i) ) \setminus \tilde{\gA}^+(i) } \pi_{\theta_t}(a)  \\
    &\le \sum_{ a^+ \in \tilde{\gA}^+(i) } \pi_{\theta_t}(a^+) + \frac{1}{2} \cdot \pi_{\theta_t}(\tilde{a}^+) \\
    &\le \frac{3}{2} \cdot \sum_{ a^+ \in \tilde{\gA}^+(i) } \pi_{\theta_t}(a^+).
    \label{eq:global_conv_case_2b_upbd_pi}
\end{align}
For any $\tilde{a}^+ \in \tilde{\gA}^+(i)$, we have,
\begin{align}
\MoveEqLeft
\label{eq:non_vanishing_nl_coefficient_stochastic_npg_value_baseline_special_claim_2_case_2_intermediate_4}
    \pi_{\theta_t}^\top r = \sum_{j \in \gA(i)} \pi_{\theta_t}(j) \cdot r(i) + \sum_{a^- \in \gA^-(i)} \pi_{\theta_t}(a^-) \cdot r(a^-) + \sum_{a^+ \in \gA^+(i)} \pi_{\theta_t}(a^+) \cdot r(a^+) \\
    &= r(i) - \sum_{a^- \in \gA^-(i)} \pi_{\theta_t}(a^-) \cdot \left( r(i) - r(a^-) \right) + \sum_{a^+ \in \gA^+(i)} \pi_{\theta_t}(a^+) \cdot \left( r(a^+) - r(i) \right) \\
    &\ge r(i) - \sum_{a^- \in \gA^-(i)} \pi_{\theta_t}(a^-) \cdot \left( r(i) - r(a^-) \right) + \sum_{a^+ \in \tilde{\gA}^+(i)} \pi_{\theta_t}(a^+) \cdot \left( r(a^+) - r(i) \right) \quad (\text{fewer positive terms}) \\
    &= r(i) - \sum_{a^- \in \gA^-(i)} \pi_{\theta_t}(a^-) \cdot \left( r(i) - r(a^-) \right) + \pi_{\theta_t}(\tilde{a}^+ ) \cdot \left( r(\tilde{a}^+ ) - r(i) \right) + \sum_{a^+ \in \tilde{\gA}^+(i) \setminus \{ \tilde{a}^+ \} }{ \pi_{\theta_t}(a^+) \cdot \left( r(a^+) - r(i) \right)  }.
\end{align}
On $\gE^\prime$, we have, for $t \ge \tau_1$, almost surely
\begin{align}
\label{eq:non_vanishing_nl_coefficient_stochastic_npg_value_baseline_special_claim_2_case_2_intermediate_5}
\MoveEqLeft
    - \sum_{a^- \in \gA^-(i)} \pi_{\theta_t}(a^-) \cdot \left( r(i) - r(a^-) \right) + \pi_{\theta_t}(\tilde{a}^+ ) \cdot \left( r(\tilde{a}^+ ) - r(i) \right) \\
    &= \pi_{\theta_t}(\tilde{a}^+) \cdot \bigg[ \left( r(\tilde{a}^+) - r(i) \right)  - \sum_{a^- \in \gA^-(i)} \frac{ \pi_{\theta_t}(a^-) }{\pi_{\theta_t}(\tilde{a}^+) } \cdot \left( r(i) - r(a^-) \right) \bigg] \\
    &= \pi_{\theta_t}(\tilde{a}^+) \cdot \bigg[ \left( r(\tilde{a}^+) - r(i) \right) - \sum_{a^- \in \gA^-(i)}  \frac{ r(i) - r(a^-) }{ \exp\{ \theta_t(\tilde{a}^+) - \theta_t(a^-) \} }  \bigg] \\
    &\ge \pi_{\theta_t}(\tilde{a}^+) \cdot \bigg[ \left( r(a^+) - r(i) \right) - \sum_{a^- \in \gA^-(i)}  \frac{ r(i) - r(a^-) }{ \exp\{ \theta_t(\tilde{a}^+) - c_2 \} }  \bigg] \qquad \left( \text{by \cref{eq:global_conv_theta_up_bad}} \right) \\
    &> c^\prime \cdot \pi_{\theta_t}(\tilde{a}^+)\,. \qquad \left( \text{by \cref{eq:non_vanishing_nl_coefficient_stochastic_npg_value_baseline_special_claim_2_case_2_intermediate_3b}} \right)
\end{align}
Combining \cref{eq:non_vanishing_nl_coefficient_stochastic_npg_value_baseline_special_claim_2_case_2_intermediate_4,eq:non_vanishing_nl_coefficient_stochastic_npg_value_baseline_special_claim_2_case_2_intermediate_5}, we have,
\begin{align}
    \pi_{\theta_t}^\top r &\ge r(i) + c^\prime \cdot \pi_{\theta_t}(\tilde{a}^+) + \sum_{a^+ \in \tilde{\gA}^+(i) \setminus \{ \tilde{a}^+ \} }{ \pi_{\theta_t}(a^+) \cdot \left( r(a^+) - r(i) \right)  } \\
    &\ge r(i) + c^\prime \cdot \pi_{\theta_t}(\tilde{a}^+) + \min_{a^+ \in \tilde{\gA}^+(i) \setminus \{ \tilde{a}^+ \} }{ \left( r(a^+) - r(i) \right) } \cdot \sum_{a^+ \in \tilde{\gA}^+(i) \setminus \{ \tilde{a}^+ \} }{ \pi_{\theta_t}(a^+)} \\
    &\ge r(i) + c^{\prime\prime} \cdot \sum_{a^+ \in \tilde{\gA}^+(i) }{ \pi_{\theta_t}(a^+)}. \qquad \left( c^{\prime\prime} \coloneqq \min\Big\{ c^\prime, \min_{a^+ \in \tilde{\gA}^+(i) \setminus \{ \tilde{a}^+ \} }{ \left( r(a^+) - r(i) \right) } \Big\} \right)
\end{align}
Therefore, by \cref{eq:non_vanishing_nl_coefficient_stochastic_npg_value_baseline_special_claim_2_intermediate_9}, on 
$\gE^\prime$, we have, for all $t \ge \tau_1$, for any $j \in \gA(i)$, almost surely,
\begin{align}
    P_t(j) &= \eta \cdot \pi_{\theta_t}(j) \cdot ( r(j) - \pi_{\theta_t}^\top r ) \\
    &< - c^{\prime\prime} \cdot \pi_{\theta_t}(j) \cdot \sum_{a^+ \in \tilde{\gA}^+(i)} \pi_{\theta_t}(a^+) < 0.
\end{align}
According to \cref{assp:reward_no_ties}, we have {\color{red}$|\gA(i)|=1$,}, and  $\pi_{\theta_t}(i) \to 1$ as $t \to \infty$. Therefore, there exists $\tau_3 < \infty$, such that for all $t\ge \tau_3$, we have $\pi_{\theta_t}(i) > 1/2$. Hence, when $t\ge \max\{\tau_1, \tau_3 \}$, we have,
\begin{align}
\label{eq:non_vanishing_nl_coefficient_stochastic_npg_value_baseline_special_claim_2_intermediate_9_apply}
    P_t(j) < - \frac{c^{\prime\prime}}{2} \cdot  \sum_{a^+ \in \tilde{\gA}^+(i)} \pi_{\theta_t}(a^+).
\end{align}
% {\color{blue}
Let $\tau^{\prime} \coloneqq \max\{  \tau_1, \tau_2,\tau_3\}$. When $t > \tau'$, for $j\in \gA(i)$, we have,
\begin{align}
    % & S_t^1(j) = \sum_{s=1}^{t-1} \pi_{\theta_s}(j)^2 \cdot (r(j) - \pi_{\theta_s}^\top r )^2 =\sum_{s=1}^{t-1} \left( \frac{ P_s(j) }{\eta} \right)^2,
    % \\&
    V_t (j) &\coloneqq \frac{5}{18} \cdot \sum_{s=1}^{t-1} \pi_{\theta_s}(a) \cdot (1-\pi_{\theta_s}(a) ) \\
    &\le \sum_{s=1}^{\tau'-1} (1- \pi_{\theta_s}(j) ) + \frac{5}{18} \cdot \sum_{s=\tau'}^{t-1} \Big( 1- \sum_{j\in \gA(i)}\pi_{\theta_s}(j) \Big) \\
    &\le \sum_{s=1}^{\tau'-1} (1- \pi_{\theta_s}(j)) + \sum_{s=\tau'}^{t-1} \sum_{a^+ \in \tilde{\gA}^+(i)} \pi_{\theta_s}(a^+).
    \qquad \left(\text{by \cref{eq:global_conv_case_2b_upbd_pi}}\right)
    \label{eq:global_conv_case_2b_upbd_pi_apply}
\end{align}
Hence, for $t\ge \tau'$, when $\gE_1$ holds (defined above \cref{eq:non_vanishing_nl_coefficient_stochastic_npg_value_baseline_special_claim_2_intermediate_14}), we have,
\begin{align}
    \theta_{t}(j) &= \EE{\theta_1(j)} + Z_t(j) + P_1(j) + \cdots + P_{\tau^\prime-1}(j)  + P_{\tau^\prime}(j) + \cdots +  P_{t-1}(j)\qquad \big( \text{by \cref{eq:non_vanishing_nl_coefficient_stochastic_npg_value_baseline_special_claim_2_intermediate_4}} \big) \\
    &\le  \EE{\theta_1(j)}  + %\underbrace{ 
     % \left\{
     36 \ \eta \ R_{\max} \ \sqrt{  (V_t(j)+4/3)\log \left( \frac{  V_t(j)+1  }{ \delta } \right) } + 12 \ \eta \ R_{\max} \ \log(1/\delta)
    % \right\}
    % }_{ (\spadesuit) }
    \\
    &\qquad + P_1(j) + \cdots + P_{\tau'-1}(j)  + 
     P_{\tau'}(j) + \cdots +  P_{t-1}(j)  
    \qquad \left( \text{by \cref{eq:global_conv_apply_conc_1}} \right)\\
    & \le  \EE{\theta_1(j)} +
    \underbrace{
    36 \ \eta \ R_{\max} \ \sqrt{  (V_t(j)+4/3)\log \left( \frac{  V_t(j)+1  }{ \delta } \right) } + 12 \ \eta \ R_{\max} \ \log(1/\delta)
    +  8 \ \eta \ R_{\max}\ \log 3 }_{ (\clubsuit) }
    \\
    &\qquad + P_1(j) + \cdots + P_{\tau'-1}(j)  
    - 
    \underbrace{
    \frac{c^{\prime\prime}}{2} \cdot \sum_{s=\tau'}^{t-1} \sum_{a^+ \in \tilde{A}^+} \pi_{\theta_t}(a^+)
    }_{(\diamondsuit)}.
    \qquad \big( \text{by \cref{eq:non_vanishing_nl_coefficient_stochastic_npg_value_baseline_special_claim_2_intermediate_9_apply}} \big)
    % \label{eq:global_conv_apply_concc_5}
\end{align}
As $t \to \infty$, both $(\clubsuit)$ and $(\diamondsuit)$ go to infinity. According to \cref{eq:global_conv_case_2b_upbd_pi_apply}, we have,  $(\diamondsuit)$ goes to infinity faster than $(\clubsuit)$. Therefore, we have $\sup_{t\ge 1} \theta_t(j) <\infty$.
% }
 
% Let $\gE_\delta' = \calE_1 \cap \calE_4$.
Since $\PP(\gE_1^c) \le \delta \to 0$ as $\delta\to 0$, with an argument parallel to that used in the previous analysis (cf. the argument after \cref{eq:global_conv_apply_conc_2}), 
we get that there exists a random constant $c_5(j)$, such that almost surely on $\gE^\prime$, $c_5(j)<\infty$ and
$\sup_{t\ge \tau^\prime} \theta_t(j) \le c_5(j)$. Denote $c_5 \coloneqq \max_{j\in \gA(i)}{ c_5(j) }$. Then,
almost surely on $\gE^\prime$, $c_5<\infty$ and
\begin{align}
\label{eq:non_vanishing_nl_coefficient_stochastic_npg_value_baseline_special_claim_2_case_2_intermediate_9}
\sup_{t\ge \tau^\prime}\max_{j\in \gA(i)} \theta_t(j) \le c_5\,.
\end{align}
According to
\cref{eq:non_vanishing_nl_coefficient_stochastic_npg_value_baseline_special_claim_2_case_2_intermediate_3}, there exists $a^+ \in \gA^+(i)$, $\tau''\ge 1$, such that almost surely on $\gE^\prime$,
$\tau''<\infty$ while we also have 
\begin{align}
\inf_{t\ge \tau''} \theta_t(a^+)\ge 0,
\label{eq:ttapp}
\end{align}
 for all $t\ge \tau''$.
Hence, on $\gE^\prime$, almost surely for all $t\ge \max\{ \tau',\tau''\}$, 
\begin{align}
\label{eq:non_vanishing_nl_coefficient_stochastic_npg_value_baseline_special_claim_2_case_2_intermediate_10}
    \sum_{j \in \gA(i)}{ \pi_{\theta_t}(j) } &= \frac{ \sum_{j \in \gA(i)}{ e^{ \theta_t(j) } } }{ \sum_{j \in \gA(i)}{ e^{ \theta_t(j) } } + \sum_{\tilde{a}^+ \in \gA^+(i)}{ e^{ \theta_t(\tilde{a}^+) } } + \sum_{a^- \in \gA^-(i)}{ e^{ \theta_t(a^-) } } } \\
    &\le \frac{ \sum_{j \in \gA(i)}{ e^{ \theta_t(j) } } }{ \sum_{j \in \gA(i)}{ e^{ \theta_t(j) } } + 	{ e^{ \theta_t(a^+) } } } \qquad \big( e^{ \theta_t(k) } > 0 \text{ for any } k\in [K] \big) \\
    &\le \frac{ \sum_{j \in \gA(i)}{ e^{ \theta_t(j) } } }
    { \sum_{j \in \gA(i)}{ e^{ \theta_t(j) } } 
    + 1 } \qquad \left( \text{by \cref{eq:ttapp} } \right) \\
    &\le \frac{ e^{c_5} \cdot \left|\gA(i) \right| }{ e^{c_5} \cdot \left|\gA(i) \right| + 1 } \qquad \left( \text{by \cref{eq:non_vanishing_nl_coefficient_stochastic_npg_value_baseline_special_claim_2_case_2_intermediate_9}} \right) \\
    &\not\to 1\,.
\end{align}
Hence, $\mathbb{P}( \gE^\prime )=0$, finishing the proof. 
\end{proof}

\textbf{\cref{lem:non_uniform_lojasiewicz_softmax_special}} (Non-uniform \L{}ojasiewicz (N\L{}), \citet[Lemma 3]{mei2020global})\textbf{.}
Assume $r$ has a unique maximizing action $a^*$. Let $\pi^* = \argmax_{\pi \in \Delta}{ \pi^\top r}$. Then, 
\begin{align}
\label{eq:non_uniform_lojasiewicz_softmax_special_result_1_appendix}
    \bigg\| \frac{d \pi_\theta^\top r}{d \theta} \bigg\|_2 \ge \pi_\theta(a^*) \cdot ( \pi^* - \pi_\theta )^\top r\,.
\end{align}
\begin{proof}
Using the definition of softmax Jacobian, we have,
\begin{align}
    \bigg\| \frac{d \pi_\theta^\top r}{d \theta} \bigg\|_2^2 &= \sum_{a \in [K]}{ \pi_{\theta}(a)^2 \cdot \left( r(a) - \pi_{\theta}^\top r \right)^2 } \\ &\ge \pi_{\theta}(a^*)^2 \cdot \left( r(a^*) - \pi_{\theta}^\top r \right)^2, \qquad \left( \text{fewer terms} \right)
\end{align}
which implies \cref{eq:non_uniform_lojasiewicz_softmax_special_result_1_appendix}.
\end{proof}

\subsection{Proof of \cref{thm:convergence_rate_and_regret_gradient_bandit_sampled_reward}}
\label{pf:thm_convergence_rate_and_regret_gradient_bandit_sampled_reward}
\textbf{\cref{thm:convergence_rate_and_regret_gradient_bandit_sampled_reward}} (Convergence rate and regret)\textbf{.}
Using \cref{alg:gradient_bandit_algorithm_sampled_reward} with $\eta = \frac{\Delta^2}{40 \cdot K^{3/2} \cdot R_{\max}^3 }$, we have, for all $t \ge 1$,
\begin{align}
\label{eq:convergence_rate_and_regret_gradient_bandit_sampled_reward_result_1_appendix}
    \EE{ \left( \pi^* - \pi_{\theta_t} \right)^\top r } &\le \frac{C}{t}, \quad \text{and} \\
\label{eq:convergence_rate_and_regret_gradient_bandit_sampled_reward_result_2_appendix}
    \expectation{ \bigg[ \sum_{t=1}^{T}{ \left( \pi^* - \pi_{\theta_t} \right)^\top r } \bigg] } &\le \min\{ \sqrt{2 \, R_{\max} \, C \, T}, C \, \log{T} + 1  \},
\end{align}
where $C \coloneqq \frac{80 \cdot K^{3/2} \cdot R_{\max}^3 }{\Delta^2 \cdot \EE{ c^2 }} $, and $c \coloneqq \inf_{t \ge 1}{\pi_{\theta_t}(a^*)} > 0$ is from \cref{thm:asymp_global_converg_gradient_bandit_sampled_reward}.
\begin{proof}
\textbf{First part, \cref{eq:convergence_rate_and_regret_gradient_bandit_sampled_reward_result_1_appendix}.} According to \cref{lem:constant_learning_rates_expected_progress_sampled_reward}, we have,
\begin{align}
\label{eq:convergence_rate_and_regret_gradient_bandit_sampled_reward_result_1_intermediate_1}
    \EEt{ \pi_{\theta_{t+1}}^\top r } - \pi_{\theta_t}^\top r &\ge \frac{\Delta^2}{80 \cdot K^{3/2} \cdot R_{\max}^3 } \cdot \bigg\| \frac{d \pi_{\theta_t}^\top r}{d \theta_t} \bigg\|_2^2 \\
    &\ge \frac{\Delta^2 \cdot \pi_{\theta_t}(a^*)^2 }{80 \cdot K^{3/2} \cdot R_{\max}^3  }  \cdot \left( r(a^*) - \pi_{\theta_t}^\top r \right)^2 \qquad \left( \text{by \cref{lem:non_uniform_lojasiewicz_softmax_special}} \right) \\
    &\ge \frac{\Delta^2 \cdot \inf_{t \ge 1} \pi_{\theta_t}(a^*)^2 }{80 \cdot K^{3/2} \cdot R_{\max}^3  }  \cdot \left( r(a^*) - \pi_{\theta_t}^\top r \right)^2 \\
    &= \frac{\Delta^2 \cdot c^2 }{80 \cdot K^{3/2} \cdot R_{\max}^3  }  \cdot \left( r(a^*) - \pi_{\theta_t}^\top r \right)^2. \qquad \left( \text{by \cref{thm:asymp_global_converg_gradient_bandit_sampled_reward}} \right) 
\end{align}
Denote $\delta(\theta_t) \coloneqq \left( \pi^* - \pi_{\theta_t} \right)^\top r$ as the sub-optimality gap. We have,
\begin{align}
\label{eq:convergence_rate_and_regret_gradient_bandit_sampled_reward_result_1_intermediate_2}
    \delta(\theta_t) - \EEt{ \delta(\theta_{t+1})} &= \left( \pi^* - \pi_{\theta_t} \right)^\top r - \left( \pi^* - \EEt{\pi_{\theta_{t+1}}} \right)^\top r \\
    &= \EEt{\pi_{\theta_{t+1}}^\top r} - \pi_{\theta_t}^\top r \\
    &\ge \frac{\Delta^2 \cdot c^2 }{80 \cdot K^{3/2} \cdot R_{\max}^3  } \cdot \delta(\theta_t)^2.
\end{align}
Taking expectation, we have,
\begin{align}
\label{eq:convergence_rate_and_regret_gradient_bandit_sampled_reward_result_1_intermediate_3}
    \expectation{ [ \delta(\theta_t) ]} - \expectation{ [ \delta(\theta_{t+1}) ]} &\ge \frac{\Delta^2 \cdot \EE{ c^2 }  }{80 \cdot K^{3/2} \cdot R_{\max}^3  } \cdot \expectation{ [ \delta(\theta_t)^2 ] } \qquad \left( c \coloneqq \inf_{t \ge 1}{\pi_{\theta_t}(a^*)} > 0 \text{ is independent with } t \right) \\
    &\ge \frac{\Delta^2 \cdot \EE{ c^2 }  }{80 \cdot K^{3/2} \cdot R_{\max}^3  } \cdot \left( \expectation{ [ \delta(\theta_t) ] } \right)^2 \qquad \left( \text{by Jensen's inequality}\right) \\
    &= \frac{1}{C} \cdot \left( \expectation{ [ \delta(\theta_t) ] } \right)^2.
\end{align}
Therefore, we have,
\begin{align}
\label{eq:convergence_rate_and_regret_gradient_bandit_sampled_reward_result_1_intermediate_4}
    \frac{1}{ \expectation{ [ \delta(\theta_t) ]} } &= \frac{1}{\expectation{ [ \delta(\theta_{1}) ]}} + \sum_{s=1}^{t-1}{ \left[ \frac{1}{\expectation{ [ \delta(\theta_{s+1}) ]}} - \frac{1}{\expectation{ [ \delta(\theta_{s}) ]}} \right] } \\
    &= \frac{1}{\expectation{ [ \delta(\theta_{1}) ]}} + \sum_{s=1}^{t-1}{ \frac{1}{\expectation{ [ \delta(\theta_{s+1}) ]} \cdot \expectation{ [ \delta(\theta_{s}) ]} } \cdot \left( \expectation{ [ \delta(\theta_{s}) ]} - \expectation{ [ \delta(\theta_{s+1}) ]} \right) } \\
    &\ge \frac{1}{\expectation{ [ \delta(\theta_{1}) ]}} + \sum_{s=1}^{t-1}{ \frac{1}{\expectation{ [ \delta(\theta_{s+1}) ]} \cdot \expectation{ [ \delta(\theta_{s}) ]} } \cdot \frac{1}{C} \cdot \left( \expectation{ [ \delta(\theta_s) ] } \right)^2 } \qquad \left( \text{by \cref{eq:convergence_rate_and_regret_gradient_bandit_sampled_reward_result_1_intermediate_3}} \right)  \\
    &\ge \frac{1}{\expectation{ [ \delta(\theta_{1}) ]}} + \sum_{s=1}^{t-1}{ \frac{1}{C} }  \qquad \left(  \expectation{ [ \delta(\theta_{s}) ]} \ge  \expectation{ [ \delta(\theta_{s+1}) ]} > 0 \right) \\
    &= \frac{1}{\expectation{ [ \delta(\theta_{1}) ]}} + \frac{1}{C} \cdot \left( t - 1 \right) \\
    &\ge \frac{t}{C}, \qquad \left( \expectation{ [ \delta(\theta_{1}) ]} \le 2 \, R_{\max} \le C = \max_{s \le t} \frac{80 \cdot K^{3/2} \cdot R_{\max}^3 }{\Delta^2 \cdot \EE{ c^2 }}  \right)
\end{align}
which implies \cref{eq:convergence_rate_and_regret_gradient_bandit_sampled_reward_result_1_appendix}.

\textbf{Second part, \cref{eq:convergence_rate_and_regret_gradient_bandit_sampled_reward_result_2_appendix}.} According to \cref{eq:convergence_rate_and_regret_gradient_bandit_sampled_reward_result_1_intermediate_4},we have,
\begin{align}
\label{eq:convergence_rate_and_regret_gradient_bandit_sampled_reward_result_2_intermediate_1}
    \expectation{ \bigg[ \sum_{t=1}^{T}{ \delta(\theta_t) } \bigg] } = \sum_{t=1}^{T}{ \EE{ \delta(\theta_t) } } \le \sum_{t=1}^{T}{ \frac{C}{t} } \le C \cdot \log{T} + 1.
\end{align}
On the other hand, we have
\begin{align}
\label{eq:convergence_rate_and_regret_gradient_bandit_sampled_reward_result_2_intermediate_2}
    \sum_{t=1}^{T}{ \EE{ \delta(\theta_t) } } &\le \sqrt{T} \cdot \left[ \sum_{t=1}^T{ \left( \expectation{ [ \delta(\theta_t) ] } \right)^2 } \right]^{\frac{1}{2}} \qquad \left( \text{by Cauchy–Schwarz} \right) \\
    &\le \sqrt{T} \cdot \left[ \sum_{t=1}^T{ C \cdot \left( \EE{\delta(\theta_t)} - \EE{\delta(\theta_{t+1})} \right)  } \right]^{\frac{1}{2}} \qquad \left( \text{by \cref{eq:convergence_rate_and_regret_gradient_bandit_sampled_reward_result_1_intermediate_3}} \right) \\
    &= \sqrt{C \cdot T \cdot \left( \EE{\delta(\theta_1)} - \EE{\delta(\theta_{T+1})} \right) } \\
    &\le \sqrt{C \cdot T \cdot 2 \cdot R_{\max} }, \qquad \left( \EE{\delta(\theta_{T+1})} \ge 0. \text{ and } \EE{\delta(\theta_1)} \le 2 \, R_{\max} \right) 
\end{align}
Combining \cref{eq:convergence_rate_and_regret_gradient_bandit_sampled_reward_result_2_intermediate_1,eq:convergence_rate_and_regret_gradient_bandit_sampled_reward_result_2_intermediate_2}, we have \cref{eq:convergence_rate_and_regret_gradient_bandit_sampled_reward_result_2_appendix}.
\end{proof}

\section{Proofs for Using Baselines}
\label{sec:proofs_for_gradient_bandit_algorithm_sampled_reward_baselines}

The following \cref{alg:gradient_bandit_algorithm_sampled_reward_baselines} is same as the gradient bandit algorithm in \citet[Section 2.8]{sutton2018reinforcement}.

\begin{algorithm}[ht]
   \caption{Gradient bandit algorithm with baselines}
\begin{algorithmic}
   \STATE {\bfseries Input:} initial parameters $\theta_1 \in \sR^K$, learning rate $\eta > 0$.
   \STATE {\bfseries Output:} policies $\pi_{\theta_t} = \softmax(\theta_t)$.
   \WHILE{$t \ge 1$}
   \STATE Sample one action $a_t \sim \pi_{\theta_t}(\cdot)$.
   \STATE Observe one reward sample $R_t(a_t)\sim P_{a_t}$.
   \STATE Choose a baseline $B_t \in \sR$.
   \FOR{all $a \in [K]$}
   \IF{$a = a_t$}
   \STATE $\theta_{t+1}(a) \gets \theta_t(a) + \eta \cdot \left( 1 - \pi_{\theta_t}(a) \right) \cdot \left( R_t(a_t) - B_t \right)$.
   \ELSE
   \STATE $\theta_{t+1}(a) \gets \theta_t(a) - \eta \cdot \pi_{\theta_t}(a) \cdot \left( R_t(a_t) - B_t \right)$.
   \ENDIF
   \ENDFOR
   \ENDWHILE
\end{algorithmic}
\label{alg:gradient_bandit_algorithm_sampled_reward_baselines}
\end{algorithm}

\begin{proposition}
\label{prop:gradient_bandit_algorithm_equivalent_to_stochastic_gradient_ascent_sampled_reward_baselines}
\cref{alg:gradient_bandit_algorithm_sampled_reward_baselines} is equivalent to the following stochastic gradient ascent update on $\pi_{\theta}^\top r$.
\begin{align}
\label{eq:stochastic_gradient_ascent_sampled_reward_baselines}
    \theta_{t+1} &\gets  \theta_{t} + \eta \cdot \frac{d \pi_{\theta_t}^\top \big( \hat{r}_t - \hat{b}_t \big) }{d \theta_t} \\
    &= \theta_t + \eta \cdot \left(  \diagonalmatrix{(\pi_{\theta_t})} - \pi_{\theta_t} \pi_{\theta_t}^\top \right) \big( \hat{r}_t - \hat{b}_t \big),
\end{align}
where $\left( \frac{d \pi_{\theta}}{d \theta} \right)^\top = \diagonalmatrix{(\pi_{\theta})} - \pi_{\theta} \pi_{\theta}^\top $ is the Jacobian of $\theta \mapsto \pi_\theta \coloneqq \softmax(\theta)$, and $\hat{r}_t(a) \coloneqq \frac{ \sI\left\{ a_t = a \right\} }{ \pi_{\theta_t}(a) } \cdot R_t(a)$ for all $a \in [K]$ is the importance sampling (IS) estimator, and we set $R_t(a)=0$ for all $a \not= a_t$. The baseline is defined as $\hat{b}_t(a) \coloneqq \frac{ \sI\left\{ a_t = a \right\} }{ \pi_{\theta_t}(a) } \cdot B_t \,$ for all $a \in [K]$.
\end{proposition}
\begin{proof}
Using the definition of softmax Jacobian, $\hat{r}_t$ and $\hat{b}_t$, we have, for all $a \in [K]$,
\begin{align}
    \theta_{t+1}(a) &\gets \theta_t(a) + \eta \cdot \pi_{\theta_t}(a) \cdot \left( \hat{r}_t(a) - \hat{b}_t(a) - \pi_{\theta_t}^\top \big( \hat{r}_t - \hat{b}_t \big) \right) \\
    &= \theta_t(a) + \eta \cdot \pi_{\theta_t}(a) \cdot \left( \hat{r}_t(a) -\hat{b}_t(a) - \left( R_t(a_t) - B_t \right) \right) \\
    &= \theta_t(a) + \begin{cases}
		\eta \cdot \left( 1 - \pi_{\theta_t}(a) \right) \cdot \left( R_t(a_t) - B_t \right), & \text{if } a_t = a\, , \\
		- \eta \cdot \pi_{\theta_t}(a) \cdot \left( R_t(a_t) - B_t \right), & \text{otherwise}\,.
    \end{cases}
    \qedhere
\end{align}
\end{proof}

\begin{lemma}[Unbiased stochastic gradient with bounded variance / scale]
\label{lem:unbiased_stochastic_gradient_bounded_scale_sampled_reward_baselines}
Using \cref{alg:gradient_bandit_algorithm_sampled_reward_baselines}, we have, for all $t \ge 1$,
\begin{align}
\label{eq:unbiased_stochastic_gradient_bounded_scale_sampled_reward_baselines_result_1_appendix}
    &\mathbb{E}_t{ \bigg[ \frac{d \pi_{\theta_t}^\top \big( \hat{r}_t - \hat{b}_t \big) }{d \theta_t} \bigg] } = \frac{d \pi_{\theta_t}^\top r}{d \theta_t }, \\
\label{eq:unbiased_stochastic_gradient_bounded_scale_sampled_reward_baselines_result_2_appendix}
    &\mathbb{E}_t{ \left[ \bigg\| \frac{d \pi_{\theta_t}^\top \big( \hat{r}_t - \hat{b}_t \big) }{d \theta_t} \bigg\|_2^2 \right] } \le 2 \, \bar{R}_{\max}^2,
\end{align}
where $\EEt{\cdot}$ is on randomness from the on-policy sampling $a_t \sim \pi_{\theta_t}(\cdot)$ and reward sampling $R_t(a_t)\sim P_{a_t}$, and $\bar{R}_{\max}$ is the range of reward minus baselines, i.e.,
\begin{align}
\label{eq:unbiased_stochastic_gradient_bounded_scale_sampled_reward_baselines_result_3_appendix}
    R_t(a_t) - B_t \in [ - \bar{R}_{\max}, \bar{R}_{\max} ].
\end{align}
\end{lemma}
\begin{proof}
\textbf{First part, \cref{eq:unbiased_stochastic_gradient_bounded_scale_sampled_reward_baselines_result_1_appendix}.} For all action $a \in [K]$, the true softmax PG is,
\begin{align}
\label{eq:unbiased_stochastic_gradient_bounded_scale_sampled_reward_baselines_result_1_intermediate_1}
    \frac{d \pi_{\theta_t}^\top r}{d \theta_t(a)} = \pi_{\theta_t}(a) \cdot \left( r(a) - \pi_{\theta_t}^\top r \right).
\end{align}
For all $a \in [K]$, the stochastic softmax PG is,
\begin{align}
\label{eq:unbiased_stochastic_gradient_bounded_scale_sampled_reward_baselines_result_1_intermediate_2}
    \frac{d \pi_{\theta_t}^\top \big( \hat{r}_t - \hat{b}_t \big) }{d \theta_t(a)} &= \pi_{\theta_t}(a) \cdot \left( \hat{r}_t(a) - \hat{b}_t(a) - \pi_{\theta_t}^\top \big( \hat{r}_t - \hat{b}_t \big) \right) \\
    &= \pi_{\theta_t}(a) \cdot \left( \hat{r}_t(a) -\hat{b}_t(a) - \left( R_t(a_t) - B_t \right) \right) \\
    &= \left( \sI\left\{ a_t = a \right\} - \pi_{\theta_t}(a) \right) \cdot \left( R_t(a_t) - B_t \right).
\end{align}
For the sampled action $a_t$, we have,
\begin{align}
\label{eq:unbiased_stochastic_gradient_bounded_scale_sampled_reward_baselines_result_1_intermediate_3}
    \expectation_{ R_t(a_t) \sim P_{a_t} }{ \bigg[ \frac{d \pi_{\theta_t}^\top \big( \hat{r}_t - \hat{b}_t \big) }{d \theta_t(a_t)} \bigg] } &= \expectation_{ R_t(a_t) \sim P_{a_t} }{ \Big[ \left( 1 - \pi_{\theta_t}(a_t) \right) \cdot \left( R_t(a_t) - B_t \right) \Big] } \\
    &= \left( 1 - \pi_{\theta_t}(a_t) \right) \cdot \expectation_{ R_t(a_t) \sim P_{a_t} }{ \Big[ R_t(a_t) - B_t \Big] } \\
    &= \left( 1 - \pi_{\theta_t}(a_t) \right) \cdot \left( r(a_t) - B_t \right).
\end{align}
For any other not sampled action $a \not= a_t$, we have,
\begin{align}
\label{eq:unbiased_stochastic_gradient_bounded_scale_sampled_reward_baselines_result_1_intermediate_4}
    \expectation_{ R_t(a_t) \sim P_{a_t} }{ \bigg[ \frac{d \pi_{\theta_t}^\top \big( \hat{r}_t - \hat{b}_t \big) }{d \theta_t(a)} \bigg] } &= \expectation_{ R_t(a_t) \sim P_{a_t} }{ \Big[ - \pi_{\theta_t}(a) \cdot \left( R_t(a_t) - B_t \right) \Big] } \\
    &= - \pi_{\theta_t}(a) \cdot \expectation_{ R_t(a_t) \sim P_{a_t} }{ \Big[ R_t(a_t) - B_t \Big] } \\
    &= - \pi_{\theta_t}(a) \cdot \left( r(a_t) - B_t \right).
\end{align}
Combing \cref{eq:unbiased_stochastic_gradient_bounded_scale_sampled_reward_baselines_result_1_intermediate_3,eq:unbiased_stochastic_gradient_bounded_scale_sampled_reward_baselines_result_1_intermediate_4}, we have, for all $a \in [K]$,
\begin{align}
\label{eq:unbiased_stochastic_gradient_bounded_scale_sampled_reward_baselines_result_1_intermediate_5}
    \expectation_{ R_t(a_t) \sim P_{a_t} }{ \bigg[ \frac{d \pi_{\theta_t}^\top \big( \hat{r}_t - \hat{b}_t \big) }{d \theta_t(a)} \bigg] } = \left( \sI\left\{ a_t = a \right\} - \pi_{\theta_t}(a) \right) \cdot \left( r(a_t) - B_t \right).
\end{align}
Taking expectation over $a_t \sim \pi_{\theta_t}(\cdot)$, we have,
\begin{align}
\label{eq:unbiased_stochastic_gradient_bounded_scale_sampled_reward_baselines_result_1_intermediate_6}
    \mathbb{E}_t{ \bigg[ \frac{d \pi_{\theta_t}^\top \big( \hat{r}_t - \hat{b}_t \big) }{d \theta_t(a)} \bigg] } &= \probability{\left( a_t = a \right) } \cdot \expectation_{ R_t(a_t) \sim P_{a_t} }{ \bigg[ \frac{d \pi_{\theta_t}^\top \big( \hat{r}_t - \hat{b}_t \big) }{d \theta_t(a)} \ \Big| \ a_t = a \bigg] } 
    \\& \qquad
    + \probability{\left( a_t \not= a \right) } \cdot \expectation_{ R_t(a_t) \sim P_{a_t} }{ \bigg[ \frac{d \pi_{\theta_t}^\top \big( \hat{r}_t - \hat{b}_t \big) }{d \theta_t(a)} \ \Big| \ a_t \not= a \bigg] } \\
    &= \pi_{\theta_t}(a) \cdot \left( 1 - \pi_{\theta_t}(a) \right) \cdot \left( r(a) - B_t \right) + \sum_{a^\prime \not= a}{ \pi_{\theta_t}(a^\prime) \cdot \left( - \pi_{\theta_t}(a) \right) \cdot \left( r(a^\prime) - B_t \right) } \\
    &= \pi_{\theta_t}(a) \cdot \sum_{a^\prime \not= a}{ \pi_{\theta_t}(a^\prime) \cdot \Big[ \left( r(a) - B_t \right) - \left( r(a^\prime) - B_t \right) \Big] } \\
    &= \pi_{\theta_t}(a) \cdot \left( r(a) - \pi_{\theta_t}^\top r \right).
\end{align}
Combining \cref{eq:unbiased_stochastic_gradient_bounded_scale_sampled_reward_baselines_result_1_intermediate_1,eq:unbiased_stochastic_gradient_bounded_scale_sampled_reward_baselines_result_1_intermediate_6}, we have, for all $a \in [K]$,
\begin{align}
\label{eq:unbiased_stochastic_gradient_bounded_scale_sampled_reward_baseline_result_1_intermediate_7}
    \mathbb{E}_t{ \bigg[ \frac{d \pi_{\theta_t}^\top \big( \hat{r}_t - \hat{b}_t \big) }{d \theta_t(a)} \bigg] } = \frac{d \pi_{\theta_t}^\top r}{d \theta_t(a)},
\end{align}
which implies \cref{eq:unbiased_stochastic_gradient_bounded_scale_sampled_reward_baselines_result_1_appendix} since $a \in [K]$ is arbitrary.

\textbf{Second part, \cref{eq:unbiased_stochastic_gradient_bounded_scale_sampled_reward_baselines_result_2_appendix}.} The squared stochastic PG norm is,
\begin{align}
\label{eq:unbiased_stochastic_gradient_bounded_scale_sampled_reward_baselines_result_2_intermediate_1}
    \bigg\| \frac{d \pi_{\theta_t}^\top \big( \hat{r}_t - \hat{b}_t \big) }{d \theta_t} \bigg\|_2^2 &= \sum_{a \in [K]}{ \left( \frac{d \pi_{\theta_t}^\top \big( \hat{r}_t - \hat{b}_t \big) }{d \theta_t(a)} \right)^2 } \\
    &= \sum_{a \in [K]}{ \left( \sI\left\{ a_t = a \right\} - \pi_{\theta_t}(a) \right)^2 \cdot \left( R_t(a_t) - B_t \right)^2 } \qquad \left( \text{by \cref{eq:unbiased_stochastic_gradient_bounded_scale_sampled_reward_baselines_result_1_intermediate_2}} \right) \\
    &\le \bar{R}_{\max}^2 \cdot \sum_{a \in [K]}{ \left( \sI\left\{ a_t = a \right\} - \pi_{\theta_t}(a) \right)^2 } \qquad \left( \text{by \cref{eq:unbiased_stochastic_gradient_bounded_scale_sampled_reward_baselines_result_3_appendix}} \right) \\
    &= \bar{R}_{\max}^2 \cdot \bigg[ \left( 1 - \pi_{\theta_t}(a_t) \right)^2 + \sum_{a \not= a_t}{ \pi_{\theta_t}(a)^2 } \bigg] \\
    &\le \bar{R}_{\max}^2 \cdot \bigg[ \left( 1 - \pi_{\theta_t}(a_t) \right)^2 + \Big( \sum_{a \not= a_t}{ \pi_{\theta_t}(a) } \Big)^2 \bigg] \qquad \left( \left\| x \right\|_2 \le \left\| x \right\|_1 \right) \\
    &= 2 \cdot \bar{R}_{\max}^2 \cdot \left( 1 - \pi_{\theta_t}(a_t) \right)^2.
\end{align}
Therefore, we have, for all $a \in [K]$, conditioning on $a_t = a$,
\begin{align}
\label{eq:unbiased_stochastic_gradient_bounded_scale_sampled_reward_baselines_result_2_intermediate_2}
    \Bigg[ \bigg\| \frac{d \pi_{\theta_t}^\top \big( \hat{r}_t - \hat{b}_t \big) }{d \theta_t} \bigg\|_2^2 \ \Big| \ a_t = a \Bigg] \le 2 \cdot \bar{R}_{\max}^2 \cdot \left( 1 - \pi_{\theta_t}(a) \right)^2.
\end{align}
Taking expectation over $a_t \sim \pi_{\theta_t}(\cdot)$, we have,
\begin{align}
\label{eq:unbiased_stochastic_gradient_bounded_scale_sampled_reward_baselines_result_2_intermediate_3}
    \mathbb{E}_t{ \left[ \bigg\| \frac{d \pi_{\theta_t}^\top \big( \hat{r}_t - \hat{b}_t \big) }{d \theta_t} \bigg\|_2^2 \right] } &= \sum_{a \in [K]}{ \probability{\left( a_t = a \right) } \cdot \Bigg[ \bigg\| \frac{d \pi_{\theta_t}^\top \big( \hat{r}_t - \hat{b}_t \big) }{d \theta_t} \bigg\|_2^2 \ \Big| \ a_t = a \Bigg] } \\
    &\le \sum_{a \in [K]}{ \pi_{\theta_t}(a) \cdot 2 \cdot \bar{R}_{\max}^2 \cdot \left( 1 - \pi_{\theta_t}(a) \right)^2 } \\
    &\le 2 \cdot \bar{R}_{\max}^2 \cdot \sum_{a \in [K]}{ \pi_{\theta_t}(a)} \qquad \left( \pi_{\theta_t}(a) \in (0, 1) \text{ for all } a \in [K] \right) \\
    &= 2 \,  \bar{R}_{\max}^2. \qedhere
\end{align}
\end{proof}

\begin{lemma}[NS between iterates]
\label{lem:non_uniform_smoothness_special_two_iterations_baselines}
Using \cref{alg:gradient_bandit_algorithm_sampled_reward_baselines} with $\eta \in \big(0, 2 / (9 \, \bar{R}_{\max} ) \big)$, we have, for all $t \ge 1$,
\begin{align}
    D(\theta_{t+1}, \theta_t) \coloneqq \left| ( \pi_{\theta_{t+1}} - \pi_{\theta_t})^\top r - \Big\langle \frac{d \pi_{\theta_t}^\top r}{d \theta_t}, \theta_{t+1} - \theta_t \Big\rangle \right| \le \frac{\beta(\theta_t)}{2} \cdot \| \theta_{t+1} - \theta_t \|_2^2,
\end{align}
where $\bar{R}_{\max}$ is from \cref{eq:unbiased_stochastic_gradient_bounded_scale_sampled_reward_baselines_result_3_appendix}, and
\begin{align}
    \beta(\theta_t) = \frac{6}{2 - 9 \cdot \bar{R}_{\max} \cdot \eta } \cdot \bigg\| \frac{d \pi_{\theta_t}^\top r}{d \theta_t} \bigg\|_2.
\end{align}
\end{lemma}
\begin{proof}
In the proofs for \cref{lem:non_uniform_smoothness_special_two_iterations}, replacing $R_{\max}$ with $\bar{R}_{\max}$, and replacing $\frac{d \pi_{\theta_t}^\top \hat{r}_t }{d \theta_t}$ with $\frac{d \pi_{\theta_t}^\top ( \hat{r}_t - \hat{b}_t) }{d \theta_t}$, we have the results.
\end{proof}

\textbf{\cref{lem:strong_growth_conditions_sampled_reward_baselines}} (Strong growth conditions / Self-bounding noise property)\textbf{.}
Using \cref{alg:gradient_bandit_algorithm_sampled_reward_baselines}, we have, for all $t \ge 1$,
\begin{align}
    \mathbb{E}_t{ \left[ \bigg\| \frac{d \pi_{\theta_t}^\top \big( \hat{r}_t - \hat{b}_t \big) }{d \theta_t} \bigg\|_2^2 \right] } 
    \le \frac{8 \cdot \bar{R}_{\max}^2 \cdot R_{\max} \cdot K^{3/2} }{ \Delta^2 } \cdot  \bigg\| \frac{d \pi_{\theta_t}^\top r}{d \theta_t} \bigg\|_2,
\end{align}
where $\Delta \coloneqq \min_{i \not= j}{ | r(i) - r(j) | } $, and $\bar{R}_{\max}$ is from \cref{eq:unbiased_stochastic_gradient_bounded_scale_sampled_reward_baselines_result_3_appendix}.
\begin{proof}
Given $t \ge 1$, denote $k_t$ as the action with largest probability, i.e., $k_t \coloneqq \argmax_{a \in [K]}{ \pi_{\theta_t}(a) }$. We have,
\begin{align}
\label{eq:strong_growth_conditions_sampled_reward_baselines_intermedita_1}
    \pi_{\theta_t}(k_t) \ge \frac{1}{K}.
\end{align}
According to \cref{eq:unbiased_stochastic_gradient_bounded_scale_sampled_reward_baselines_result_2_intermediate_3}, we have,
\begin{align}
\label{eq:strong_growth_conditions_sampled_reward_baselines_intermedita_2}
    \mathbb{E}_t{ \left[ \bigg\| \frac{d \pi_{\theta_t}^\top \big( \hat{r}_t - \hat{b}_t \big) }{d \theta_t} \bigg\|_2^2 \right] } &= \sum_{a \in [K]}{ \probability{\left( a_t = a \right) } \cdot \Bigg[ \bigg\| \frac{d \pi_{\theta_t}^\top \big( \hat{r}_t - \hat{b}_t \big) }{d \theta_t} \bigg\|_2^2 \ \Big| \ a_t = a \Bigg] } \\
    &\le \sum_{a \in [K]}{ \pi_{\theta_t}(a) \cdot 2 \cdot \bar{R}_{\max}^2 \cdot \left( 1 - \pi_{\theta_t}(a) \right)^2 } \\
    &= 2 \cdot \bar{R}_{\max}^2 \cdot \bigg[ \pi_{\theta_t}(k_t) \cdot \left( 1 - \pi_{\theta_t}(k_t) \right)^2 + \sum_{a \not= k_t}{ \pi_{\theta_t}(a) \cdot \left( 1 - \pi_{\theta_t}(a) \right)^2  } \bigg] \\
    &\le 2 \cdot \bar{R}_{\max}^2 \cdot \bigg[ 1 - \pi_{\theta_t}(k_t) + \sum_{a \not= k_t}{ \pi_{\theta_t}(a) } \bigg] \qquad \left( \pi_{\theta_t}(a) \in (0, 1) \text{ for all } a \in [K] \right) \\
    &= 4 \cdot \bar{R}_{\max}^2 \cdot \left( 1 - \pi_{\theta_t}(k_t) \right).
\end{align}
Therefore, we have,
\begin{align}
    \mathbb{E}_t{ \left[ \bigg\| \frac{d \pi_{\theta_t}^\top \big( \hat{r}_t - \hat{b}_t \big) }{d \theta_t} \bigg\|_2^2 \right] } &\le 4 \cdot \bar{R}_{\max}^2 \cdot \left( 1 - \pi_{\theta_t}(k_t) \right) \\
    &\le \frac{4 \cdot \bar{R}_{\max}^2 \cdot K}{ \Delta^2 } \cdot \sum_{a \in [K] }{ \pi_{\theta_t}(a) \cdot (r(a) - \pi_{\theta_t}^\top r)^2 } \qquad \left( \text{by \cref{eq:strong_growth_conditions_sampled_reward_intermedita_6}} \right) \\
    &\le \frac{4 \cdot \bar{R}_{\max}^2 \cdot K}{ \Delta^2 } \cdot 2 \cdot \sqrt{K} \cdot R_{\max} \cdot \bigg\| \frac{d \pi_{\theta_t}^\top r}{d \theta_t} \bigg\|_2 \qquad \left( \text{by \cref{eq:strong_growth_conditions_sampled_reward_intermedita_4}} \right) \\
    &= \frac{8 \cdot \bar{R}_{\max}^2 \cdot R_{\max} \cdot K^{3/2} }{ \Delta^2 } \cdot \bigg\| \frac{d \pi_{\theta_t}^\top r}{d \theta_t} \bigg\|_2. \qedhere
\end{align}
\end{proof}

\begin{lemma}[Constant learning rate]
Using \cref{alg:gradient_bandit_algorithm_sampled_reward_baselines} with $\eta = \frac{\Delta^2}{40 \cdot K^{3/2} \cdot \bar{R}_{\max}^2 \cdot R_{\max} }$, we have, for all $t \ge 1$,
\begin{align}
    \pi_{\theta_t}^\top r - \EEt{ \pi_{\theta_{t+1}}^\top r } \le - \frac{\Delta^2}{80 \cdot K^{3/2} \cdot \bar{R}_{\max}^2 \cdot R_{\max} } \cdot \bigg\| \frac{d \pi_{\theta_t}^\top r}{d \theta_t} \bigg\|_2^2,
\end{align}
where $\bar{R}_{\max}$ is from \cref{eq:unbiased_stochastic_gradient_bounded_scale_sampled_reward_baselines_result_3_appendix}.
\end{lemma}
\begin{proof}
Using the learning rate,
\begin{align}
\label{eq:constant_learning_rates_expected_progress_sampled_reward_baselines_intermediate_0a}
    \eta &= \frac{\Delta^2}{40 \cdot K^{3/2} \cdot \bar{R}_{\max}^2 \cdot R_{\max} } \\
    &= \frac{4}{45 \cdot \bar{R}_{\max} } \cdot \frac{\Delta^2}{\bar{R}_{\max} \cdot R_{\max}} \cdot \frac{1}{K^{3/2}} \cdot \frac{45}{4} \cdot \frac{1}{40} \\
    &\le \frac{4}{45 \cdot \bar{R}_{\max} } \cdot 4 \cdot \frac{1}{2 \cdot \sqrt{2}} \cdot  \frac{45}{4} \cdot \frac{1}{40}, \qquad \left( \Delta \le 2 \cdot R_{\max}, \, \Delta \le 2 \cdot \bar{R}_{\max}, \text{ and } K \ge 2 \right) \\
\label{eq:constant_learning_rates_expected_progress_sampled_reward_baselines_intermediate_0b}
    &< \frac{4}{45 \cdot \bar{R}_{\max} },
\end{align}
we have $\eta \in \big(0, 2 / (9 \, \bar{R}_{\max}) \big)$. According to \cref{lem:non_uniform_smoothness_special_two_iterations_baselines}, we have,
\begin{align}
\MoveEqLeft
\label{eq:constant_learning_rates_expected_progress_sampled_reward_baselines_intermediate_1}
    \left| ( \pi_{\theta_{t+1}} - \pi_{\theta_t})^\top r - \Big\langle \frac{d \pi_{\theta_t}^\top r}{d \theta_t}, \theta_{t+1} - \theta_t \Big\rangle \right| \le \frac{3}{2 - 9 \cdot \bar{R}_{\max} \cdot \eta } \cdot \bigg\| \frac{d \pi_{\theta_t}^\top r}{d \theta_t} \bigg\|_2 \cdot \| \theta_{t+1} - \theta_t \|_2^2 \\
    &\le \frac{3}{2 - 9 \cdot \bar{R}_{\max} \cdot \frac{4}{45 \cdot \bar{R}_{\max} } } \cdot \bigg\| \frac{d \pi_{\theta_t}^\top r}{d \theta_t} \bigg\|_2 \cdot \| \theta_{t+1} - \theta_t \|_2^2 \qquad \left( \text{by \cref{eq:constant_learning_rates_expected_progress_sampled_reward_baselines_intermediate_0b}} \right) \\
    &= \frac{5}{2} \cdot \bigg\| \frac{d \pi_{\theta_t}^\top r}{d \theta_t} \bigg\|_2 \cdot \| \theta_{t+1} - \theta_t \|_2^2,
\end{align}
which implies that,
\begin{align}
\label{eq:constant_learning_rates_expected_progress_sampled_reward_baselines_intermediate_2}
    \pi_{\theta_t}^\top r - \pi_{\theta_{t+1}}^\top r &\le - \Big\langle \frac{d \pi_{\theta_t}^\top r}{d \theta_t}, \theta_{t+1} - \theta_t \Big\rangle +  \frac{5}{2} \cdot \bigg\| \frac{d \pi_{\theta_t}^\top r}{d \theta_t} \bigg\|_2 \cdot \| \theta_{t+1} - \theta_{t} \|_2^2 \\
    &= - \eta \cdot \Big\langle \frac{d \pi_{\theta_t}^\top r}{d \theta_t}, \frac{d \pi_{\theta_t}^\top \big( \hat{r}_t - \hat{b}_t \big) }{d \theta_t} \Big\rangle + \frac{5}{2} \cdot \bigg\| \frac{d \pi_{\theta_t}^\top r}{d \theta_t} \bigg\|_2 \cdot \eta^2 \cdot \bigg\| \frac{d \pi_{\theta_t}^\top \big( \hat{r}_t - \hat{b}_t \big) }{d \theta_t} \bigg\|_2^2,
\end{align}
where the last equation uses \cref{alg:gradient_bandit_algorithm_sampled_reward_baselines}.
Taking expectation over $a_t \sim \pi_{\theta_t}(\cdot)$ and $R_t(a_t) \sim P_{a_t}$, we have,
\begin{align}
\MoveEqLeft
\label{eq:constant_learning_rates_expected_progress_sampled_reward_baselines_intermediate_3}
    \pi_{\theta_t}^\top r - \EEt{ \pi_{\theta_{t+1}}^\top r } \le - \eta \cdot \Big\langle \frac{d \pi_{\theta_t}^\top r}{d \theta_t}, \mathbb{E}_t{ \bigg[ \frac{d \pi_{\theta_t}^\top \hat{r}_t}{d \theta_t} \bigg] } \Big\rangle + \frac{5}{2} \cdot \bigg\| \frac{d \pi_{\theta_t}^\top r}{d \theta_t} \bigg\|_2 \cdot \eta^2 \cdot \mathbb{E}_t{ \Bigg[ \bigg\| \frac{d \pi_{\theta_t}^\top \hat{r}_t}{d \theta_t} \bigg\|_2^2  \Bigg] } \\
    &= - \eta \cdot \bigg\| \frac{d \pi_{\theta_t}^\top r}{d \theta_t} \bigg\|_2^2 + \frac{5}{2} \cdot \bigg\| \frac{d \pi_{\theta_t}^\top r}{d \theta_t} \bigg\|_2 \cdot \eta^2 \cdot \mathbb{E}_t{ \Bigg[ \bigg\| \frac{d \pi_{\theta_t}^\top \big( \hat{r}_t - \hat{b}_t \big) }{d \theta_t} \bigg\|_2^2  \Bigg] } \qquad \left( \text{by \cref{lem:unbiased_stochastic_gradient_bounded_scale_sampled_reward_baselines}} \right) \\
    &\le - \eta \cdot \bigg\| \frac{d \pi_{\theta_t}^\top r}{d \theta_t} \bigg\|_2^2 + \frac{5}{2} \cdot \bigg\| \frac{d \pi_{\theta_t}^\top r}{d \theta_t} \bigg\|_2 \cdot \eta^2 \cdot \frac{8 \cdot \bar{R}_{\max}^2 \cdot R_{\max} \cdot K^{3/2} }{ \Delta^2 } \cdot \bigg\| \frac{d \pi_{\theta_t}^\top r}{d \theta_t} \bigg\|_2 \qquad \left( \text{by \cref{lem:strong_growth_conditions_sampled_reward_baselines}} \right) \\
    &= \left( - \eta + \eta^2 \cdot \frac{20 \cdot \bar{R}_{\max}^2 \cdot R_{\max} \cdot K^{3/2} }{ \Delta^2 } \right) \cdot \bigg\| \frac{d \pi_{\theta_t}^\top r}{d \theta_t} \bigg\|_2^2 \\
    &= - \frac{\Delta^2}{80 \cdot K^{3/2} \cdot \bar{R}_{\max}^2 \cdot R_{\max} } \cdot \bigg\| \frac{d \pi_{\theta_t}^\top r}{d \theta_t} \bigg\|_2^2. \qquad \left( \text{by \cref{eq:constant_learning_rates_expected_progress_sampled_reward_baselines_intermediate_0a}} \right) \qedhere
\end{align}
\end{proof}

\begin{theorem}
    Using \cref{alg:gradient_bandit_algorithm_sampled_reward_baselines}, we have, the sequence $\{ \pi_{\theta_t}^\top r \}_{t\ge 1}$ converges with probability one.
\end{theorem}
\begin{proof}
    As in the proof for \cref{thm:asymp_global_converg_gradient_bandit_sampled_reward}, we set
    % Using the above notations, we have,
\begin{align}
\MoveEqLeft
% \label{eq:non_vanishing_nl_coefficient_stochastic_npg_value_baseline_special_claim_2_intermediate_7}
    W_{t+1}(a) = \theta_{t+1}(a) - \EEt{\theta_{t+1}(a)} \\
    &=  \theta_{t}(a) + \eta \cdot \left[ I_t(a) - \pi_{\theta_t}(a) \right] \cdot ( R_t(a_t) - B_t )  - \left[ \theta_{t}(a) + \eta \cdot \pi_{\theta_t}(a) \cdot \left( r(a) - \pi_{\theta_t}^\top r \right) \right] \\
    &= \eta \cdot \left[ I_t(a) - \pi_{\theta_t}(a) \right] \cdot ( R_t(a_t) - B_t )    - \eta \cdot \pi_{\theta_t}(a) \cdot \left[ r(a) - \pi_{\theta_t}^\top r \right],
\end{align}
which implies that,
\begin{align}
% \label{eq:non_vanishing_nl_coefficient_stochastic_npg_value_baseline_special_claim_2_intermediate_8}
    Z_t(a) &= W_1(a) + \cdots + W_t(a)
    \\&=\sum_{s=1}^t  \eta \cdot \left[ I_s(a) - \pi_{\theta_s}(a) \right] \cdot ( R_s(a_s) - B_s )    - \eta \cdot \pi_{\theta_s}(a) \cdot \left[ r(a) - \pi_{\theta_s}^\top r \right] .
\end{align}
We also have,
\begin{align}
% \label{eq:non_vanishing_nl_coefficient_stochastic_npg_value_baseline_special_claim_2_intermediate_9}
    P_t(a) &= \EEt{\theta_{t+1}(a)} - \theta_t(a) 
    = \eta \cdot \pi_{\theta_t}(a) \cdot \left[  r(a) - \pi_{\theta_t}^\top r \right].
\end{align}
In the remaining part of the proofs for \cref{thm:asymp_global_converg_gradient_bandit_sampled_reward}, replacing $R_{\max}$ with $\Bar{R}_{\max}$, we have the results.
\end{proof}

\section{Miscellaneous Extra Supporting Results}
\label{sec:supporting_results}
Recall that 
$(X_t,\cF_t)_{t\ge 1}$ is a \emph{sub-martingale} (super-martingale, martingale) if $(X_t)_{t\ge 1}$ is adapted to the filtration $(\cF_t)_{t\ge 1}$ 
and $\EE{X_{t+1}|\cF_t} \ge X_t$
($\EE{X_{t+1}|\cF_t} \le X_t$, $\EE{X_{t+1}|\cF_t} = X_t$, respectively)
 holds almost surely for any $t\ge 1$.
For brevity, let $\EEt{\cdot}$ denote $\EE{\cdot|\cF_t}$ where the filtration should be clear from the context and we also extend this notation to $t=0$ such that $\chE_0{U} = \EE{U}$.

\begin{theorem}[Doob's supermartingale convergence theorem \citep{doob2012measure}]
\label{thm:smc}
If $(Y_t)_{t\ge 1}$ is an $\{ \cF_t \}_{t\ge 1}$-adapted sequence such that $\EE{Y_{t+1}|\cF_t}\le Y_t$ and
$\sup_t \EE{ |Y_t| }<\infty$ then $\{Y_t\}_{t\ge 1}$ almost surely converges (a.s.) and, in particular, $Y_t \to Y$ a.s. as $t\to\infty$ where $Y=\limsup_{t\to\infty}Y_t$ is such that $\EE{|Y|}<\infty$.
\end{theorem}

\begin{lemma}[Extended Borel-Cantelli Lemma, Corollary 5.29 of \citep{breiman1992probability}]
\label{lem:ebc}
Let $( \cF_n)_{n \ge 1}$ be a filtration, $A_n \in \cF_n$.
Then, almost surely, 
\begin{align*}
\{ \omega \,: \, \omega \in A_n \text{ infinitely often } \} = \left\{ \omega \, : \, 
\sum_{n=1}^\infty \PP(A_n|\cF_n) \right\}\,.
\end{align*}
\end{lemma}

% {\color{blue}
\begin{theorem}\label{thm:conc_new}
    Let $X_1, X_2, \ldots$ be a sequence of random variables, such that for all $t \ge 1$, $|X_t|\le 1/2 $. Define 
    % the sum of bounded martingale difference sequence 
    % $S_n = \sum_{t=1}^n | \bbE[ L_t | L_1, \ldots, L_{t-1} ] - L_t |$ and 
    % the associated martingale $S_n = V_1 + \ldots + V_n$ with conditional variance 
    % $V_n = \sum_{t=1}^n \var[ X_t | X_1, \ldots, X_{t-1} ]$. 
\begin{align}
    S_n \coloneqq \left| \sum_{t=1}^n \bbE[ X_t | X_1, \ldots, X_{t-1} ] - X_t \right|
    \quad \text{ and } \quad 
    V_n \coloneqq \sum_{t=1}^n \mathrm{Var}[ X_t | X_1, \ldots, X_{t-1} ].
\end{align}
Then, for all $\delta> 0$,
\begin{align}
    \bbP\left( \exists \ n:\ S_n\ge  6 \ \sqrt{  \left(V_n+ \frac{4}{3} \right) \ \log \left( \frac{  V_n+1  }{ \delta } \right) } + 2\log \left( \frac{1}{\delta} \right)  + \frac{4}{3} \log 3
    ~ \right)  \le \delta.
\end{align}
\end{theorem}
% }

\subsection{Proof of \cref{thm:conc_new}}
\begin{proof}
Let $g(x) \coloneqq  2\log \left( \frac{ (x+1)(x+3) }{ \delta } \right)$ with $x\ge 0$.
We have
\begin{align}
    &\bbP\left( \exists \ n:\ S_n\ge g(V_n) + \sqrt{ g(V_n) \ V_n }  \right)  \\
    &= \sum_{i=0}^\infty \bbP \left( \exists \ n:\ S_n\ge g(V_n) + \sqrt{ g(V_n) \ V_n}, \ i\le V_n <i+1 \right) \\
    &\le \sum_{i=0}^\infty \bbP \left( \exists \ n:\ S_n\ge g(V_n) + \sqrt{ g(V_n) \ i}, \ V_n \le i+1 \right) \\
    &= \sum_{i=0}^\infty \bbP \left( \exists \  n:\ S_n \ge \frac{ g(V_n) + \sqrt{ g(V_n) \ i} }{i+1} \cdot (i+1), \ V_n \le i+1 \right) \\
    &
    \le \sum_{i=0}^\infty \exp\left\{ - \frac{ \left( g(V_n) + \sqrt{ g(V_n) \ i} \right)^2 \cdot (i+1) }{ 2 \ (i+1)^2 \cdot \left(  1+  \frac{ g(V_n) + \sqrt{ g(V_n) \ i} }{3\ (i+1)} \right) }   \right\} \qquad (\text{by \cref{lem:freedman_ineq}}) \\
    &= \sum_{i=0}^\infty \exp\left\{ - \frac{ \left( g(V_n) + \sqrt{ g(V_n) \ i} \right)^2 }{  2\ (i+1) + \frac{ 2 }{ 3} \left( g(V_n) + \sqrt{ g(V_n) \ i}\right)  }   \right\} \\
    &\le  \sum_{i=0}^\infty \exp\left( - \frac{ g(i)   }{  2 }   \right) \qquad \left(g(x) \ge 2 \log 3, \text{ and by \cref{claim:conc_algebra}}\right)\label{eq:conc_derive} \\
    &= \delta \cdot \sum_{i=0}^\infty \frac{1}{ (x+1)(x+3)}
    = \frac{\delta}{2} \cdot \sum_{i=0}^\infty \left(  \frac{1}{x+1} - \frac{1}{x+3}  \right)
    = \frac{ \delta }{2} \cdot \left( 1- \frac{1}{3} + \frac{1}{2} - \frac{1}{4} + \frac{1}{3} - \frac{1}{5} + \ldots \right) \\
    & \le \frac{ \delta }{2} \cdot \left(  1 + \frac{1}{2} \right)
    \le \delta. ~\label{eq:conc_to_rewrite}
\end{align}
Plugging $g(x)=  2\log \left( \frac{ (x+1)(x+3) }{ \delta } \right)$ into \cref{eq:conc_to_rewrite}, we have
\begin{align}
    \bbP\left( \exists \ n:\ S_n\ge 2\log \left( \frac{ (V_n+1)(V_n+3) }{ \delta } \right) + \sqrt{ 2 \ V_n \ \log \left( \frac{ (V_n+1)(V_n+3) }{ \delta } \right) } ~ \right)  \le \delta.
\end{align}
Since $x+4/3 \ge \log\left((x+1)(x+3)\right)$ for all $x\ge 0$, we have
\begin{align}
    & 2\log \left( \frac{ (V_n+1)(V_n+3) }{ \delta } \right) + \sqrt{ 2 \ V_n \ \log \left( \frac{ (V_n+1)(V_n+3) }{ \delta } \right) } \\
    &= 2\log(1/\delta) + \sqrt{ 2 \ V_n \ \log \left( \frac{ (V_n+1)(V_n+3) }{ \delta } \right) } +  2\log \left( (V_n+1)(V_n+3) \right)  \\
    &\le 2\log(1/\delta) + \sqrt{ 2 \ V_n \ \log \left( \frac{ (V_n+1)(V_n+3) }{ \delta } \right) } +  2  \sqrt{  (V_n+4/3)\log \left( \frac{ (V_n+1)(V_n+3) }{ \delta } \right) } \\
    &\le 2\log(1/\delta) + (2+ \sqrt{2})\cdot  \sqrt{  (V_n+4/3)\log \left( \frac{ (V_n+1)(V_n+3) }{ \delta } \right) }.
\end{align}
Let $f(x) = (x+4/3)\log(x+3) - 2(x+1)\log(x+1) - 4/3 \log 3 $ for all $x\ge 0$, then
\begin{align}
    & f'(x) = \log(x+3) + \frac{x+4/3}{x+3} - 2\log(x+1) - 2,%  = \log(x+3)-2\log(x+1)-\frac{23/3}{x+3}-1,
    \\& f''(x)= \frac{1}{x+3} - \frac{2}{x+1} + \frac{5/3}{(x+3)^2} = - \frac{ x+5 }{(x+1)(x+3)} + \frac{5/3}{(x+3)^2}.
\end{align}
%
% $f'(x) = \log(x+3) + \frac{x}{x+3} - 2\log(x+1) - 2= \log(x+3)-2\log(x+1)-\frac{3}{x+3} -1$
% and 
% $f''(x)= \frac{1}{x+3} - \frac{2}{x+1} + \frac{3}{(x+3)^2} = - \frac{ x+5 }{(x+1)(x+3)} + \frac{3}{(x+3)^2}$.
Since $f''(x)<0$ for all $x\ge 0$, indicating that $f'(x)$ is non-increasing with $x$. Since $f'(0) = \log 3 -14/9<0$, we have $f'(x)<0$ for all $x\ge 0$ and hence $f(x)$ is non-increasing with $x$. Since $f(0)=0$, we have $f(x)\le 0$ for all $x\ge 0$.
Hence, we have
\begin{align}
    & (V_n+4/3)\log \left( \frac{ (V_n+1)(V_n+3) }{ \delta } \right) 
    \\&
    = (V_n+4/3)\log \left( \frac{ V_n+1  }{ \delta } \right) + (V_n+4/3)\log \left( \frac{ V_n+3  }{ \delta } \right) 
    \\&
    \le (V_n+4/3)\log \left( \frac{ V_n+1  }{ \delta } \right) + 2(V_n+1)\log \left( \frac{ V_n+1  }{ \delta } \right) + \frac{4}{3} \log 3
    \\&
    \le 3(V_n+4/3)\log \left( \frac{ V_n+1  }{ \delta } \right) + \frac{4}{3}  \log 3 .
\end{align}
Since $(2+\sqrt{2})\cdot \sqrt{3} \le 6$, we have
\begin{equation*}
    \bbP\left( \exists n:\ S_n\ge  6\sqrt{  (V_n+4/3)\log \left( \frac{  V_n+1  }{ \delta } \right) } + 2\log(1/\delta) + \frac{4}{3} \log 3
    ~ \right)  \le \delta. \qedhere
\end{equation*}

\end{proof}

\begin{lemma}%[{\citet[Lemma 1]{cesa2005improved}}]
\label{lem:freedman_ineq}
    Let $L_1, L_2, \ldots$ be a sequence of random variables, such that for all $t \ge 1$, $ | L_t|\le 1/2 $. Define the bounded martingale difference sequence $\xi_t = \bbE[ L_t | L_1, \ldots, L_{t-1} ] - L_t$ and the associated martingale $S_n = | \xi_1 + \ldots + \xi_n | $  with conditional variance $V_n = \sum_{t=1}^n \mathrm{Var}[ L_t | L_1, \ldots, L_{t-1} ]$. 
    Then, for all $x, v^2>0$,
    \begin{align}
        \bbP(   S_n \ge xv^2, \ V_n \le v^2 ) \le \exp \left(  - \frac{ x^2 v^4 }{ 2v^2+2xv^2/3 } \right)
        = \exp \left(  - \frac{ x^2 v^2 }{ 2(1+x /3) } \right).
    \end{align}
    % Then, for all $s,k\ge 0$,
    % \begin{align}
    %     \bbP( |S_n| \ge s, K_n \le k ) \le \exp \left(  - \frac{ s^2 }{ 2k+2s/3 } \right).
    % \end{align}
\end{lemma}
% The following result
\cref{{lem:freedman_ineq}} is known as Freedman's inequality, which was originally from \citet{freedman1975on}. 
In detail, it is implied by Lemma 1 in \citet{cesa2005improved}, which bounds $\bbP(  \sum_{t=1}^n \xi_i \ge xv^2, \ V_n \le v^2 )$ for $0\le L_t\le 1$. For $|M_t|\le 1/2$, we can apply Lemma 1 in \citet{cesa2005improved} with $L_t^{(1)} = M_t+1/2$ and $L_t^{(2)} = 1/2-M_t$ individually to obtain \cref{lem:freedman_ineq}.

\begin{lemma} 
\label{claim:conc_algebra}
    If $u\ge 2\log 3$, then
    \begin{align}
        \min_{i\ge 0} \frac{  \left( u + \sqrt{u\cdot i} \right)^2 }{ \frac{2}{3}  \left( u + \sqrt{u\cdot i} \right) + 2(i+1) } \ge \frac{u}{2}.
    \end{align}
\end{lemma}
\begin{proof}
    Let
    \begin{align}
        f(x)=  \frac{  \left( u + \sqrt{u\cdot x} \right)^2 }{ \frac{2}{3}  \left( u + \sqrt{u\cdot x} \right) + 2(x+1) },
    \end{align}
    then
    \begin{align}
        &
        f'(x) = \frac{ \sqrt{u/x} \left( u + \sqrt{u\cdot x} \right) \cdot  \left[  \frac{2}{3}  \left( u + \sqrt{u\cdot x} \right) + 2(x+1) \right] -  \left( u + \sqrt{u\cdot x} \right)^2  \cdot \left[ \sqrt{u/x}/3  + 2\right]   }{  \left[  \frac{2}{3}  \left( u + \sqrt{u\cdot x} \right) + 2(x+1) \right]^2}
        \\&
        = \frac{ \left( u + \sqrt{u\cdot x} \right) \cdot \left[ \left( \frac{2}{3}   u\sqrt{u/x} + \frac{2}{3}u   + \cancel{ 2 \sqrt{u \cdot x} }+ 2\sqrt{u/x} \right) - \left( u\sqrt{u/x}/3  + 2u + u/3   + \cancel{2 \sqrt{u \cdot x}} \right) \right]   }{  \frac{4}{9}  \left[  \left( u + \sqrt{u\cdot x} \right) + 2(x+1) \right]^2}
        \\&
        = \frac{ \left( u + \sqrt{u\cdot x} \right) \cdot  \left( \frac{1}{3}   u\sqrt{u/x} - \frac{5}{3}u   \right)    }{  \frac{4}{9}  \left[  \left( u + \sqrt{u\cdot x} \right) + 2(x+1) \right]^2}
        \\&
        = \frac{ u\left( u + \sqrt{u\cdot x} \right) \cdot  \left( \sqrt{u/x} - 5  \right)    }{  \frac{4}{3}  \left[  \left( u + \sqrt{u\cdot x} \right) + 2(x+1) \right]^2}.
    \end{align}
    We see that
    \begin{itemize}
        \item $x<u/25$: $f'(x)>0$ and $f(x)$ increases with $x$; 
        \item $x= u/25$: $f'(x)=0$;
        \item $x>u/25$: $f'(x)<0$ and $f(x)$ decreases with $x$.
    \end{itemize}
    Moreover, we have
    \begin{align}
        & f(0) = \frac{ u^2 }{  \frac{2u}{3} + 2},
        \quad
        \lim_{x\rightarrow \infty} f(x) =  \lim_{x\rightarrow \infty} \frac{  ux }{ 2x } = \frac{u}{2}.
    \end{align}
    Note that by assumption, we have $u \ge 2\log 3 \ge 3/2$, and then
    \begin{align}
        u \ge \frac{ 3 }{2}
        ~\Leftrightarrow~  \frac{ 2u }{3} \ge 1
        ~\Leftrightarrow~  u\ge \frac{1}{2} \left( \frac{2u}{3} + 2 \right)
        ~\Leftrightarrow~ f(0) \ge \frac{u}{2}.
    \end{align}
Hence, when  $u \ge 2\log 3$, $f(x)\ge u/2$ for all $x\ge 0$, which completes the proof.
\end{proof}

%%%%%%%%%%%%%%%%%%%%%%%%%%%%%%%%%%%%%%%%%%%%%%%%%%%%%%%%%%%%%%%%%%%%%%%%%%%%%%%
%%%%%%%%%%%%%%%%%%%%%%%%%%%%%%%%%%%%%%%%%%%%%%%%%%%%%%%%%%%%%%%%%%%%%%%%%%%%%%%

\end{document}